%% file: iclr2027_arxiv.tex
\documentclass{article} 
\usepackage{iclr2027_arxiv,times}

\input{math_commands.tex}

\usepackage{hyperref}
\usepackage{url}
\usepackage{booktabs}       
\usepackage{amsfonts}       
\usepackage{nicefrac}       
\usepackage{microtype}      
\usepackage{xcolor}         
\usepackage{xspace}
\usepackage{enumitem}
\usepackage{amsmath,amsthm,amssymb,amsfonts,dsfont,pifont,bm,bbm,mathrsfs,mathtools,nicefrac,extarrows,relsize}
\usepackage{algorithm,algpseudocode,listings}
\usepackage{algorithmicx}
\usepackage{algpseudocode}
\usepackage{booktabs,multirow,adjustbox,diagbox,threeparttable,tabularray,setspace,wrapfig}
\usepackage{caption}
\usepackage{listings,colortbl,pifont}

\definecolor{codegreen}{rgb}{0,0.6,0}
\definecolor{codegray}{rgb}{0.5,0.5,0.5}
\definecolor{codepurple}{rgb}{0.58,0,0.82}
\definecolor{backcolour}{rgb}{1.0,1.0,1.0}
\lstdefinestyle{mystyle}{
    backgroundcolor=\color{backcolour},
    commentstyle=\color{codegreen},
    keywordstyle=\color{magenta},
    numberstyle=\tiny\color{codegray},
    stringstyle=\color{codepurple},
    basicstyle=\ttfamily\scriptsize,
    breakatwhitespace=false,
    breaklines=true,
    captionpos=b,
    keepspaces=true,
    numbers=left,
    numbersep=5pt,
    showspaces=false,
    showstringspaces=false,
    showtabs=false,
    tabsize=2,
    frame=single, 
    rulecolor=\color{black}
}
\definecolor{myred}{RGB}{213,123,101}
\definecolor{myblue}{RGB}{128,127,222}
\definecolor{mygreen}{RGB}{112,173,71}
\definecolor{cvprblue}{rgb}{0.21,0.49,0.74}
\usepackage[capitalize]{cleveref}  
\newtheorem{theorem}{Theorem}
\newtheorem{proposition}{Proposition}

\crefname{section}{Sec.}{Secs.}
\Crefname{section}{Section}{Sections}
\crefname{appendix}{App.}{Apps.}
\Crefname{appendix}{Appendix}{Appendices}
\crefname{table}{Tab.}{Tabs.}
\Crefname{table}{Table}{Tables}
\crefname{figure}{Fig.}{Figs.}
\Crefname{figure}{Figure}{Figures}
\crefname{equation}{Eq.}{Eqs.}
\Crefname{equation}{Equation}{Equations}
\crefname{theorem}{Thm.}{Thms.}
\Crefname{theorem}{Theorem}{Theorems}
\crefname{lemma}{Lem.}{Lems.}
\Crefname{lemma}{Lemma}{Lemmas}
\crefname{remark}{Rem.}{Rems.}
\Crefname{remark}{Remark}{Remarks}
\crefname{corollary}{Cor.}{Cors.}
\Crefname{corollary}{Corollary}{Corollaries}
\crefname{algorithm}{Alg.}{Algs.}
\Crefname{algorithm}{Algorithm}{Algorithms}
\title{Distilling Drifting Transformers with \\ Representation Autoencoders}

\author{%
Jiawei Zhang$^{1,2}$ \quad
Mengfei Xia$^2$\thanks{Corresponding authors.} \quad
Gen Li$^3$ \quad
Yuantao Gu$^1$\footnotemark[1]
\AND
$^1$Tsinghua University \quad
$^2$Ant Group \quad
$^3$CUHK
}

\newcommand{\tocite}[1]{{\color{red} [TO CITE]}}
\newcommand{\methodname}{Drift-RAE}
\newcommand{\method}{\mbox{\texttt{\methodname}}\xspace}

\iclrfinalcopy 
\begin{document}

\maketitle

\input{sections/0.abs.tex}

\input{sections/1.intro.tex}

\input{sections/2.related.tex}

\input{sections/3.method.tex}
\input{sections/4.exp.tex}
\input{sections/5.conclusion.tex}

\input{sections/6.ref.tex}

\newpage
\appendix

\input{sections/8.appendix.tex}

\end{document}

%% file: math_commands.tex
\usepackage{amsmath,amsfonts,bm}

\def\eqref#1{equation~\ref{#1}}

\def\1{\bm{1}}

\DeclareMathAlphabet{\mathsfit}{\encodingdefault}{\sfdefault}{m}{sl}
\SetMathAlphabet{\mathsfit}{bold}{\encodingdefault}{\sfdefault}{bx}{n}



%% file: sections/0.abs.tex
\begin{abstract}

Despite the significant training acceleration and promising performance, Representation Autoencoders (RAEs) are mainly criticized for poor distillation effectiveness.
In this work, we argue that RAE is competent at high-quality one-step generation.
We achieve \textbf{1.48 FID with only 16-epoch distillation} on ImageNet 256 dataset, surpassing various state-of-the-art methods.
To achieve this, we quantitatively study the geometrical behavior of different underlying data spaces.
We conclude that conventional distillation methods heavily rely on priors of plain teacher denoising trajectories, while RAE incurs much complex trajectories with poor properties due to ill anisotropical latent space.
We introduce the recently proposed drifting field as the distillation methodology, which makes use of semantically rich RAE latents and provides direct supervision involving no dependency.
Bridging our \method with previous generative paradigms, we propose several insightful modifications, including \textit{the first extrapolation-based guided sampling pipeline for one-step generation with barely no cost}.
The code will be made publicly available.

\end{abstract}

%% file: sections/1.intro.tex
\section{Introduction}\label{sec:1.intro}

Diffusion and flow-based models~\citep{dickstein2015dm, ho2020ddpm,song2021ddim,lipman2023,liu2023rectifiedflow, peebles2023dit} have become the most dominant paradigm in generative modeling, achieving remarkable success in image~\citep{esser2024sd3,labs2025flux1kontextflowmatching} and video~\citep{blattmann2024svd,liu2024sora} synthesis.
Among the literature, the autoencoders serve as a core role, which compress the data with low information density into compact latent spaces~\citep{rombach2O22ldm,zheng2025diffusion}.
This technique allows stable and efficient training in large-dimensional spaces such as high-resolution images and video.
Representation Autoencoders (RAEs)~\citep{zheng2025diffusion,scale-rae-2026,singh2026raev2}, one of the most representative work, replace the conventional VAE with pre-trained representation modules and corresponding decoders.
The resulting latent spaces allocate much richer semantics and strongly label-wise clustered features, thus enabling better efficiency and performance.

Despite the great potentials, RAEs are centrally criticized for poor compatibility with distillation.
Resembling any diffusion and flow-based models, RAE also relies on the iterative sampling steps along the denoising trajectories by probability-flow ODE.
Conventional methods distill the pre-trained models using the intrinsic priors of the denoising trajectories~\citep{salimans2022pd,song2023cm, wang2024phased, sauer2024add, sauer2024ladd, lin2024sdxllightning, yin2024dmd, yin2024dmd2, zhou2024sid, yin2024improved}, \textit{i.e.}, compressing the lengthy trajectory consisting of hundreds of thousands of discretized steps into one-step or few-step generation.
However, existing attempt~\citep{hu2025raemeanflow} suggests that distillation in RAE latent space might be ineffective and unstable.
Therefore, the unacceptable inference cost and unsatisfactory distillation performance make RAE an impractical alternative in large-scale generation tasks.

In this work, we introduce an efficient and effective pipeline of RAE distillation.
We first claim in \Cref{sec:3.method.2} that, the instability and poor effectiveness of distillation in RAE latent space stem from the severe anisotropy of the patch-wise token distributions.
Consequently, the mismatch for RAE to transport isotropic Gaussian noises to anisotropic underlying latent space would enforces extremely curved denoising trajectories.
On the other hand, conventional distillation methods implicitly rely on plain or almost linear teacher trajectories.
Therefore, the large curvatures of trajectories in RAE greatly increase the difficulty of distillation.
To this end, we propose to distill using the recent Drifting Models~\citep{deng2026generative}.
Drifting Models are technically a one-step generative paradigm, which directly compute a field to evaluate the discrepancy between generated and real distributions.
The drifting field enforces the generated distribution to coincide with the real one and involves no reliance on teacher trajectories, thus could serve as a competent technique for RAE distillation.

However, the original Drifting Model heavily relies on an auxiliary MAE module as overhead extractor to filter and concentrate features in the underlying space, hindering the practical deployment.
By theoretically studying the behavior, we argue that it is the heavily dispersed involved samples that annihilates the effectiveness of drifting field, and initialization with poor semantics fails to provide meaningful warmup.
Nevertheless, the semantically richer RAE latent space is full of well clustered patch embeddings conversely, enabling effective drifting field supervision.
To this end, we propose our RAE distillation method, namely \method, with high efficiency and great performance.
To make a further step, we bridge the theory of Drifting Model and other generative paradigms, introducing several practical modifications with high efficacy.
Inspired by the drifting methodology, we introduce Internal Drifting Guidance (IDG), the first extrapolation-based guided sampling pipeline for one-step generation with barely no cost.

We evaluate the proposed \method on ImageNet $256\times256$.
Benefiting from the effective designed techniques, we achieve \textbf{1.48 FID within 16-epoch distillation}, surpassing or appearing competitive with state-of-the-art one-step or even few-step distillation methods beyond RAE latent spaces~\citep{Kim2024pagoda,piflow,tong2025freeflow, liu2026learning, peng2026facm}.
Meanwhile, compared with the original Drifting Model, our method achieves improved $\text{FD}_\text{DINOv2}$ without requiring any auxiliary module.
These results verify our analyses of both RAE and Drifting Model, and confirm the efficacy to employ drifting fiels in RAE distillation.

%% file: sections/2.related.tex
\section{Related Work}\label{sec:2.related}

\subsection{Flow-Based Models and Distillation}\label{sec:2.related.1}

Flow-based Models, including Diffusion Models~\citep{dickstein2015dm,song2021scoresde,ho2020ddpm} and Flow Matching~\citep{liu2023rectifiedflow,lipman2023}, are designed to formulate the relation between data and noise distributions through differential equations.
Detailedly, the training stage introduces a forward process by corrupting initial data signals with independent noises, while the inference stage involves an iterative denoiser with scores following either SDE or ODE trajectory.
However, approximating the scores of the whole process in a huge pixel space is extremely time-consuming.
To this end, LDM~\citep{rombach2O22ldm} and RAE~\citep{zheng2025diffusion} separately introduce to train flow-based models in a compressed latent space instead of the original pixel space.
Despite the unprecedented capability, the iterative reverse process hinders the sampling efficiency of flow-based models.
To address this issue, many attempts have been made to distill the knowledge from pre-trained models and reduce the denoising steps~\citep{salimans2022pd,song2023cm,luo2023a,yin2024dmd,zhou2024sid}.

\subsection{One-Step Generation Trained From Scratch}\label{sec:2.related.2}

Generative Adversarial Network (GAN) is the most representative paradigm to train a one-step generator from scratch~\citep{goodfellow2014gan}, which simultaneously train a generator and a discriminator via adversarial training.
Recently, however, GAN seems to fall from the grace on synthesis performance due to mode collapse~\citep{arjovsky2017a}.
Another family of methods directly realizes the one-step generation by incorporating the prior SDE or ODE dynamic and overfitting the corresponding trajectories~\citep{song2023cm,song2024ict,geng2025meanflow}.
Drifting Model~\citep{deng2026generative} is a novel framework, which proposes to progressively evolve the generated distribution towards the real one with a specially designed drifting field.
Concretely, the drifting field is computed to evaluate the discrepancy between two distributions via instance-wise distances and contrastive learning.

\subsection{Guided Sampling}

Modern diffusion and flow-based models applies various guidance method to improve sampling quality.
Classifier Guidance~\citep{dhariwal2021adm} and Classifier-Free Guidance~\citep{ho2021cfg} propose to extrapolate the conditional score with classifier gradient and addtionally trained unconditional score respectively, which increase the conditional fidelity.
AutoGuidance~\citep{karras2024autoguidance} generalizes the idea beyond conditional posterior, and introduce extrapolation between a ``bad version'' and the model itself, serving as a new paradigm of guided sampling.
More recently, \citet{zhou2025IG} demonstrates the internal representations of the diffusion transformer are feasible to provide guidance for the backbone output.
\citet{singh2026raev2} further shows that the REPA head~\citep{yu2025repa} can be viewed as a ``weak model'' for AutoGuidance.
Such guidance methods obtained from intermediate layers avoid auxilary models or additional forward passes and enjoys superior computional overhead.

%% file: sections/3.method.tex
\section{Method}\label{sec:3.method}


\subsection{Prerequisites}\label{sec:3.method.1}

Denote by $\mathbf y\sim q(\mathbf y)$ the real data distribution.
Flow matching~\citep{liu2023rectifiedflow}, one of the most representative flow-based models, defines a forward dynamic by linear interpolation, \textit{i.e.},
\begin{align}
\mathbf y_t=(1-t)\mathbf y+t\boldsymbol\epsilon,
\end{align}
in which $t\in[0,1]$ and $\boldsymbol\epsilon\sim\mathcal N(0,\mathbf I)$.
Then the flow matching starts the generation process at $t=1$ from pure Gaussian noises with an underlying velocity term $\mathbf v(\mathbf y_t,t)$:
\begin{align}
\mathrm d\mathbf y_t=\mathbf v(\mathbf y_t,t)\mathrm dt,
\end{align}
in which the velocity term $\mathbf v(\mathbf y_t,t)$ has closed-form expression as below:
\begin{align}
\mathbf v(\mathbf y_t,t)=\mathbb E[\dot{\mathbf y_t}|\mathbf y_t]=\mathbb E[\boldsymbol\epsilon-\mathbf y|\mathbf y_t].
\end{align}
Therefore, flow matching employs a model $\mathbf v_\theta(\mathbf y_t,t)$ to approximate $\mathbf v(\mathbf y_t,t)$ by optimizing the objective below:
\begin{align}
\mathcal L(\theta)=\int_0^1\mathbb E_{\mathbf y,\boldsymbol\epsilon}\|\mathbf v_\theta(\mathbf y_t,t)-(\boldsymbol\epsilon-\mathbf y)\|^2\mathrm dt.
\end{align}

Drifting Model~\citep{deng2026generative} trains a one-step generator from scratch by computing the drifting field between real data samples $\{\mathbf y_i\}$ and synthesized samples $\{\mathbf x_j\}$.
Notably, the drifting field enforces each $\mathbf x_j$ to move away from other $\{\mathbf x_k\}_{k\neq j}$ (negative samples) and towards $\{\mathbf y_i\}$ (positive samples).
Formally, the drifting field $\mathbf V_j$ for each $\mathbf x_j$ could be formulated as below:
\begin{align}\label{eq:drifting}
\mathbf V_j=\sum_{i}\frac{e^{-\frac{1}{\tau}\|\mathbf y_i-\mathbf x_j\|}}{\sum\limits_{l}e^{-\frac{1}{\tau}\|\mathbf y_l-\mathbf x_j\|}}(\mathbf y_i-\mathbf x_j)-\sum_{k\neq j}\frac{e^{-\frac{1}{\tau}\|\mathbf x_k-\mathbf x_j\|}}{\sum\limits_{m\neq j}e^{-\frac{1}{\tau}\|\mathbf x_m-\mathbf x_j\|}}(\mathbf x_k-\mathbf x_j).
\end{align}
\citet{deng2026generative} claim that, when all drifting fields are annihilated, the synthesized distribution would coincide with real distribution.

\subsection{Rethinking the Dynamics of RAE and Drifting Model}\label{sec:3.method.2}

Trajectory-based distillation methods typically rely on the implicit assumption that, the underlying latent space is approximately isotropic, such that the induced flow ODE trajectories remain sufficiently smooth and have moderate curvature~\citep{fan2026scot}.
Motivated by this, we compare in \Cref{fig:curvatures} the curvatures of $32$ flow ODE trajectories in RAE and traditional SD-VAE latent spaces, following the analysis protocol of~\cite{chen2024trajectory}.
The results show that trajectories in the RAE latent space have curvature values approximately \textbf{two orders of magnitude larger} than those of SD-VAE.
In addition, \Cref{tab:anisotropy} reports the average participation ratio (PR) and spectral entropy (SE) per token, further revealing that the RAE latent space is substantially more anisotropic.
These observations suggest that conventional trajectory-based distillation methods can become unstable or inefficient when directly applied to RAEs, motivating the need for alternative approaches that explicitly account for the geometry of the RAE latent space.

\input{figs/anisotropy_layout}

To address this issue, we claim that the recently proposed Drifting Models~\citep{deng2026generative} are well suited for flow distillation in RAE latent spaces.
Drifting Models are designed to narrow the gap between distributions by directly computing the drifting field with two batches of samples instead of matching ODE trajectories.
Therofore, unlike conventional distillation methods, the negative effects by highly curved trajectories in RAE latent spaces are mostly alleviated.

Despite the straightforward methodology, we further argue that RAE latent spaces could conversely complement the training dynamic of Drifting Models.
Recall that in original Drifting Model, empirically it is necessary to involve a supernumerary MAE as the feature extractor.
We below give a theoretical analysis to confirm the necessity of MAE under some ill-posed assumptions.
Corresponding proof is deferred to Appendix~\ref{sec:proof.1}.

\begin{theorem}\label{thm:counterexample}
Let $\{\mathbf y_i\}_{i=1}^d$ be the positive samples uniformly sampled from $d$-dimensional unit sphere $\mathbb S^{d-1}$, and $\{\mathbf x_j\}_{j=1}^d$ be the negative samples uniformly sampled from $[-r,r]^d$ with fixed $r>0$.
Consider the simplified drifting term in \cref{eq:drifting}, \textit{i.e.},
\begin{align}
\mathbf V_j=\mathbf{V}_j^+-\mathbf{V}_j^-,
\end{align}
\begin{align}
\mathbf V_j^+&=\sum_{i=1}^d\frac{e^{-\frac{1}{\tau}\|\mathbf y_i-\mathbf x_j\|}}{\sum\limits_{l=1}^de^{-\frac{1}{\tau}\|\mathbf y_l-\mathbf x_j\|}}\mathbf y_i, \quad\mathbf V_j^-=\sum_{k\neq j}\frac{e^{-\frac{1}{\tau}\|\mathbf x_k-\mathbf x_j\|}}{\sum\limits_{m\neq j}e^{-\frac{1}{\tau}\|\mathbf x_m-\mathbf x_j\|}}\mathbf x_k,
\end{align}
in which $\mathbf V_j^+$ and $\mathbf V_j^-$ are the positive and negative components, respectively.
When $d\rightarrow+\infty$ we claim that (1) $\|\mathbf V_j^+\|^2\approx\frac{1}{d}$, and (2) $\mathbf V_j\rightarrow\mathbf 0$. 
\end{theorem}

\Cref{thm:counterexample} shows that when positive samples are overly dispersed, their induced attraction for each generated sample tends to be almost annihilated in high-dimensional spaces, which can drive the generator towards a sub-optimal solution.
Further empirical evidence is reported in \Cref{tab:latent_dispersion}, in which NN-d reports the average nearest-neighbor distance and S-MMD evaluates the maximum mean discrepancy between the sample distribution and a spherical distribution.
It is noteworthy that SD-VAE suggests severely dispersed latent space.
That is to say, to guarantee the stability of Drifting Models, a well-trained MAE, especially the one fine-tuned with classification loss, is involved to yield more concentrated semantic features.

In contrast, RAE enjoys substantially more concentrated latent spaces, suggesting that RAE could serve as a more favorable underlying latent space for Drifting Models and relieve the redundant module.
Furthermore, we note that \Cref{thm:counterexample} also indicates that poor initialization can be detrimental to Drifting Models.
Yet in distillation stage, the pretrained model itself is already a sufficiently good initialization for the generator.
Therefore in the sequel, we focus only on the distillation in RAE latent spaces via drifting dynamic.
More discussions on training Drifting Models from scratch in RAE latent space is addressed at Appendix~\ref{sec:attempts_train_from_scratch}.

\subsection{Distillation via Drifting in RAE Latent Spaces}\label{sec:3.method.3}

We now introduce our method for distilling flow-based models in RAE latent spaces using Drifting Models.
Let $\mathbf{v}_\theta(\mathbf{y}, t)$ denote a pretrained flow model, we form a one-step generator to distill as:
\begin{align}
\mathbf{G}_\theta(\mathbf{z})=\mathbf{z}-\mathbf{v}_\theta(\mathbf{z},1), 
\quad \mathbf{z}\sim\mathcal{N}(\mathbf{0},\mathbf{I}),
\end{align}
where $\theta$ is initialized from the pretrained flow model.
Given a batch of latent $\{\mathbf{y}_i\}_{i=1}^{N_\mathrm{pos}}$ sampled from the real distribution, we write
\begin{align}
\mathbf{y}_i=(\mathbf{y}_i^1,\dots,\mathbf{y}_i^c,\dots,\mathbf{y}_i^C),
\end{align}
where each $\mathbf{y}_i^c\in\mathbb{R}^D$ denotes the $c$-th patch token, $C$ is the number of patch tokens, and $D$ is the hidden size.
The token-wise output of the generator is defined analogously as $\mathbf{G}_\theta^c(\mathbf{z})$.

As suggested in \Cref{sec:3.method.2}, RAE latent spaces already provide semantically meaningful and sufficiently concentrated features for drifting-based training, requiring no auxiliary feature extractor.
Therefore, it is feasible to directly define drifting objective on each tokenof the RAE latent as
\begin{align}\label{eq:our_drifting}
L(\theta)=
\sum_{j=1}^{N_{\mathrm{neg}}}\sum_{c=1}^{C}
\left\|
\mathbf{G}_\theta^c(\mathbf{z}_j)
-
\operatorname{sg}\left[
\mathbf{G}_\theta^c(\mathbf{z}_j)+\tilde{\mathbf{V}}_j^c
\right]
\right\|^2,
\qquad 
\mathbf{z}_j\sim\mathcal{N}(\mathbf{0},\mathbf{I}),
\end{align}
where $N_\mathrm{neg}$ is the number of generated samples, $\operatorname{sg}[\cdot]$ denotes the stop-gradient operation, and $\tilde{\mathbf{V}}_j^c=\frac{\mathbf{V}_j^c}{\|\mathbf{V}_j^c\|}$ is the \textit{normalized drifting field} with $\mathbf{V}_j^c$ computed by:
\begin{align}
\mathbf{V}_j^c
\!=&
\sum_{i=1}^{N_\mathrm{pos}}\!
\frac{
e^{-\frac{1}{\tau}{\|\mathbf{y}_i^c-\mathbf{G}_\theta^c(\mathbf{z}_j)\|}}\left(\mathbf{y}_i^c-\mathbf{G}_\theta^c(\mathbf{z}_j)\right)
}{
\sum\limits_{l=1}^{N_{\mathrm{pos}}}
e^{-\frac{1}{\tau}{\|\mathbf{y}_l^c-\mathbf{G}_\theta^c(\mathbf{z}_j)\|}}
}
-\!
\sum\limits_{k\neq j}\!
\frac{
e^{-\frac{1}{\tau}{\|\mathbf{G}_\theta^c(\mathbf{z}_k)-\mathbf{G}_\theta^c(\mathbf{z}_j)\|}}\left(\mathbf{G}_\theta^c(\mathbf{z}_k)-\mathbf{G}_\theta^c(\mathbf{z}_j)\right)
}{
\sum\limits_{l\neq j}
e^{-\frac{1}{\tau}{\|\mathbf{G}_\theta^c(\mathbf{z}_l)-\mathbf{G}_\theta^c(\mathbf{z}_j)\|}}
}
.
\end{align}

Beyond the objective in \cref{eq:our_drifting}, we subsequently raise three pillars of modification to further improve the drifting dynamic.
Notably, the insights build upon theoretical perspectives of bridging Drifting Models with Diffusion-GAN~\citep{wang2023diffusiongan}.
Concretely, the drifting field can be recognized as the supervision of the optimal discriminator in GAN literature.
Detailed descriptions are located in Appendix~\ref{sec:appendix:diffusiongan}.

\paragraph{Softmax dimension.}
Recall that original Drifting Model proposed a bi-directional softmax trick which is claimed to improve training stability.
Yet it fails to exactly follow the gradient direction induced by the corresponding potential any longer.
To this end, we retain only one single softmax over sample indices during the computation of drifting field.
This naturally arises from differentiating a log-sum-exp potential, thus is more consistent with the theoretical formulation.

\paragraph{Perturbing inputs with noises.}
Original Drifting Models compute drifting field using the raw version of generated samples.
However, previous works suggest that this is highly likely to lead to unstable training and gradient vanishing due to non-intersection or transversal intersection between real data and generated manifolds~\citep{arjovsky2017a,arjovsky2017wgan}.
We therefore replace $\mathbf{G}_\theta^c(\mathbf{z}_j)$ with a slightly perturbed version:
\begin{align}
\bar{\mathbf{G}}_\theta^c(\mathbf{z}_j)
=
\mathbf{G}_\theta^c(\mathbf{z}_j)+\tau \mathbf{n}_j^c,
\end{align}
where $\mathbf{n}_j^c\in\mathbb{R}^D$ is a random-direction vector whose norm is sampled from a standard Laplace distribution, and $\tau$ is the same temperature hyperparameter used in the drifting field.
Note that this operation resembles the strategy of Diffusion-GAN~\citep{wang2023diffusiongan}, thus does not affect the training convergence while alleviating gradient vanishing and improving the training robustness.

\paragraph{Partially detaching negative samples.}
Recall that \citet{deng2026generative} concludes that enlarging $N_\text{pos}$ and $N_{\mathrm{neg}}$ is benefitial.
%
However, even when each GPU processes only one class, the per-GPU memory budget limits the number of gradient-carrying negative samples to at most $N_{\mathrm{neg}}=64$.
To further increase the number of samples with no additional GPU consumption, we propose to partially detach the generated samples.
These auxiliary samples aim to more accurately approximate the generated distribution, which improves the stability of distillation.
To summarize, by fixing $N_\text{pos}=256$ and $64$ negative samples to backpropogate, we employ $N_{\mathrm{extra\_neg}}=192$ negative samples to detach and equivalently achieve $N_\mathrm{total\_neg}=256$.

\paragraph{CLS token and registers.}
We follow~\cite{caron2021dino} to concatenate a CLS token and four extra register tokens with the pachitified sequence.
Specifically, CLS token is taken to compute the drifting loss, while the extra registers have no target and only provide redundancy.

\subsection{Further Improvements via Internal Drifting Guidance}\label{sec:3.method.4}

Unlike multi-step generators, it is well recognized that one-step generation is incompatible with extrapolation-based guided sampling.
This stems from the fact that guidance acts via intermediate predicted score functions, while one-step generator directly outputs denoised samples and involves no scores.
Nevertheless, we here propose Internal Drifing Guidance (IDG), which is to the best of our knowledge the first attempt of guided sampling on one-step generation.

Inspired by previous works~\citep{zhou2025IG,singh2026raev2}, we design the guidance with internal representations from middle Transformer block.
Then equipped with an projection head $\mathbf P_\psi$, we performance guidance during inference with guidance strength $\omega$ by
\begin{align}\label{IDG}
\mathbf{G}_\theta^{\text{IDG}}(\mathbf{z}, \omega)
=
\omega \mathbf{G}_\theta(\mathbf{z})
-
(\omega-1)\mathbf{P}_\psi(\phi(\mathbf{z})).
\end{align}
Recall that IG~\citep{zhou2025IG} introduces the same MSE loss to supervise $\mathbf P_\psi$, and assumes that with fewer Transformer blocks $\mathbf{P}_\psi(\phi(\mathbf{z}))$ is highly likely to approximate a ``bad version'' of $\mathbf{G}_\theta(\mathbf{z})$.
That is to say, IG is techniquely some variant of AutoGuidance~\citep{karras2024autoguidance}.
However, $\mathbf{P}_\psi(\phi(\mathbf{z}))$ is not \emph{theoretically guaranteed} to be weaker than $\mathbf{G}_\theta(\mathbf{z})$.
Besides, the joint training methodology might weaken the performance of the backbone.

To this end, we propose to optimize $\mathbf P_\psi$ through the drifting loss, in which $\mathbf{G}_\theta(\mathbf{z})$ serves as the positive terms.
This provides an explicit and strict constraint on the quality order between $\mathbf{P}_\psi(\phi(\mathbf{z}))$ and $\mathbf{G}_\theta(\mathbf{z})$.
Further detaching the intermediate $\phi(\mathbf{z})$ from the backbone, the guidance pipeline would have no negative influence on the distillation performance.

It is noteworthy that our IDG involves neither auxiliary noise-data pairs nor additional distillation stage, and could be implemented with barely no cost during inference.
Most notably, benefiting from the drifting objective on $\mathbf{P}_\psi$, the extrapolation in \cref{IDG} could be interpreted as a one-negative drifting from theoretical perspective, \textit{i.e.}, $\mathbf{G}_\theta(\mathbf{z})$ is drifted away from the negative $\mathbf{P}_\psi(\phi(\mathbf{z}))$.
Extensive ablations on efficacy and efficiency of IDG are addressed in~\Cref{app:idg_ablation}.

\begin{figure}[t]
    \centering
    \includegraphics[width=0.93\linewidth]{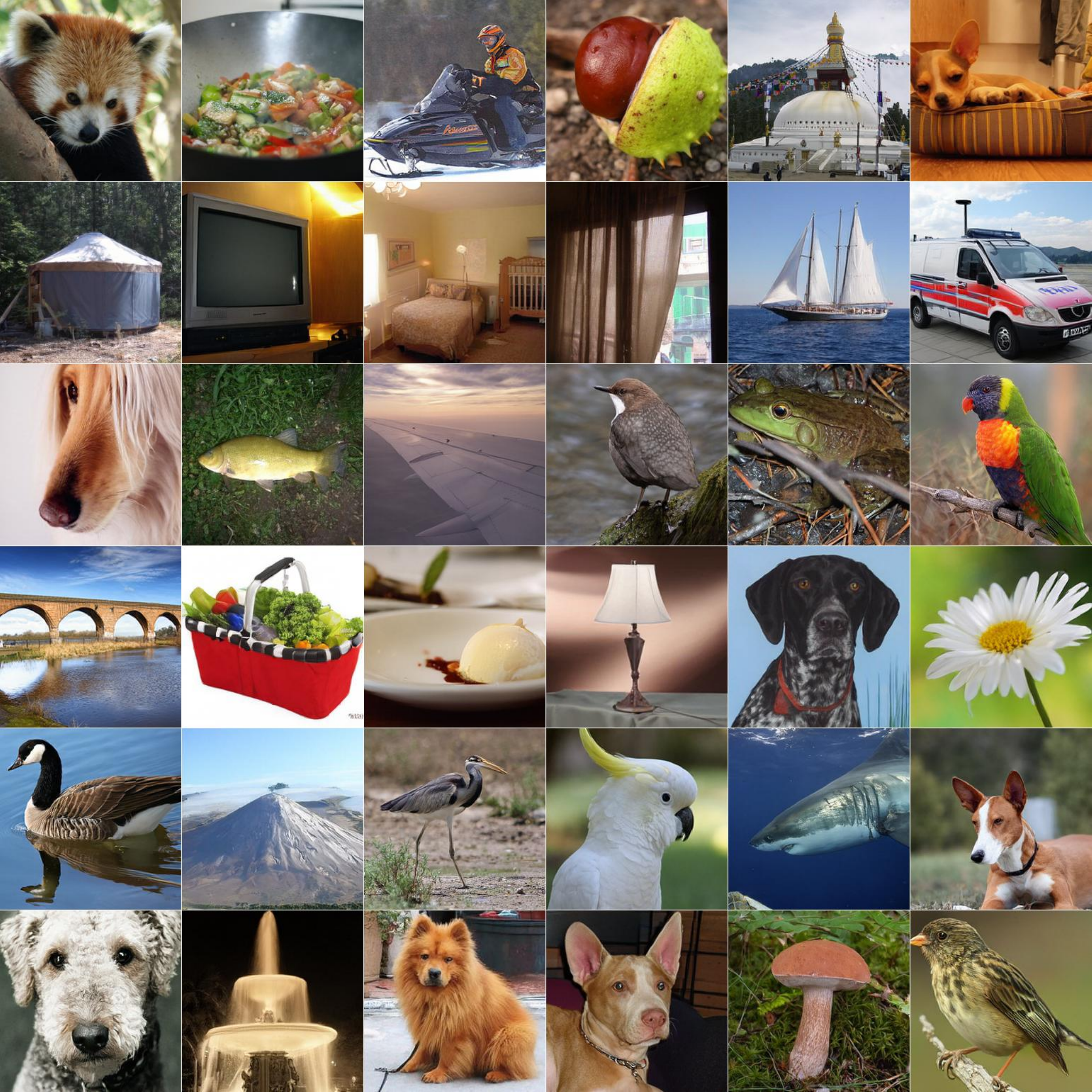}
    \caption{\textbf{Visualizations} of generated samples from distilled $\text{DiT}^\text{DH}\text{-XL}$ (FID = 1.48).}
    \label{fig:main_vis}
    \vspace{-10pt}
\end{figure}

%% file: figs/anisotropy_layout.tex
\begin{figure}[t]
\centering
\begin{minipage}[t]{0.49\linewidth}
\centering
\vspace{-5pt}
\includegraphics[width=0.98\linewidth]{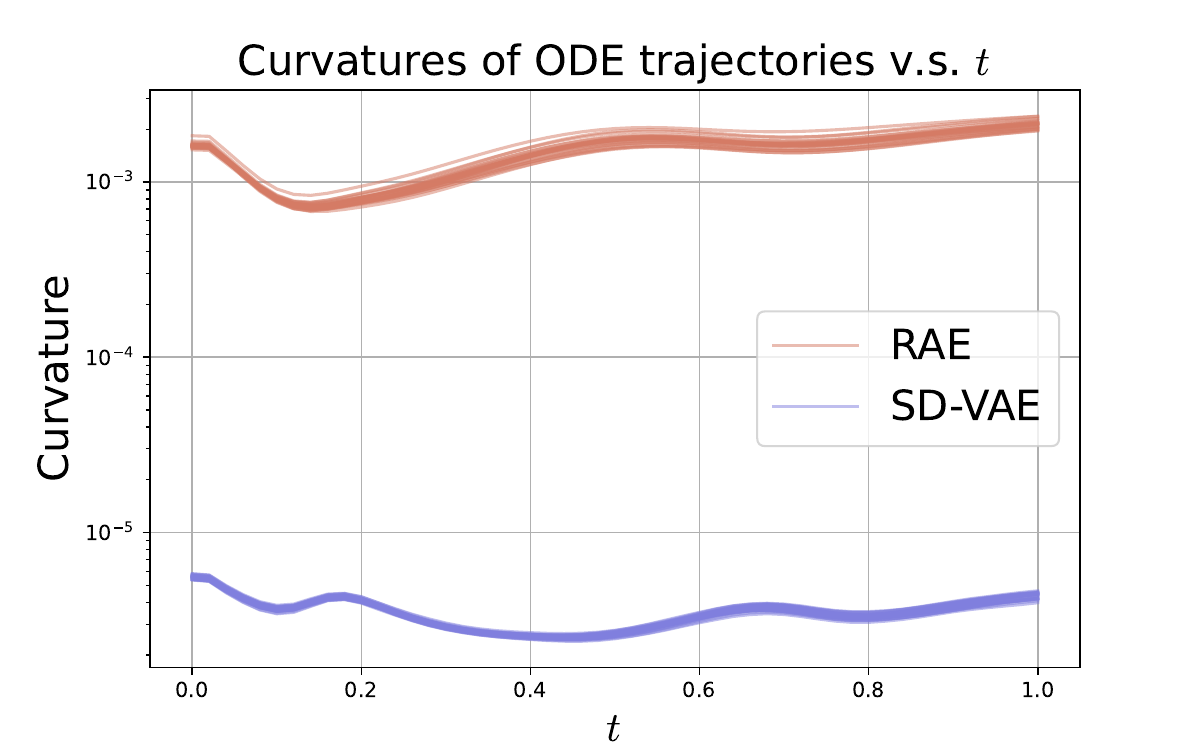}
\vspace{-2pt}
\caption{
    \textbf{Curvatures} of ODE trajectories in RAE and SD-VAE latent spaces.
}
\label{fig:curvatures}
\end{minipage}
\hfill
\begin{minipage}[t]{0.48\linewidth}
\centering
\vspace{0pt}
\captionof{table}{
    \textbf{Isotropy} statistics of latent features.
}
\label{tab:anisotropy}
\SetTblrInner{rowsep=1.55pt}                     
\SetTblrInner{colsep=12.9pt}                     
\footnotesize
\begin{tblr}{
    cell{1-3}{2-3}={halign=c,valign=m},          
    cell{1-3}{1}={halign=l,valign=m},            
    hline{1,2,4}={1-3}{1.0pt},                   
}
Latent space             & PR $(\uparrow)$ & SE $(\uparrow)$ \\
\textcolor{myblue}{SD-VAE} &      \bf 0.2068 &      \bf 0.3781 \\
\textcolor{myred}{RAE}     &          0.0630 &          0.2110 \\
\end{tblr}
\vspace{10pt}
\captionof{table}{
    \textbf{Dispersion} statistics of latent features.
}
\vspace{6pt}
\label{tab:latent_dispersion}
\SetTblrInner{rowsep=1.55pt}                     
\SetTblrInner{colsep=8.5pt}                      
\footnotesize
\begin{tblr}{
    cell{1-3}{2-3}={halign=c,valign=m},          
    cell{1-3}{1}={halign=l,valign=m},            
    hline{1,2,4}={1-3}{1.0pt},                   
}
Latent space             & NN-d $(\downarrow)$ & S-MMD $(\uparrow)$ \\
\textcolor{myblue}{SD-VAE} &            4.3083 &              0.00576 \\
\textcolor{myred}{RAE}     &        \bf 1.0664 &          \bf 0.07913 \\
\end{tblr}
\end{minipage}
\vspace{-4pt}
\end{figure}

%% file: sections/4.exp.tex
\section{Experiments}\label{sec:4.exp}

We evaluate the proposed method on class-conditional image generation with representation autoencoders. 
\Cref{sec:4.exp.1} describes the experimental setup, with detailed hyperparameter settings provided in Appendix~\ref{app:hyperparameters}. 
\Cref{sec:4.exp.2} presents the main results, and \Cref{sec:4.exp.3} provides ablation studies on the modifications introduced in \Cref{sec:3.method.3}.

\subsection{Experimental Setup}\label{sec:4.exp.1}

\noindent\textbf{Dataset and pretrained checkpoints.}
We evaluate our method on ImageNet $256\times256$~\citep{deng2009imagenet} using $\text{DiT}^{\text{DH}}\text{-XL}$ and $\text{DiT}^{\text{DH}}\text{-L}$ from the official RAE codebase\footnote{https://github.com/bytetriper/RAE}~\citep{zheng2025diffusion}.
Specifically, we directly use the released $\text{DiT}^{\text{DH}}\text{-XL}$ checkpoint, and train $\text{DiT}^{\text{DH}}\text{-L}$ ourselves strictly following the official implementation.

\noindent\textbf{Evaluation metrics.}
We report FID~\citep{heusel2017fid} to evaluate generation quality.
In addition, we use Precision and Recall~\citep{kynkaanniemi2019pr} to measure sample fidelity and diversity, respectively.
All metrics are computed using $50{,}000$ generated samples.
We adopt class-balanced sampling~\cite{zheng2025diffusion}, \textit{i.e.}, generating $50$ images for each of the $1{,}000$ ImageNet classes.

\noindent\textbf{Training configurations.}
Following~\cite{deng2026generative}, the Drifting field is computed at three temperature values, $\{0.02, 0.05, 0.2\}$, and the final objective is obtained by averaging the corresponding losses. 
All tokens are normalized by the average norm of real samples within the minibatch.
We fix $N_{\mathrm{class}}=32$ and train for $10{,}000$ steps.
AdamW is used with $\beta_1=0.9$, $\beta_2=0.95$, and weight decay is set to $0$.
For stable distillation in RAE latent spaces, we linearly decay the learning rate from $3\times10^{-5}$ to $3\times10^{-7}$ over training and apply gradient clipping with a threshold of $1.0$. For IDG, we concatenate the features from layers $8$ and $14$, and then apply a projection head consisting of a three-layer MLP with SiLU activations. 
The guidance weight is default to $\omega=1.4$.
We retain the CLS tokens produced by the DINO encoder as positive samples for the generated CLS tokens.
Further implementation details are available in Appendix~\ref{app:hyperparameters}.

\subsection{Main Results}\label{sec:4.exp.2}

\input{tables/configABC}

\Cref{tab:modifications} reports the practical effects of the modifications introduced in~\Cref{sec:3.method.3} with $\text{DiT}^\text{DH}\text{-XL}$. Here the baseline uses the additional $y$-softmax and sets $N_{\mathrm{pos}}=N_{\mathrm{neg}}=64$ following the original Drifting Model~\citep{deng2026generative}.
Perturbing inputs with noises and removing the additional $y$-softmax make the implementation more consistent with theoretical analysis, while also improving empirical performance. 
Further increasing the numbers of positive and negative samples yields the best result.

Final generation results are reported in Table~\ref{tab:main_results}. 
Our proposed \method{} achieves an FID of $1.82$ with $\text{DiT}^\text{DH}\text{-L}$ and $1.48$ with $\text{DiT}^\text{DH}\text{-XL}$ using a single generation step, outperforming the previous distillation method in RAE latent spaces, while remaining competitive with distilled models in other spaces.
Moreover, \method{} reaches an FID of $1.48$ within only $16$ epochs, demonstrating favorable training efficiency.
These results suggest that Drifting provides an effective framework for distilling flow-based models trained in representation spaces.

Compared with the original Drifting Model, \method{} achieves comparable FID while eliminating the need for an auxilary MAE. 
This is consistent with our analysis in \Cref{thm:counterexample}: the auxiliary MAE in the original Drifting Model helps mitigate the effect of overly dispersed latents, while the more compact RAE latent space provides a more favorable geometry for Drifting-based distillation. Qualitative results are shown in~\Cref{fig:main_vis}, with more visualizations provided in Appendix~\ref{sec:more_visualizations}.

\input{tables/main_results_decrease_by_NFE}

\subsection{Ablation Studies}\label{sec:4.exp.3}

\noindent\textbf{Input perturbation and softmax dimension.}
As shown in~\Cref{tab:ablation_softmax_and_noise}, 
removing the $y$-softmax and including input perturbations both lead to improvement over the original drifting baseline. 
Simultaneously applying both modifications yields similar performance as solely removing the $y$-softmax, while we observe \textbf{more stable distillation} with input perturbation. 
We thus apply both modifications in all our experiments.

\noindent\textbf{Increasing the number of samples.}
\Cref{tab:ablation_num_samples} study the effect of increasing the number of positive and negative samples used to estimate the Drifting field. 
Increasing only the number of positive samples or only the number of negative samples leads to substantial performance degradation, while the best performance is achieved when positive samples and negative samples are balanced. 
This observation suggests that an imbalanced estimation of the attractive and repulsive components can bias the resulting Drifting direction.

\noindent\textbf{Increasing the number of gradient-carrying negative samples.}
With $N_{\mathrm{total\_neg}}=256$ fixed, we further vary the number of gradient-carrying negatives as $N_{\mathrm{neg}}\in\{64,128,192,256\}$. 
As shown in \Cref{tab:ablation_more_bp}, increasing $N_{\mathrm{neg}}$ does not bring performance gains.
This indicates that the main benefit of using more negative samples comes from improving the empirical approximation of the generated distribution, rather than from applying gradients to more generated samples. 
Such conclusion is also implementation-friendly, as detached negative samples require no gradient storage.

\begin{figure}[t]
    \centering
    \includegraphics[width=0.44\linewidth]{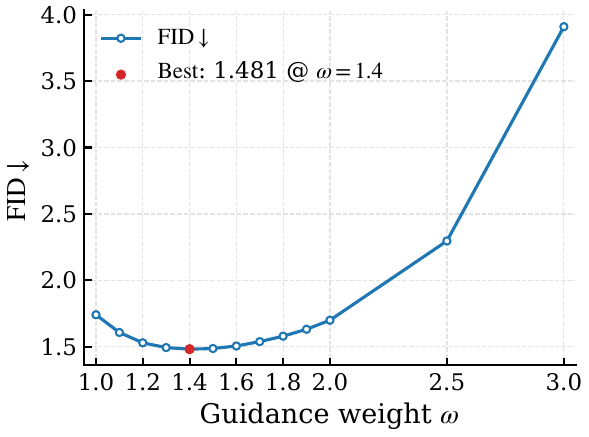}\hfill
    \includegraphics[width=0.44\linewidth]{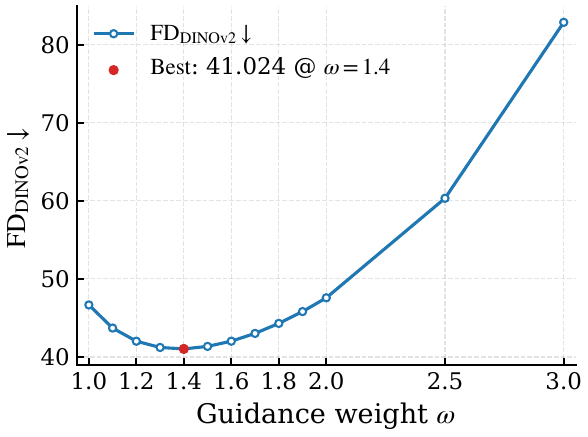}
    \vspace{-4pt}
    \caption{Ablation on the guidance weight $\omega$ of IDG.}
    \label{fig:guidance_weight_ablation}
\end{figure}

\noindent\textbf{Guidance weight.}
\Cref{fig:guidance_weight_ablation} reports FID and $\mathrm{FD}_{\mathrm{DINOv2}}$ under different guidance weights.
Similar to other guidance methods, both FID and $\text{FD}_\text{DINOv2}$ first decrease and then increase as the guidance weight increases, with the optimal point for both metrics occurring at $\omega=1.4$.



\input{tables/ablation_BandC}

\subsection{Limitations and Future Work}

Although this work eliminates the need for an additional MAE for Drifting-based distillation in RAE spaces, training Drifting Models from scratch without an auxiliary MAE remains an open problem.
Developing native Drifting training methods that do not rely on auxiliary modules is an important direction. 
Moreover, the need for many same-class positive samples at each Drifting update may hinder scalability, especially in text-to-image generation settings.
Reducing this dependence on abundant positive samples is also a promising direction for future research.


%% file: tables/configABC.tex
\setlength{\columnsep}{16pt}
\begin{wraptable}{R}{0.5\textwidth}
\centering
\vspace{-9pt}
\caption{
Effect of modifications proposed in \Cref{sec:3.method.3}.
}
\vspace{-3pt}
\label{tab:modifications}
\footnotesize
\SetTblrInner{rowsep=1.75pt}                     
\SetTblrInner{colsep=3.0pt}                      
\begin{tblr}{
  colspec={c l c},
  row{1}={font=\bfseries},
  hline{1,Z}={1pt},
  hline{2}={0.5pt},
}
Config & Modifications & FID ($\downarrow$) \\
A & Baseline & 2.00 \\
B & $+$ Input pertubation, $-$ $y$-softmax & 1.86 \\
C & $+$ 192 pos., $+$ 192 detached neg.  & \textbf{1.74} \\
\end{tblr}
\vspace{-5pt}
\end{wraptable}

%% file: tables/main_results_decrease_by_NFE.tex
\begin{table*}[t]
\centering
\caption{
Main results on ImageNet $256\times256$.
$^\dag$ indicates distillation methods. 
For each group separately, i.e., vanilla diffusion models, few-step models, and \textit{one-step} models, \textbf{bold} indicates the best result, and \underline{underlining} indicates the second-best result.
}
\vspace{-4pt}
\label{tab:main_results}
\footnotesize
\SetTblrInner{rowsep=1.0pt}                     
\SetTblrInner{colsep=6.0pt}                      
\begin{tblr}{
    cell{1-28}{2-7}={halign=c,valign=m},         
    cell{1-28}{1}={halign=l,valign=m},           
    hline{1-2,29}={1-7}{1.0pt},                  
    hline{7,14}={1-7}{},                  
    hline{27}={1-7}{dotted},                  
}
\textbf{Method}                                            & \textbf{Space} & NFE & Epochs & FID ($\downarrow$) & Prec. ($\uparrow$) & Rec. ($\uparrow$) \\
ADM-G~\citep{dhariwal2021adm}                              & Pixel          & 250 &    400 &               4.59 &   \underline{0.82} &              0.52 \\
DiT-XL/2~\citep{peebles2023dit}                            & VAE            & 250 &   1400 &               2.27 &           \bf 0.83 &              0.57 \\
SiT-XL/2~\citep{ma2024sit}                                 & VAE            & 250 &   1400 &               2.06 &   \underline{0.82} &              0.59 \\
LightningDiT-XL~\citep{Yao2025ReconstructionVG}            & VAE            & 250 &    800 &   \underline{1.35} &               0.79 &  \underline{0.65} \\
DiT$^{\mathrm{DH}}$-XL~\citep{zheng2025diffusion}          & RAE            &  50 &    800 &           \bf 1.13 &               0.78 &          \bf 0.67 \\
IMM-XL/2 ($\omega=1.5$)~\citep{zhou2025inductive}          & VAE            &   8 &   3837 &               1.99 &                 -- &                -- \\
STEI$^\dag$~\citep{liu2026learning}                        & VAE            &   8 &     20 &               1.96 &                 -- &                -- \\
MeanFlow-XL/2+~\citep{geng2025meanflow}                    & VAE            &   2 &   1000 &               2.20 &                 -- &                -- \\
iMF-XL/2~\citep{geng2025improved}                          & VAE            &   2 &    800 &   \underline{1.61} &               \textbf{0.79} &              \textbf{0.63} \\
$\pi$-Flow$^\dag$~\citep{piflow}                           & VAE            &   2 &     76 &               1.97 &                 -- &                -- \\
FACM$^\dag$~\citep{peng2026facm}                           & VAE            &   2 &     60 &           \bf 1.32 &                 -- &                -- \\
MF-RAE$^\dag$~\citep{hu2025raemeanflow}                    & RAE            &   2 &     41 &               1.89 &                 -- &                -- \\
BigGAN~\citep{brock2019biggan}                             & Pixel          &   1 &     -- &               6.95 &      \textbf{0.89} &              0.38 \\
StyleGAN-XL~\citep{sauer2022styleganxl}                    & Pixel          &   1 &     -- &               2.30 &               0.78 &              0.53 \\
Drifting Model, L/16~\citep{deng2026generative}            & Pixel          &   1 &    640 &               1.61 &   \underline{0.81} &              0.60 \\
MeanFlow-H/16~\citep{lu2026one}                            & Pixel          &   1 &    320 &               2.29 &               0.80 &              0.59 \\
PaGoDa$^\dag$~\citep{Kim2024pagoda}                        & Pixel          &   1 &     -- &               1.56 &                 -- &              0.59 \\
IMM-XL/2 ($\omega=1.5$)~\citep{zhou2025inductive}          & VAE            &   1 &   3837 &               8.05 &                 -- &                -- \\
STEI$^\dag$~\citep{liu2026learning}                        & VAE            &   1 &     20 &               7.12 &                 -- &                -- \\
MeanFlow-XL/2~\citep{geng2025meanflow}                     & VAE            &   1 &    240 &               3.43 &                 -- &                -- \\
iMF-XL/2~\citep{geng2025improved}                          & VAE            &   1 &    800 &               1.82 &               0.78 &              0.63 \\
Drifting Model, L/2~\citep{deng2026generative}             & VAE            &   1 &   1280 &               1.54 &               0.79 &              0.63 \\
$\pi$-Flow$^\dag$~\citep{piflow}                           & VAE            &   1 &    448 &               2.85 &                 -- &                -- \\
FreeFlow-XL/2$^\dag$~\citep{tong2025freeflow}              & VAE            &   1 &    300 &      \textbf{1.45} &                 -- &                -- \\
MF-RAE$^\dag$~\citep{hu2025raemeanflow}                    & RAE            &   1 &     41 &               2.03 &                 -- &                -- \\
\method (DiT$^{\mathrm{DH}}$-L)$^\dag$                     & RAE            &   1 &     16 &               1.82 &               0.77 &  \underline{0.64} \\
\method (DiT$^{\mathrm{DH}}$-XL)$^\dag$                    & RAE            &   1 &     16 &   \underline{1.48} &               0.77 &     \textbf{0.66} \\
\end{tblr}
\vspace{-10pt}
\end{table*}

%% file: tables/ablation_BandC.tex
\begin{table}[t]
    \label{tab:ablation_configs}
    \begin{minipage}[t]{0.48\linewidth}
    \centering
    \captionof{table}{
        Ablation on input pertubation and removing $y$-softmax.
    }
    \vspace{-9pt}
    \label{tab:ablation_softmax_and_noise}
    \input{tables/ablation_configB}
    \vspace{0pt}
    \centering
    \captionof{table}{
        Ablation on $N_{\mathrm{extra\_neg}}$.
    }
    \vspace{0pt}
    \label{tab:ablation_more_bp}
    \input{tables/ablation_more_bp}
    \end{minipage}
    \hfill
    \begin{minipage}[t]{0.48\linewidth}
        \centering
        \captionof{table}{
            Ablation on different combinations of $N_{\mathrm{pos}}$, $N_{\mathrm{total\_neg}}$, and $N_{\mathrm{extra\_neg}}$.
            We highlight the \textbf{balanced} configurations with $N_{\mathrm{pos}}=N_{\mathrm{total\_neg}}$ in color \textcolor{gray}{gray}.
        }
        \vspace{-2pt}
        \label{tab:ablation_num_samples}
        \input{tables/ablation_configC}
    \end{minipage}
\end{table}

%% file: tables/ablation_configB.tex
\footnotesize
\SetTblrInner{rowsep=0.4pt}                     
\SetTblrInner{colsep=4.3pt}                      
\begin{tblr}{
  hline{1,2,Z}={1pt},
  cell{1-5}{1-3}={halign=c,valign=m},         
}
Input Pertubation             & Remove $y$-softmax            & FID ($\downarrow$) \\
\textcolor{red}{\ding{55}}    & \textcolor{red}{\ding{55}}    & 2.00               \\
\textcolor{red}{\ding{55}}    & \textcolor{green}{\checkmark}   & 1.90     \\
\textcolor{green}{\checkmark} & \textcolor{red}{\ding{55}}    & 1.88               \\
\textcolor{green}{\checkmark} & \textcolor{green}{\checkmark}   & \textbf{1.86}      \\
\end{tblr}

%% file: tables/ablation_more_bp.tex
\footnotesize
\SetTblrInner{rowsep=0.4pt}                     
\SetTblrInner{colsep=8.0pt}                      
\begin{tblr}{
  hline{1,2,Z}={1pt},
  cell{1-5}{1-4}={halign=c,valign=m},         
  cell{2}{1,2}={r=4}{},                        
}
$N_{\mathrm{pos}}$ & $N_{\mathrm{total\_neg}}$ & $N_\mathrm{extra\_neg}$ & FID ($\downarrow$) \\
256                & 256                       & 192                     &               1.74 \\
                   &                           & 128                     &               1.77 \\
                   &                           & 64                      &               1.75 \\
                   &                           & 0                       &               1.74 \\
\end{tblr}

%% file: tables/ablation_configC.tex
\footnotesize
\SetTblrInner{rowsep=1.45pt}                     
\SetTblrInner{colsep=8.0pt}                      
\begin{tblr}{
  cell{1-8}{1-4}={halign=c,valign=m},          
  hline{1,2,Z}={1pt},
  hline{5,7}={0.5pt},
  cell{2,6,8}{2-4}={bg=lightgray!35},          
  cell{2}{1}={r=3}{},                          
  cell{5,7}{1}={r=2}{},                        
}
$N_{\mathrm{pos}}$ & $N_{\mathrm{total\_neg}}$ & $N_\mathrm{extra\_neg}$ & FID ($\downarrow$) \\
64                 & 64                        & 0                       &               1.86 \\
                   & 128                       & 64                      &               2.74 \\
                   & 256                       & 192                     &               4.66 \\
128                & 64                        & 0                       &               2.20 \\
128                & 128                       & 64                      &               1.82 \\
256                & 64                        & 0                       &               4.49 \\
256                & 256                       & 192                     &      \textbf{1.74} \\
\end{tblr}

%% file: sections/5.conclusion.tex
\section{Conclusion}\label{sec:5.conclusion}

In this paper, we propose Drifting-based distillation for flow-based models in RAE latent spaces. 
We first quantitatively analyze the geometry of RAE latent spaces and theoretically study the dynamics of Drifting, showing that the highly curved ODE trajectories in RAE spaces make trajectory-based distillation challenging, while their semantically concentrated representations allow Drifting to operate without an auxilary MAE. 
Motivated by a connection between Drifting Models and the Diffusion-GAN framework, we introduce several practical modifications that improve distillation performance and the first extrapolation-based guided sampling method for one-step generation with barely no cost.
Experiments on ImageNet $256\times256$ demonstrate that our method outperforms previous distillation methods in RAE latent spaces and achieves performance comparable to the original Drifting Model and distilled models in other spaces.

%% file: sections/6.ref.tex
{
\bibliographystyle{iclr2027_conference}
\bibliography{ref.bib}
}

%% file: sections/8.appendix.tex
\section*{Appendix}\label{sec:appendix}

In this appendix, we provide additional technical details and discussions omitted from the main text. 
Appendix~\ref{sec:proof} presents the detailed proof of~\Cref{thm:counterexample} in~\Cref{sec:3.method.2}, and establishes the connection between Drifting Models and Diffusion-GAN. 
Appendix~\ref{sec:implementation_details} provides further implementation details and more ablation studies. Appendix~\ref{sec:attempts_train_from_scratch} presents additional attempts and discussions on training Drifting Models from scratch.
Appendix~\ref{sec:more_visualizations} presents additional qualitative generation results.

\section{Proofs and Derivatives}\label{sec:proof}

\subsection{Proof of~\Cref{thm:counterexample}}\label{sec:proof.1}

\begin{proof}
Note that the drifting term $\mathbf V_j$ can be reformulated as below:
\begin{align}
\mathbf V_j&=\mathbf V_j^+-\mathbf V_j^-,
\end{align}
\begin{align}
\mathbf V_j^+&=\sum_{i=1}^d\frac{e^{-\frac{1}{\tau}\|\mathbf y_i-\mathbf x_j\|}}{\sum\limits_{l=1}^de^{-\frac{1}{\tau}\|\mathbf y_l-\mathbf x_j\|}}\mathbf y_i, \quad\mathbf V_j^-=\sum_{k\neq j}\frac{e^{-\frac{1}{\tau}\|\mathbf x_k-\mathbf x_j\|}}{\sum\limits_{m\neq j}e^{-\frac{1}{\tau}\|\mathbf x_m-\mathbf x_j\|}}\mathbf x_k.
\end{align}
We first compute the effect of the negative part $\mathbf V_j^-$.
To simplify the derivation, we could reformulate the case, \textit{i.e.}, for uniformly sampled $\mathbf z_1,\mathbf z_2,\cdots,\mathbf z_m,\mathbf x \stackrel{\text{i.i.d.}}{\sim} \mathcal U_{[-r,r]^d}$, we compute the behavior of the following $\mathbf z$:
\begin{align}
\mathbf z=\sum_{i=1}^m\frac{e^{-\frac{1}{\tau}\|\mathbf z_i-\mathbf x\|}}{\sum\limits_{l=1}^me^{-\frac{1}{\tau}\|\mathbf z_l-\mathbf x\|}}\mathbf z_i.
\end{align}
Note that $\|\mathbf z_i-\mathbf x\|=\sqrt{\|\mathbf x\|^2+\|\mathbf z_i\|^2-2\langle\mathbf x,\mathbf z_i\rangle}$.
Let $R_i=\sqrt{\|\mathbf x\|^2+\|\mathbf z_i\|^2}$ and $u_i=\langle\mathbf x,\mathbf z_i\rangle$, then by Taylor's series we have
\begin{align}\label{eq:taylor_general}
\|\mathbf z_i-\mathbf x\|&=R_i\sum_{k=0}^{\infty}\binom{1/2}{k}\left(\frac{-2u_i}{R_i^2}\right)^k=R_i+R_i\sum_{k=1}^{\infty}\binom{1/2}{k}\left(\frac{-2u_i}{R_i^2}\right)^k,
\end{align}
where the convergence radius is $\left|\frac{2u_i}{R_i^2}\right|<1$.
Note that
\begin{align}
\left|\frac{2u_i}{R_i^2}\right|=\left|\frac{2\langle\mathbf x,\mathbf z_i\rangle}{\|\mathbf x\|^2+\|\mathbf z_i\|^2}\right|\leqslant\frac{2\|\mathbf x\|\|\mathbf z_i\|}{\|\mathbf x\|^2+\|\mathbf z_i\|^2}\leqslant1,
\end{align}
and the equality holds if and only if $\mathbf x=\pm\mathbf z_i$.
That is to say,~\cref{eq:taylor_general} holds almost everywhere.
Then we have
\begin{align}
e^{-\frac{1}{\tau}\|\mathbf z_i-\mathbf x\|}&=e^{-\frac{R_i}{\tau}}e^{-\frac{R_i}{\tau}\sum\limits_{k=1}^{\infty}\binom{1/2}{k}(\frac{-2u_i}{R_i^2})^k} \\
&=e^{-\frac{R_i}{\tau}}(1+O(\frac{u_i}{R_i})).
\end{align}
Note that for uniformly sampled $\mathbf z_i$, we have $\mathbb E\|\mathbf z_i\|^2=\frac{d}{3}r^2$ with standard deviation $\frac{2r^2}{3}\sqrt{\frac{d}{5}}$, and $\mathbb E[u_i]=\langle \mathbf x,\mathbb E[\mathbf z_i]\rangle=0$ with standard deviation $\frac{\|\mathbf x\|}{\sqrt3}r$.
Since the second standard deviation is independent with dimension $d$, one can deduce that $\frac{u_i}{R_i}\approx(\frac{1}{\|\mathbf x\|^2+\frac{dr}{3}})^{\frac{1}{2}}u_i\rightarrow0$ as $d$ goes to infinity.
Therefore we have
\begin{align}
\mathbf z&=\sum_{i=1}^m\frac{e^{-\frac{1}{\tau}\|\mathbf z_i-\mathbf x\|}}{\sum\limits_{l=1}^me^{-\frac{1}{\tau}\|\mathbf z_l-\mathbf x\|}}\mathbf z_i\approx\sum_{i=1}^m\frac{e^{-\frac{R_i}{\tau}}}{\sum\limits_{l=1}^me^{-\frac{R_l}{\tau}}}\mathbf z_i \\
&\rightarrow\sum_{i=1}^m\frac{e^{-\frac{1}{\tau}\sqrt{\frac{d}{3}r^2+\|\mathbf x\|^2}}}{\sum\limits_{l=1}^me^{-\frac{1}{\tau}\sqrt{\frac{d}{3}r^2+\|\mathbf x\|^2}}}\mathbf z_i \\
&=\frac{1}{m}\sum_{i=1}^m\mathbf z_i.
\end{align}
Then the mean and standard deviation of $\|\mathbf z\|$ can be deduced as below by Central Limit Theorem:
\begin{align}
\mathbb E\|\mathbf z\|&\rightarrow\sqrt{\frac{d}{3m}}r,\;\mathrm{std}(\|\mathbf z\|)\rightarrow\frac{r}{\sqrt{6m}}\quad\text{as }d\rightarrow+\infty.
\end{align}
Therefore, when $m=d-1$, the standard deviation will tend to zero, and we have
\begin{align}
\left\|\mathbf V_j^-\right\|\rightarrow\sqrt{\frac{1}{3}}r\quad\text{as }d\rightarrow+\infty.
\end{align}
That is to say, $\left\|\mathbf V_j^-\right\|$ converges to $\sqrt{\frac{1}{3}}r$ which is independent with the dimension $d$.

As for the positive part $\mathbf V_j^+$, note that for any $\mathbf x\in\mathbb R^d$, we have $\|\mathbf y_i-\mathbf x\|=\sqrt{1+\|\mathbf x\|^2-2\langle \mathbf y_i,\mathbf x_i\rangle}$.
Let $R=\sqrt{1+\|\mathbf x\|^2}$ and $u_i=\langle \mathbf y_i,\mathbf x\rangle$, then we have similar equality which holds for any $\mathbf x\neq\pm\mathbf y_i$:
\begin{align}
e^{-\frac{1}{\tau}\|\mathbf y_i-\mathbf x\|}=e^{-\frac{R}{\tau}}(1+O(\frac{u_i}{R})).
\end{align}
Denote by
\begin{align}
\mathbf y=\sum_{i=1}^d\frac{e^{-\frac{1}{\tau}\|\mathbf y_i-\mathbf x\|}}{\sum\limits_{l=1}^de^{-\frac{1}{\tau}\|\mathbf y_l-\mathbf x\|}}\mathbf y_i.
\end{align}
Recall that $\mathbf y_i$ is uniformly sampled from $\mathbb S^{d-1}$, then $\|\mathbf y_i\|=1$ and $\mathbb E[u_i]=\langle\mathbf x,\mathbb E[\mathbf y_i]\rangle=0$ with standard deviation $\frac{\|\mathbf x\|}{\sqrt d}$.
Since $\frac{\|\mathbf x\|}{\sqrt d}$ tends to zero as $d\rightarrow\infty$, we can still deduce that $\frac{u_i}{R}\rightarrow0$.
Then we have
\begin{align}
\mathbf y&=\sum_{i=1}^d\frac{e^{-\frac{1}{\tau}\|\mathbf y_i-\mathbf x\|}}{\sum\limits_{l=1}^de^{-\frac{1}{\tau}\|\mathbf y_l-\mathbf x\|}}\mathbf y_i\approx\sum_{i=1}^d\frac{e^{-\frac{R}{\tau}}}{\sum\limits_{l=1}^de^{-\frac{R}{\tau}}}\mathbf y_i \\
&=\frac{1}{d}\sum_{i=1}^d\mathbf y_i.
\end{align}
Note that
\begin{align}
\mathbb E\|\mathbf y\|^2=\frac{1}{d},\quad\mathrm{var}(\|\mathbf y\|^2)=\frac{2(d-1)}{d^4}.
\end{align}
Therefore $\mathbf y\rightarrow\mathbf0$ for almost any $\mathbf x\in\mathbb R^d$ as the dimension $d$ goes to infinity.
That is to say, the positive part directly vanishes.

Recall that $\left\|\mathbf V_j^-\right\|\rightarrow\sqrt{\frac{1}{3}}r$ as $d$ goes to infinity, we can deduce that
\begin{align}
\|\mathbf V_j\|\rightarrow\sqrt{\frac{1}{3}}r\quad\text{as }d\rightarrow+\infty.
\end{align}
Note that the length of the diagonal of $[-r,r]^d$ is $r\sqrt{d}$, and $\frac{\frac{1}{3}}{d}\rightarrow0$ when $d$ goes to infinity.
Therefore we can deduce that the optimizing target of the drifting field collapses to the origin with sufficiently large dimension $d$.
\end{proof}

\subsection{Drifting as Empirical Diffusion-GAN}\label{sec:appendix:diffusiongan}

In this section, we connect Drifting Models to adversarial training, especially the viewpoint of Diffusion-GAN~\citep{wang2023diffusiongan}.
The key observation is that, after smoothing the real and generated empirical distributions with a kernel, the Drifting field can be interpreted as the gradient of the logit of the optimal discriminator.
%

Let $q$ denote the target distribution and $p={\mathbf{G}_\theta}_\#p_{\mathbf z}$ denote the generated distribution where $p_{\mathbf{z}}$ is the noise distribution.
For a fixed generator, the optimal discriminator for the standard GAN objective is
\begin{align}
D^*_{q,p}(\mathbf x)=\frac{q(\mathbf x)}{q(\mathbf x)+p(\mathbf x)},
\end{align}
whose logit is
\begin{align}
\operatorname{logit}D^*_{q,p}(\mathbf x)
=
\log\frac{D^*_{q,p}(\mathbf x)}{1-D^*_{q,p}(\mathbf x)}
=
\log q(\mathbf x)-\log p(\mathbf x).
\end{align}
Therefore, the gradient of the optimal discriminator logit gives the score difference
\begin{align}\label{eq:optimal_discriminator_score_diff}
\nabla_{\mathbf x}\operatorname{logit}D^*_{q,p}(\mathbf x)
=
\nabla_{\mathbf x}\log q(\mathbf x)-\nabla_{\mathbf x}\log p(\mathbf x).
\end{align}
The gradient of the non-saturating generator loss $-\log D^*_{q,p}(\mathbf x)$ is the negative score difference up to a positive scalar factor, since
\begin{align}
\nabla_{\mathbf x}\left[-\log D^*_{q,p}(\mathbf x)\right]
=
\frac{p(\mathbf x)}{q(\mathbf x)+p(\mathbf x)}
\left(
\nabla_{\mathbf x}\log p(\mathbf x)
-
\nabla_{\mathbf x}\log q(\mathbf x)
\right).
\end{align}

We now consider empirical distributions smoothed by a kernel.
Given real samples $\{\mathbf y_i\}_{i=1}^{N}\sim q$ and generated samples $\{\mathbf x_j\}_{j=1}^{M}\sim p$, define empirical measures
\begin{align}
\hat q=\frac{1}{N}\sum_{i=1}^{N}\delta_{\mathbf y_i},
\qquad
\hat p=\frac{1}{M}\sum_{j=1}^{M}\delta_{\mathbf x_j}.
\end{align}
For $l>0$, consider the exponential kernel
\begin{align}
k_l(\mathbf x;\tau)=\exp\left(-\frac{\|\mathbf x\|_2^l}{\tau}\right),
\end{align}
and the smoothed empirical densities
\begin{align}
\hat q_l(\mathbf x)=(k_l*\hat q)(\mathbf x)
=\frac{1}{N}\sum_{i=1}^{N}k_l(\mathbf x-\mathbf y_i;\tau),
\quad
\hat p_l(\mathbf x)=(k_l*\hat p)(\mathbf x)
=\frac{1}{M}\sum_{j=1}^{M}k_l(\mathbf x-\mathbf x_j;\tau).
\end{align}
Then the following proposition establishes the connection of the drifting field and the gradient of the logit of the optimal discriminator, which is also closely related to the results in~\cite{lai2026unifiedviewscorebaseddrifting}.

\begin{proposition}\label{prop:drifting_diffusion_gan}
Let
\begin{align}
\mathbf V_l(\mathbf x)
=
\sum_{i=1}^{N}
\alpha_i^+(\mathbf x)
\|\mathbf y_i-\mathbf x\|_2^{l-2}
(\mathbf y_i-\mathbf x)
-
\sum_{j=1}^{M}
\alpha_j^-(\mathbf x)
\|\mathbf x_j-\mathbf x\|_2^{l-2}
(\mathbf x_j-\mathbf x),
\end{align}
where
\begin{align}
\alpha_i^+(\mathbf x)
=
\frac{k_l(\mathbf x-\mathbf y_i;\tau)}
{\sum_{n=1}^{N}k_l(\mathbf x-\mathbf y_n;\tau)},
\qquad
\alpha_j^-(\mathbf x)
=
\frac{k_l(\mathbf x-\mathbf x_j;\tau)}
{\sum_{m=1}^{M}k_l(\mathbf x-\mathbf x_m;\tau)}.
\end{align}
Then
\begin{align}
\nabla_{\mathbf x}\operatorname{logit}D^*_{\hat q_l,\hat p_l}(\mathbf{x})
=
\frac{l}{\tau}\mathbf V_l(\mathbf x).
\end{align}
\end{proposition}

\begin{proof}
We first compute the score of $\hat q_l$.
Since
\begin{align}
\nabla_{\mathbf x} k_l(\mathbf x-\mathbf y_i;\tau)
=
-\frac{l}{\tau}
\|\mathbf x-\mathbf y_i\|_2^{l-2}
(\mathbf x-\mathbf y_i)
k_l(\mathbf x-\mathbf y_i;\tau),
\end{align}
we have
\begin{align}\label{eq:logqhat}
\nabla_{\mathbf x}\log \hat q_l(\mathbf x)
&=
\frac{\sum_{i=1}^{N}\nabla_{\mathbf x}k_l(\mathbf x-\mathbf y_i;\tau)}
{\sum_{n=1}^{N}k_l(\mathbf x-\mathbf y_n;\tau)}
\\
&=
\frac{l}{\tau}
\sum_{i=1}^{N}
\frac{k_l(\mathbf x-\mathbf y_i;\tau)}
{\sum_{n=1}^{N}k_l(\mathbf x-\mathbf y_n;\tau)}
\|\mathbf y_i-\mathbf x\|_2^{l-2}
(\mathbf y_i-\mathbf x)
\\
&=
\frac{l}{\tau}
\sum_{i=1}^{N}
\alpha_i^+(\mathbf x)
\|\mathbf y_i-\mathbf x\|_2^{l-2}
(\mathbf y_i-\mathbf x).
\end{align}
Analogously,
\begin{align}\label{eq:logphat}
\nabla_{\mathbf x}\log \hat p_l(\mathbf x)
=
\frac{l}{\tau}
\sum_{j=1}^{M}
\alpha_j^-(\mathbf x)
\|\mathbf x_j-\mathbf x\|_2^{l-2}
(\mathbf x_j-\mathbf x).
\end{align}
Subtracting~\Cref{eq:logqhat} and~\Cref{eq:logphat} to~\Cref{eq:optimal_discriminator_score_diff} gives the desired result.
\end{proof}

\Cref{prop:drifting_diffusion_gan} shows that Drifting estimates the optimal discriminator logit gradient by Monte Carlo samples.
In particular, when $l=2$, the norm factor disappears and $\mathbf V_l$ becomes the standard attraction-repulsion field induced by a Gaussian/RBF kernel, up to the constant factor $2/\tau$.
When $l=1$ as used in practice, the same derivation yields a \textit{normalized} displacement direction.
This is also consistent with the practical implementation of Drifting Models, where feature vectors are often normalized and only the direction of the drifting field is used. In high-dimensional spaces, the distance between two normalized feature vectors becomes nearly constant. 
Thus Drifting can be interpreted as an estimate of the gradient of the optimal discriminator logit between two perturbed distributions, making it closely related to the Diffusion-GAN~\citep{wang2023diffusiongan} framework.

Such connection provides the motivation for the modifications introduced in~\Cref{sec:3.method.3}:

First, the softmax weights in Drifting arise from differentiating the log-density of an exponential-kernel mixture.
Therefore, the additional $y$-softmax is not directly induced by this derivation and may alter the gradient direction of the optimal discriminator logit. 

Second, as Drifting is connected to the discriminator between perturbed distributions, the input should be sampled from the perturbed generated distribution $\hat p_l(\mathbf{x})$, which can be approximated by injecting noise into negative samples. 
However, using the theoretically matched perturbation scale can be inefficient in high-dimensional spaces, as perturbed samples may rarely stay near the clean generated samples and thus provide less direct supervision to the generator. 
In practice, we therefore use random-direction noise whose norm follows a Laplace distribution, which provides a practical trade-off between sampling from the perturbed distribution and maintaining sample efficiency.

Finally, Drifting relies on Monte Carlo estimation of the underlying distributions. 
Using more samples improves the empirical approximation of both real and generated distributions, leading to more accurate estimates of the Drifting direction.

\section{Additional Implementation Details}
\label{sec:implementation_details}

\subsection{Details of Statistical Analysis in~\Cref{sec:3.method.2}}
\label{app:statistical_analysis}
Here we provide additional details for the statistical analyses used in~\Cref{sec:3.method.2}, including trajectory curvature, isotropy, and dispersion statistics.

\noindent\textbf{Trajectory curvature.}
For the SD-VAE latent space, we use a DiT-XL model from~\cite{peebles2023dit}; for the RAE latent space, we use the DiT$^{\mathrm{DH}}$-XL model from~\cite{zheng2025diffusion}.
The curvature is computed using the open-source implementation of~\citet{chen2024trajectory}.

\noindent\textbf{Isotropy statistics.}
We quantify the isotropy of latents using the participation ratio (PR) and spectral entropy (SE).
For \textit{each class} and \textit{each spatial token}, we collect the corresponding latents across samples and compute their covariance matrix.
Let $\{\mu_i\}_{i=1}^{r}$ be the non-negative eigenvalues of this covariance matrix, where $r$ is the number of eigenvalues.
We define the normalized spectrum as
\begin{align}
p_i=\frac{\mu_i}{\sum_{j=1}^{r}\mu_j}.
\end{align}
The normalized participation ratio is computed as
\begin{align}
\mathrm{PR}
=
\frac{\left(\sum_{i=1}^{r}\mu_i\right)^2}
{r\sum_{i=1}^{r}\mu_i^2}
=
\frac{1}{r\sum_{i=1}^{r}p_i^2},
\end{align}
and the normalized spectral entropy is computed as
\begin{align}
\mathrm{SE}
=
\frac{1}{r}
\exp\left(-\sum_{i=1}^{r}p_i\log p_i\right).
\end{align}
Both metrics are normalized to lie in $[1/r,1]$, with larger values indicating a more isotropic spectrum.
We compute PR and SE separately for each class and each token, and then average over all classes and all tokens.
%
We note that a recent concurrent work~\citep{zhang2026rit} also studies representation-space geometry and reports conclusions that appear different from ours.
We suspect that the discrepancy mainly comes from the aggregation protocol.
\cite{zhang2026rit} measures global statistics after mixing samples from all classes and aggregating token positions, while our analysis is performed per class and per token.
For our purpose, the latter protocol is more preferred since DiT-type models condition on class embeddings and process latents as token sequences. Therefore, the relevant geometry is the local within-class, within-token geometry rather than the global geometry obtained by aggregating all classes and tokens.

\paragraph{Dispersion statistics.}
We measure how dispersed samples are in SD-VAE and RAE latent spaces using nearest-neighbor distance (NN-d) and spherical maximum mean discrepancy (S-MMD). Since SD-VAE and RAE latents have different dimensionalities, we normalize all Euclidean distances by $\sqrt{d}$, where $d$ denotes the corresponding feature dimension, which removes the scaling of Euclidean distance with dimensionality and makes the statistics more comparable across latent spaces.

NN-d is the average distance from each sample to its nearest neighbor within the same class:
\begin{align}
\mathrm{NN\mbox{-}d}
=
\frac{1}{\sum_{c=1}^C n_c}\sum_{c=1}^C
\sum_{i=1}^{n_c}
\min_{j\neq i}\frac{1}{\sqrt{d}}\|\mathbf x_{c, i}-\mathbf x_{c, j}\|_2,
\end{align}
where $C$ is the total number of classes and $n_c$ is the number of samples within class $c$.
A smaller NN-d indicates closer neighbors and thus more concentrated samples.

S-MMD compares the empirical sample distribution with a reference spherical distribution.
For a set of \textit{centered} samples $\{\tilde{\mathbf x}_{c,i}\}_{i=1}^{n_c}$, we set the sphere radius to the average centered norm,
\begin{align}
\rho=\frac{1}{n_c}\sum_{i=1}^{n_c}\|\tilde{\mathbf x}_{c,i}\|_2.
\end{align}
Instead of sampling infinitely many points from the sphere, we use a simple deterministic approximation consisting of all poles of the sphere:
\begin{align}
\mathcal S_\rho=\{\pm \rho\mathbf e_1,\ldots,\pm \rho\mathbf e_d\},
\end{align}
where $\{\mathbf e_i\}_{i=1}^{d}$ denotes the standard basis.
We then compute the standard squared MMD between the centered samples and $\mathcal S_\rho$:
\begin{align}
\mathrm{MMD}^2(\mathcal{X},\mathcal{Y})
=
\frac{1}{|\mathcal{X}|^2}\sum_{\mathbf x,\mathbf x'\in \mathcal{X}}k(\mathbf x,\mathbf x')
+
\frac{1}{|\mathcal{Y}|^2}\sum_{\mathbf y,\mathbf y'\in \mathcal{Y}}k(\mathbf y,\mathbf y')
-
\frac{2}{|\mathcal{X}||\mathcal{Y}|}\sum_{\mathbf x\in \mathcal{X}}\sum_{\mathbf y\in \mathcal{Y}}k(\mathbf x,\mathbf y),
\end{align}
where $|\mathcal{X}|$ denotes the number of samples in a finite set $\mathcal{X}$, and $k$ is a selected kernel as
\begin{align}
    k(\mathbf{x},\mathbf{y})=e^{-\frac{1}{\tau\sqrt{d}}\|\mathbf{x}-\mathbf{y}\|_2},
\end{align}
with $\tau=1.0$.
A larger S-MMD indicates that the sample distribution is less sphere-like.

\subsection{Hyperparameter Settings}
\label{app:hyperparameters}
We summarize the main hyperparameters used for \method{} in~\Cref{tab:hyperparameters}.
We omit the exponential moving average (EMA) of model parameters, as it does not yield consistent improvement. 
For IDG with multiple intermediate layers, we concatenate the features from all selected layers along the feature dimension.
The concatenated feature is then projected to the output dimension by a projection head comprising three linear layers with SiLU activations. For results with IDG, the guidance weight is default to $\omega=1.4$.

For the generated CLS tokens, we directly take the CLS tokens of real samples produced by the frozen DINO encoder to compute the drifting field.
Following~\cite{caron2021dino}, we also append four extra registers tokens to the Transformer token sequence.
These extra tokens have no explicit target or auxiliary loss and simply provide redundancy during training and inference.
\input{tables/hyperparameters}

\subsection{Additional Ablations}
\label{app:idg_ablation}
Here we further ablate the additional tokens and the implementation details of IDG.
All results are evaluated on ImageNet $256\times256$ with distilled $\text{DiT}^{\text{DH}}\text{-XL}$.

\input{tables/idg_layers_ablation}

\input{tables/idg_tokens_ablation}

\noindent\textbf{CLS and extra tokens.}
Table~\ref{tab:idg_tokens} reports the comparison of FID scores with and without the additional tokens. The inclusion of the additional tokens leads to slight improvement across guidance scales.

\noindent\textbf{Intermediate layers for IDG.}
\Cref{tab:idg_layers} compares all combinations of layers $\{8,10,12,14\}$ for IDG.
Using two intermediate layers yields slightly better results than using a single layer, while adding more layers does not consistently improve performance.
Among the tested configurations, layers $\{8,14\}$ achieve the best FID of $1.48$ and are used in our final model.

\input{tables/ablation_idg_schedule}

\noindent\textbf{Training schedule for IDG.}
We also evaluate a two-stage training schedule as applied in~\cite{fu2026frozenpixelspacediffusionmodel}.
Specifically, we first distill the backbone and then train the IDG head in a dedicated stage with the forzen backbone.
As shown in~\Cref{tab:idg_training_schedule}, this two-stage schedule achieves an FID of $1.50$ compared to $1.48$ of our joint training strategy.
Given these commensurate performance, we adopt the joint training paradigm as it eliminates the need for an additional training stage.

\input{tables/ablation_idg_loss}

\noindent\textbf{Training loss for guidance head.} 
Though training the guidance head via drifting brings theoretical benifits via a clear quality order, constructing the drifting field is generally more consuming.
Theorefore, we further evaluate MSE and cosine similarity losses to train the guidance head. As pair data is not available in one-step regime, we take the final output of the backbone as the auxilary target for these losses. As shown in~\Cref{tab:idg_training_loss}, MSE loss yields similary results as Drifting loss, while cosine similarity fails to produce reasonable guidance, though its the default choice in REPA~\citep{yu2025repa}. Using cosine similarity alone lost amplitude information, and thus we introduce a modified loss as
\begin{align}
    \mathcal{L}_{\text{Cosine sim.+norm}}(\mathbf{x}, \mathbf{y})=-\frac{\min\left\{\|\mathbf{x}\|,\| \mathbf{y}\|\right\}}{\max\left\{\|\mathbf{x}\|,\|\mathbf{y}\|\right\}}\cos \angle(\mathbf{x}, \mathbf{y}).\nonumber
\end{align}
As the cosine similarity is positive upon convergence, such loss is minimized when the cosine similarity converges to $1$ while the norms match. Applying the modified loss also leads to comparable results with drifting loss.

\section{Attempts to Train Drifting Models from Scratch with RAEs}\label{sec:attempts_train_from_scratch}
\input{tables/from_scratch}
We also explore whether Drifting Models can be trained from scratch without auxiliary MAEs. 
As shown in~\Cref{tab:from_scratch}, directly training Drifting Models in either SD-VAE or RAE latent spaces without an additional MAE fails to produce effective results. 
\Cref{thm:counterexample} suggests that this failure in RAE latent space may be caused by the poor initialization encountered in from-scratch Drifting training.

To alleviate this issue, we further try a simple decode-encode strategy. 
Specifically, we first decode the generated latents using the RAE decoder, and then re-encode the decoded samples with the RAE encoder to compute the Drifting direction and apply gradient backpropagation. 
We denote this variant as ``+ decode-encode''. 
As shown in~\Cref{tab:from_scratch}, this strategy enables training from scratch in RAE spaces and achieves an FID of $7.04$. 
We hypothesize that the decode-encode process can effectively project off-manifold generated samples back toward the data manifold, thus alleviating the influence of poor initialization.
Despite this improvement, training Drifting Models from scratch in RAE spaces still lags behind state-of-the-art methods. 
We leave further improvement of MAE-free from-scratch Drifting training as an important direction for future work.

\section{More Visualizations}\label{sec:more_visualizations}

Additional class-wise qualitative results are shown in~
\Cref{fig:appendix_more_visualizations_1,fig:appendix_more_visualizations_2,fig:appendix_more_visualizations_4}.

\newcommand{\morevis}[2]{%
\begin{minipage}[t]{0.485\linewidth}
\centering
\includegraphics[width=\linewidth]{figs/more_visualizations/classes_guidance1.4_8samples/class_#1.pdf}\\[-0.25em]
{\scriptsize\textbf{Class #1:} #2}
\end{minipage}%
}

\begin{figure}[p]
\centering
\morevis{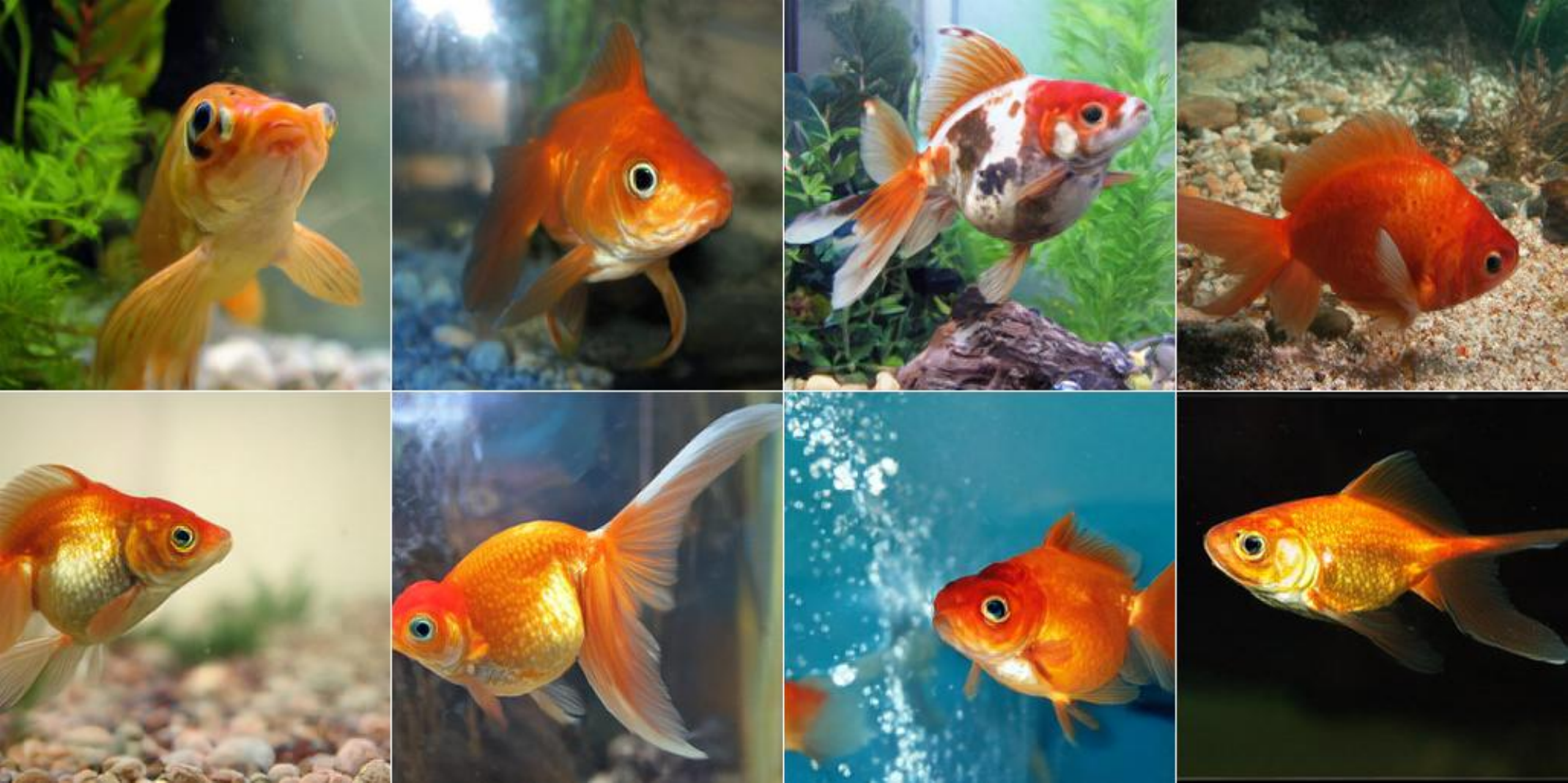}{goldfish}\hfill
\morevis{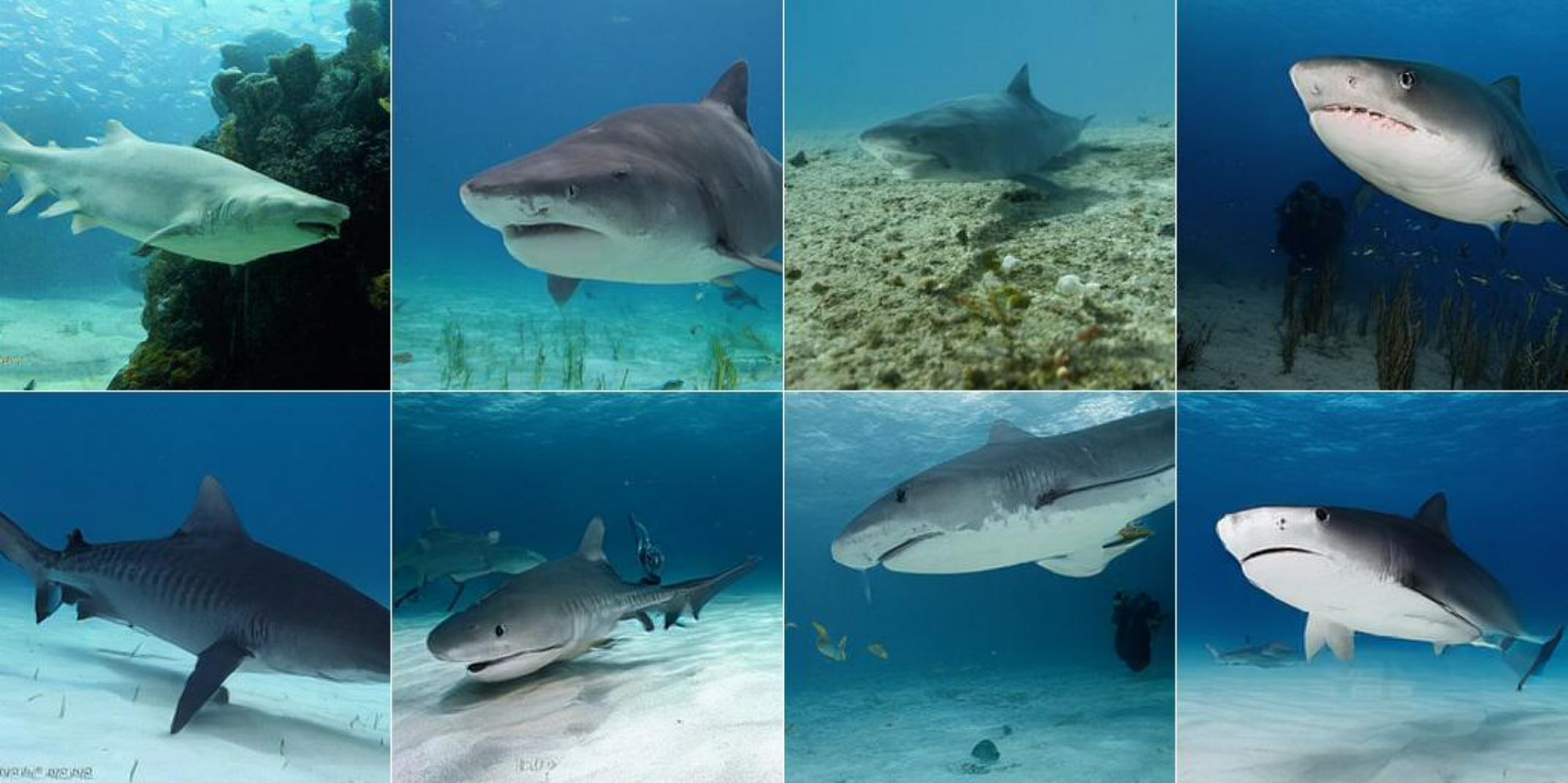}{tiger shark}\\[0.35em]
\morevis{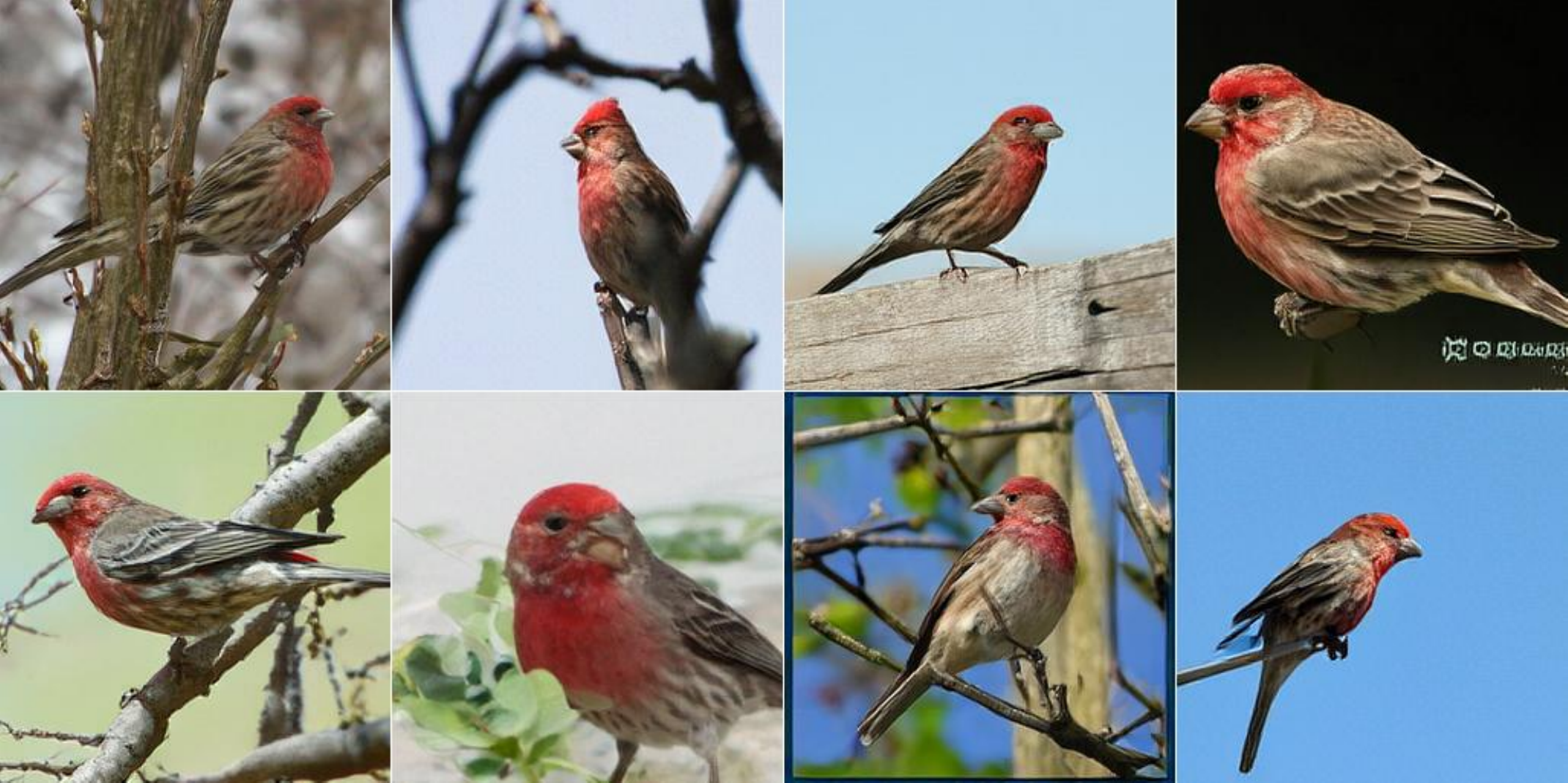}{house finch}\hfill
\morevis{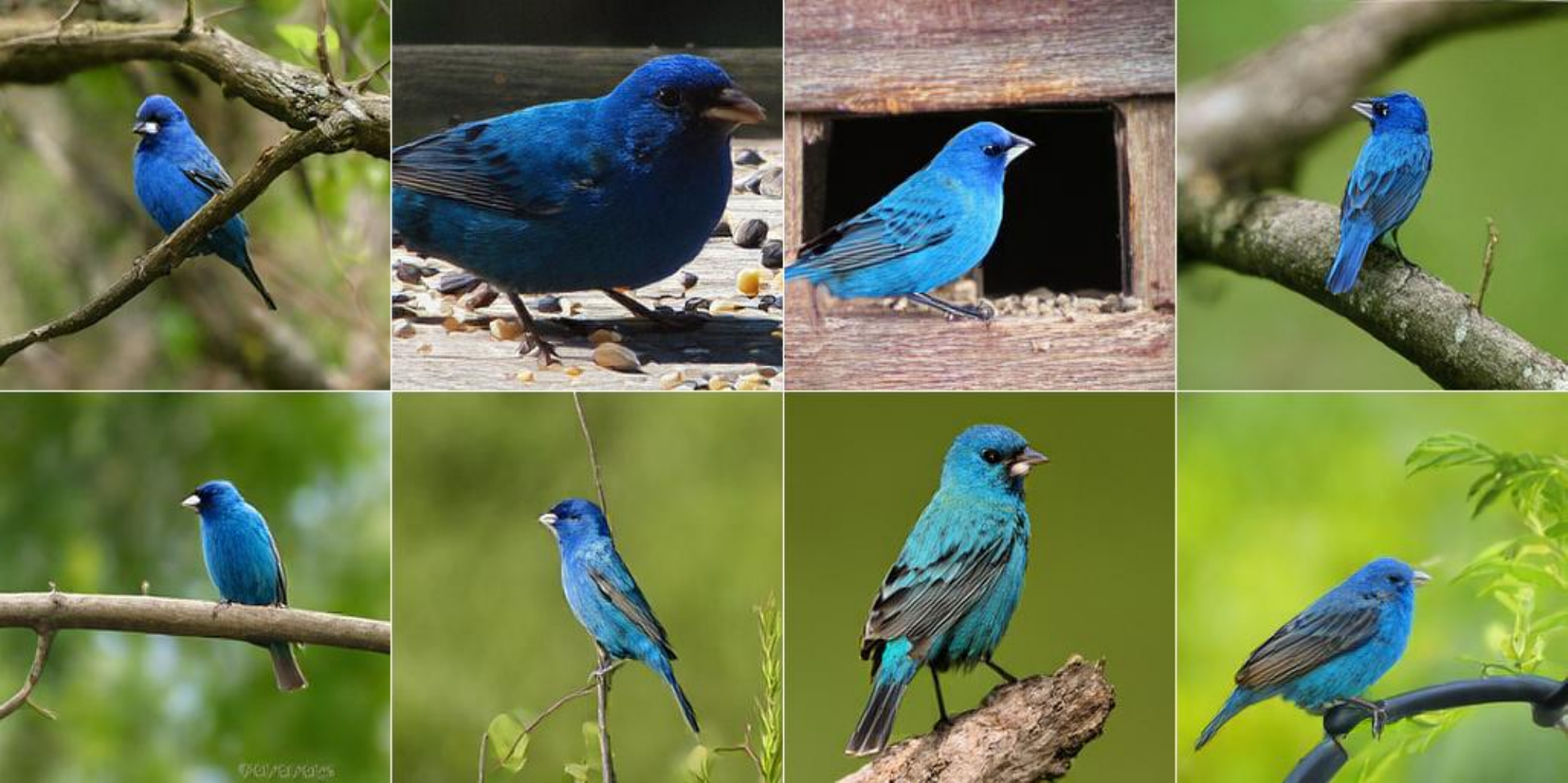}{indigo bunting}\\[0.35em]
\morevis{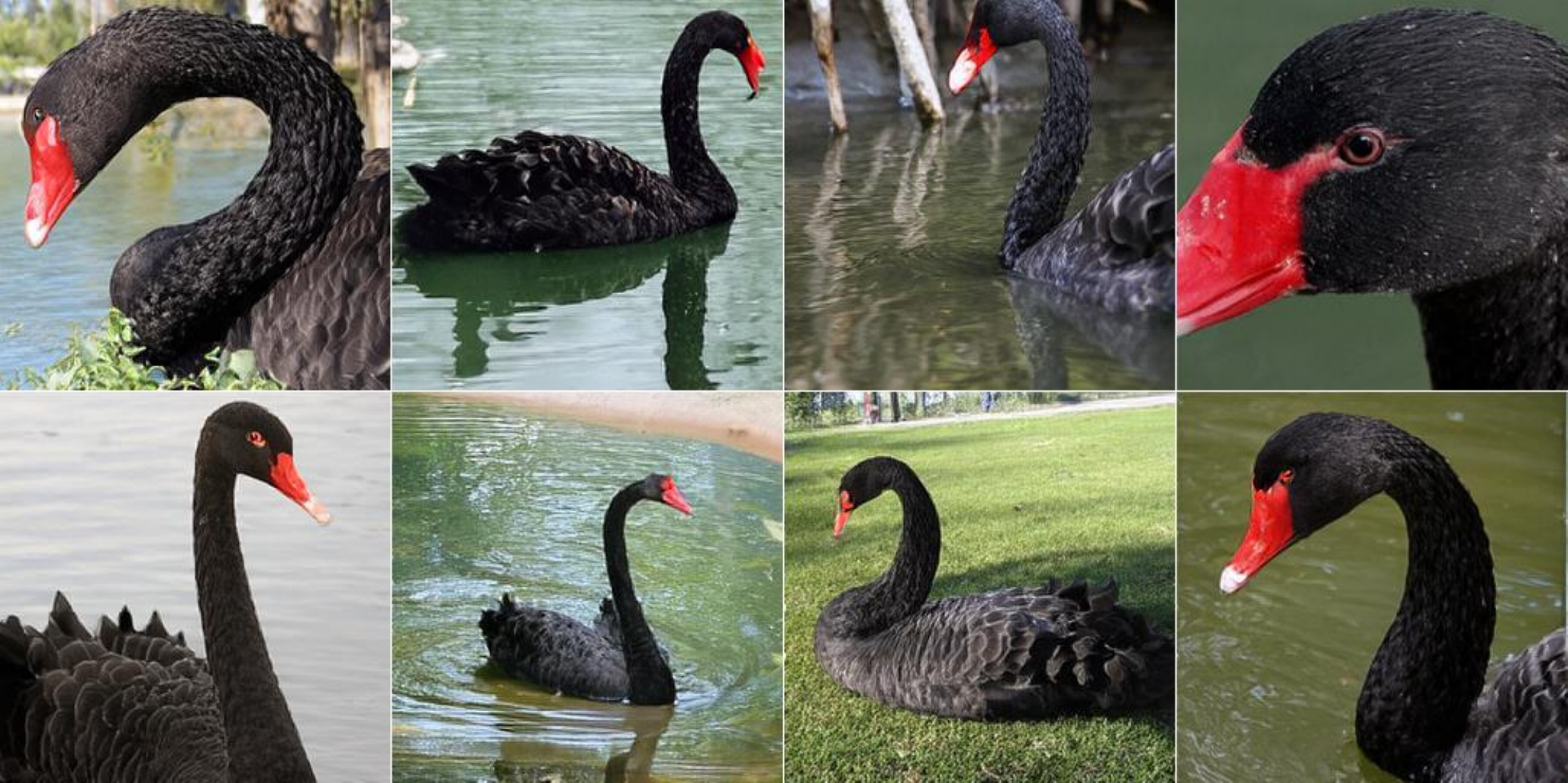}{black swan}\hfill
\morevis{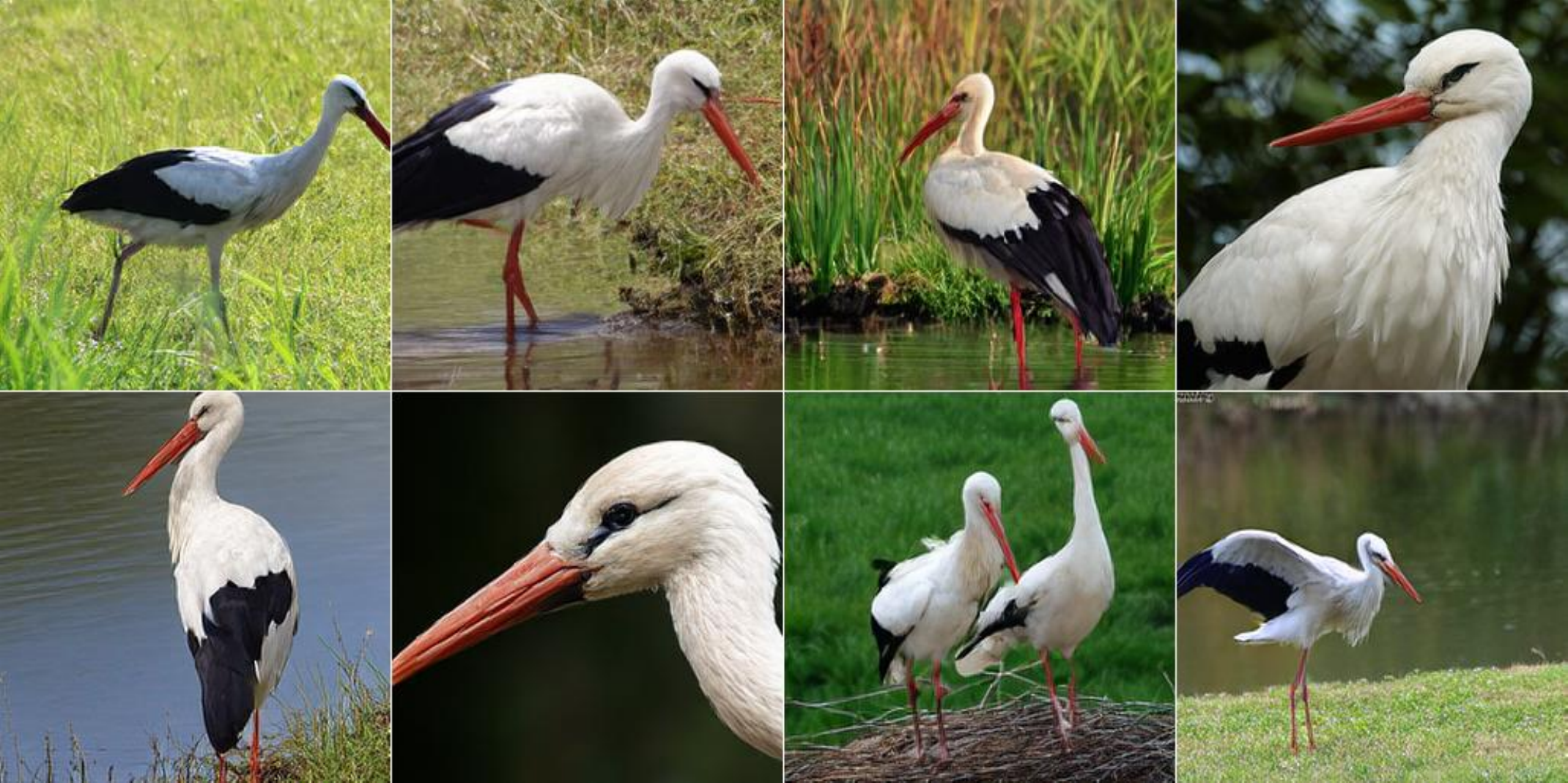}{white stork}\\[0.35em]
\morevis{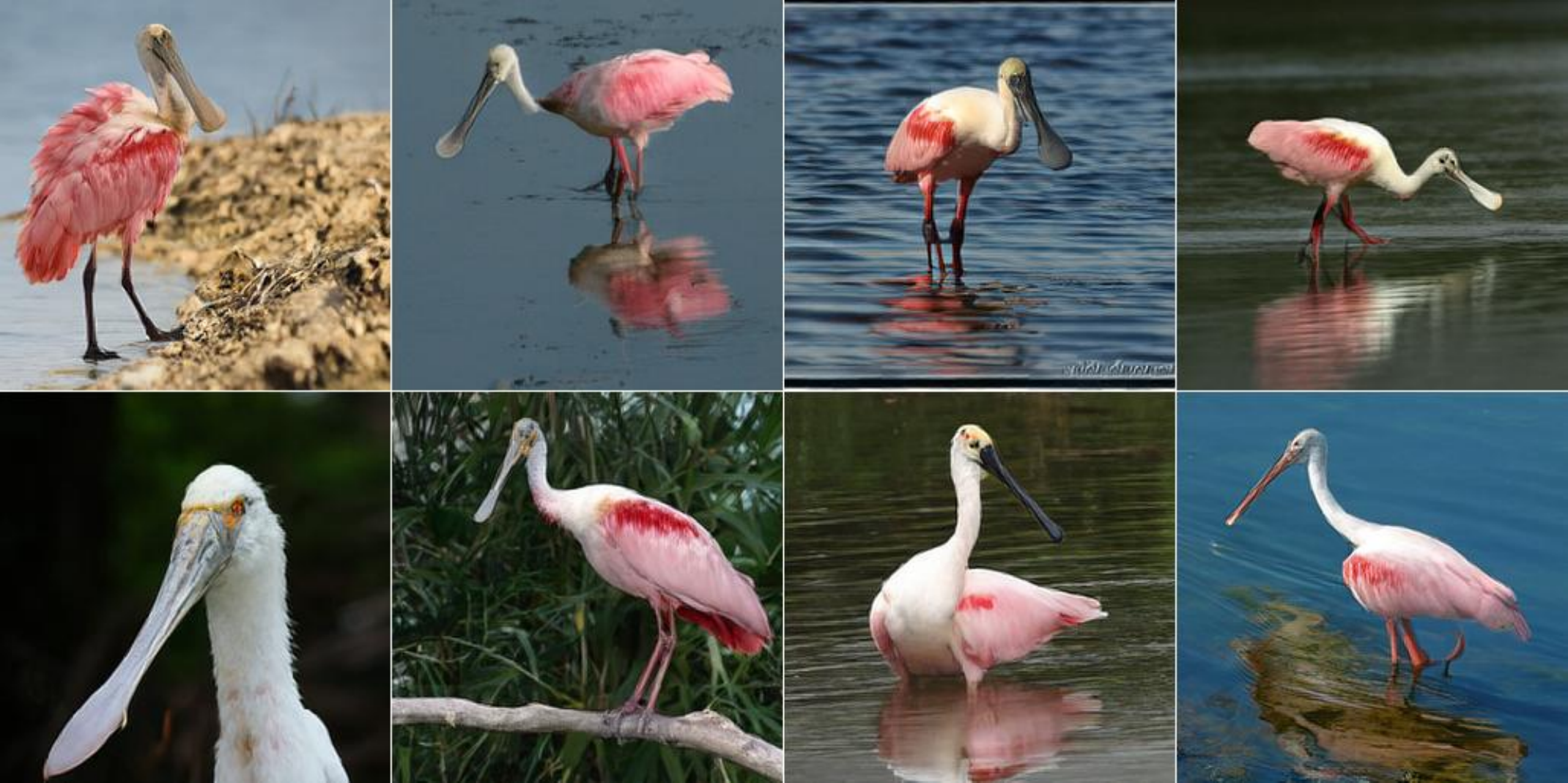}{spoonbill}\hfill
\morevis{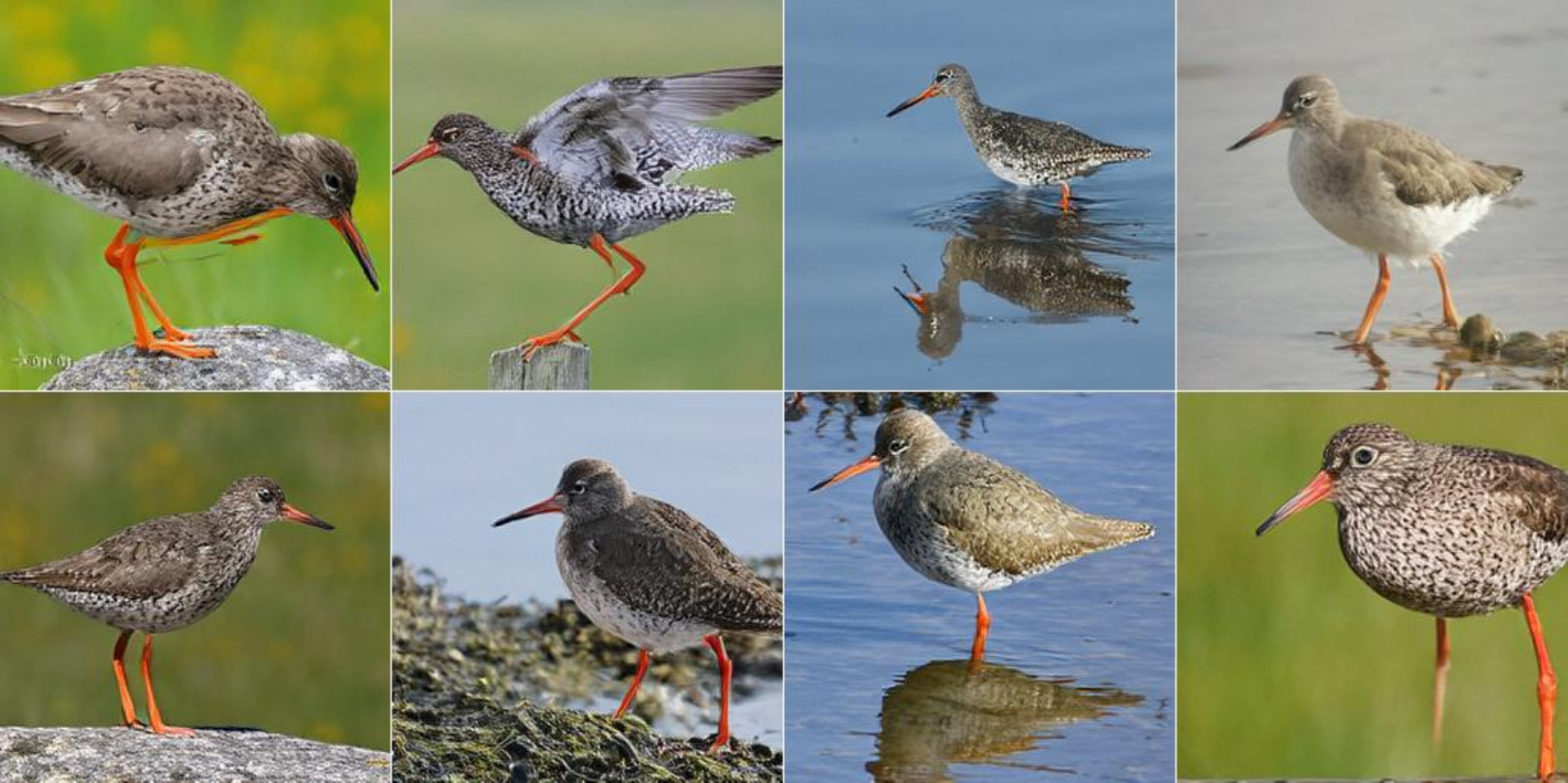}{redshank}\\[0.35em]
\morevis{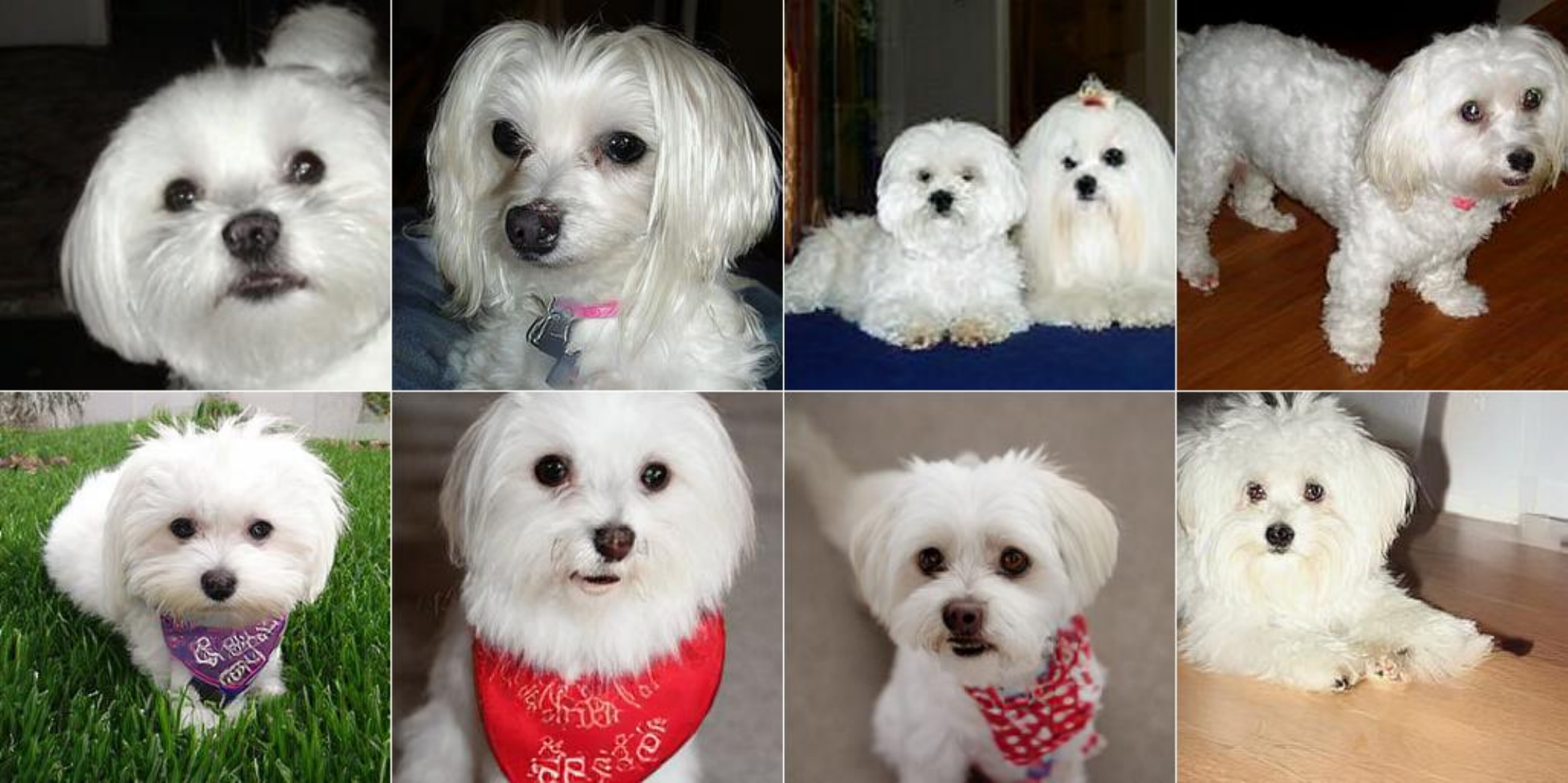}{Maltese dog}\hfill
\morevis{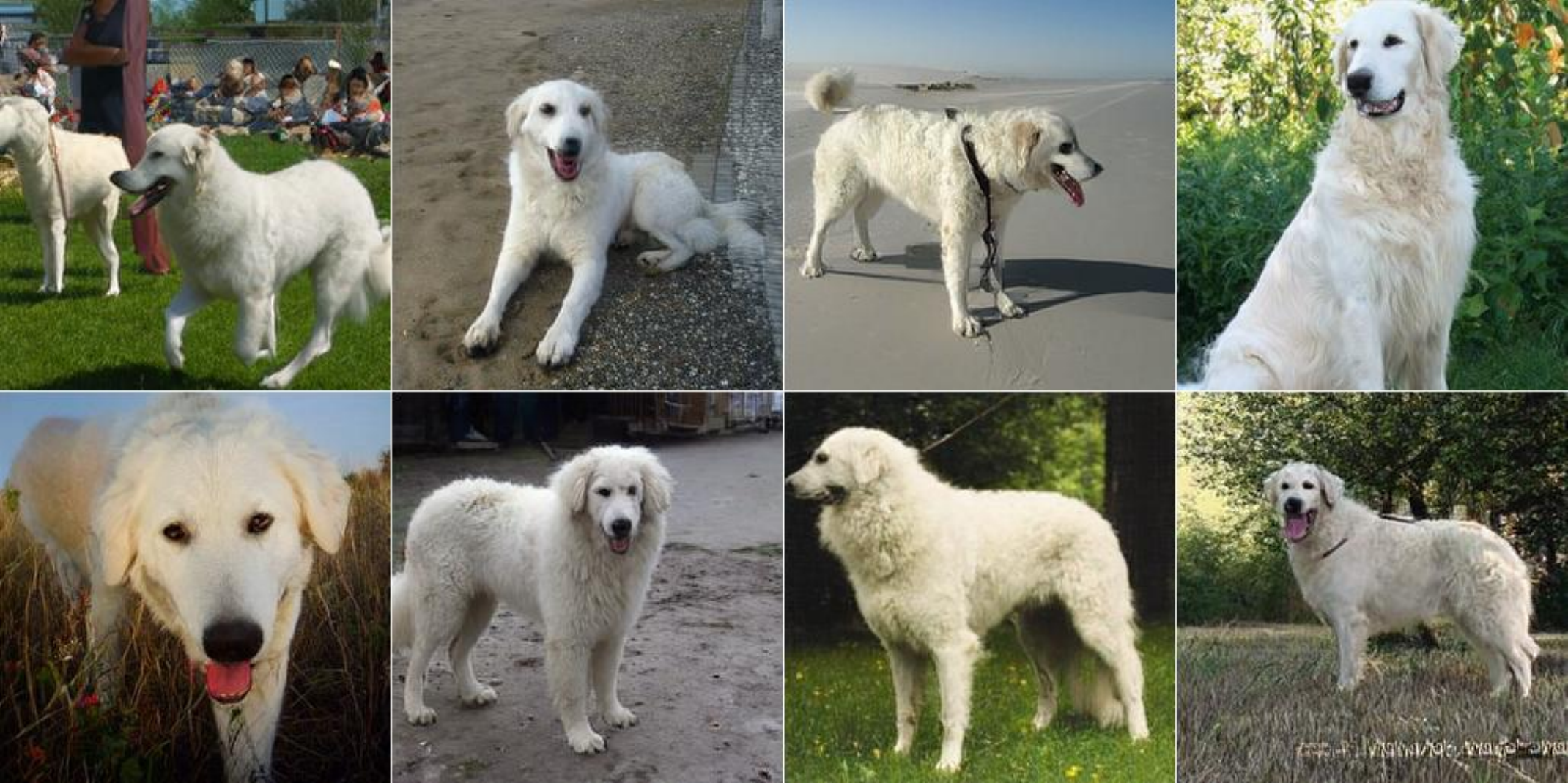}{kuvasz}
\caption{Additional visualizations of generated samples from distilled $\text{DiT}^\text{DH}\text{-XL}$ (FID=1.48).}
\label{fig:appendix_more_visualizations_1}
\end{figure}
\clearpage

\begin{figure}[p]
\centering
\morevis{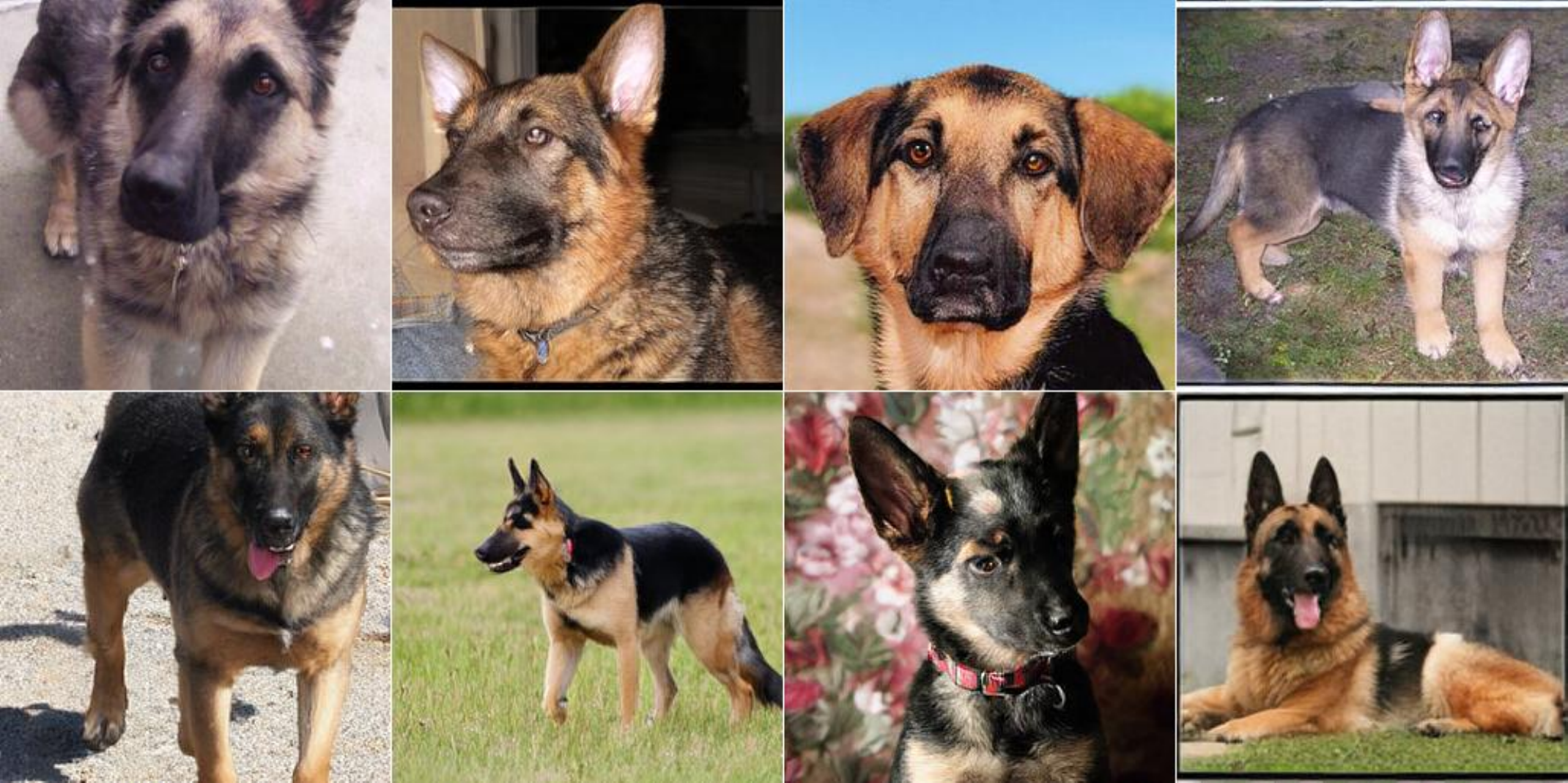}{German shepherd}\hfill
\morevis{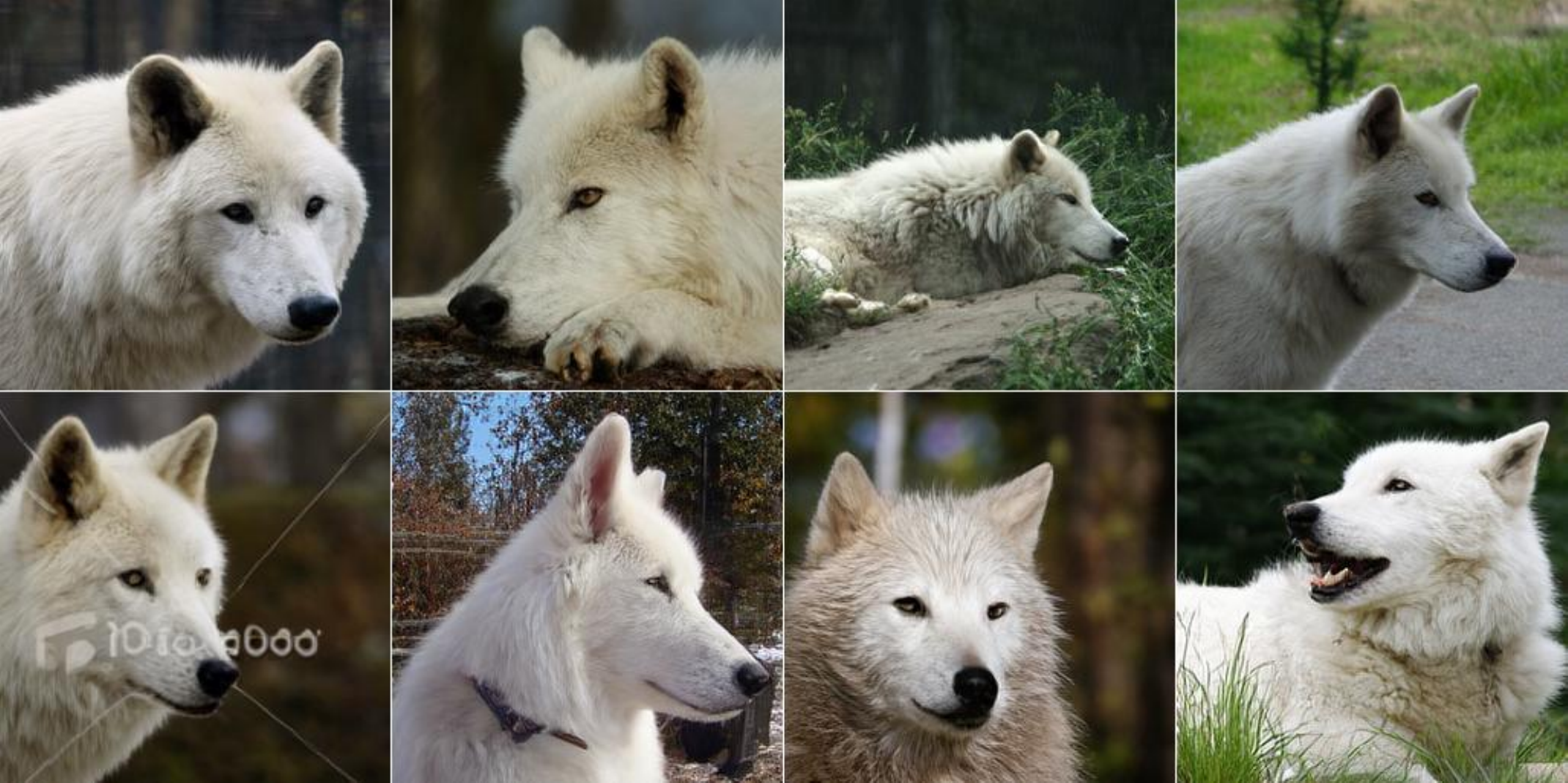}{white wolf}\\[0.35em]
\morevis{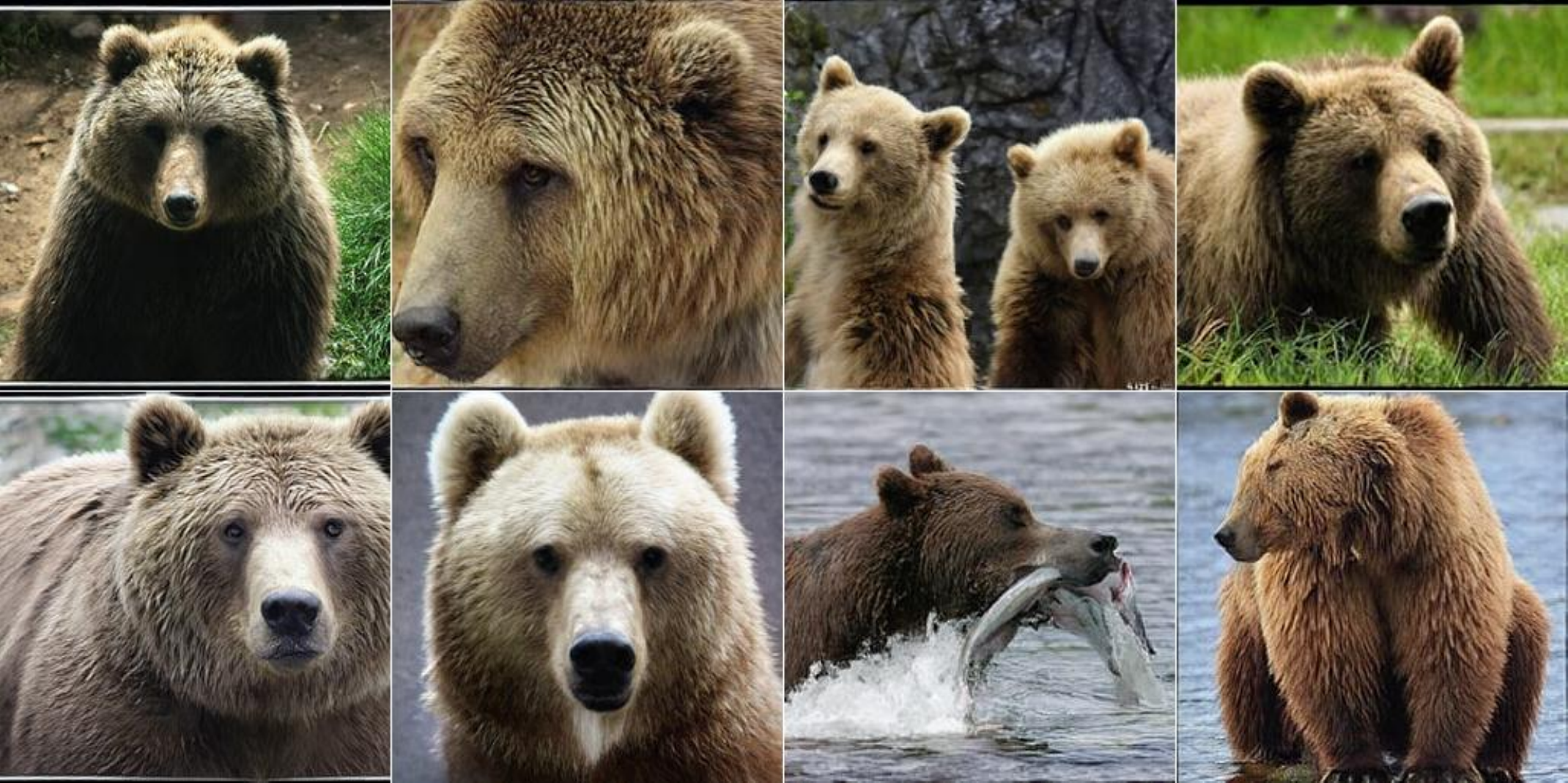}{brown bear}\hfill
\morevis{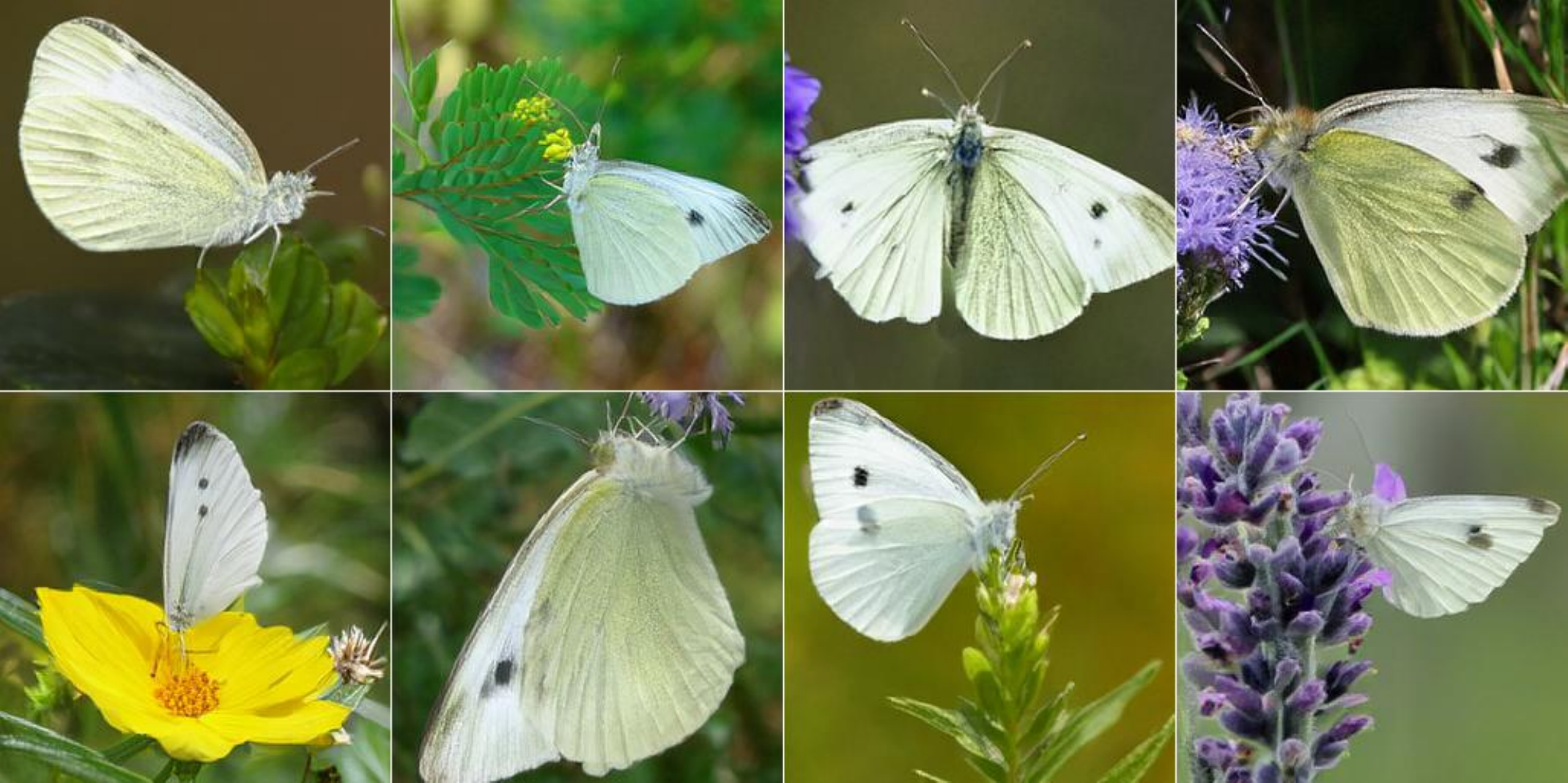}{cabbage butterfly}\\[0.35em]
\morevis{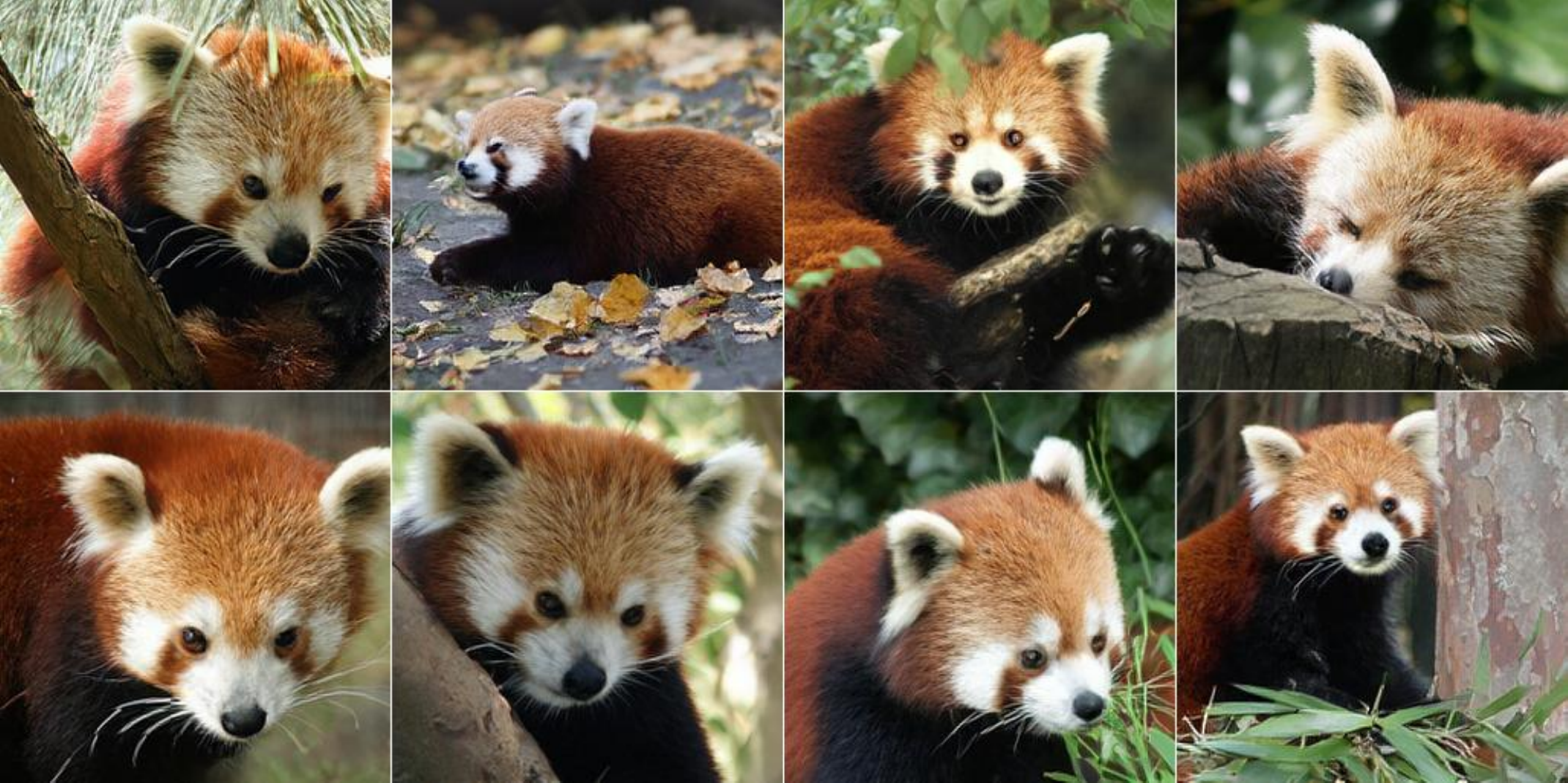}{red panda}\hfill
\morevis{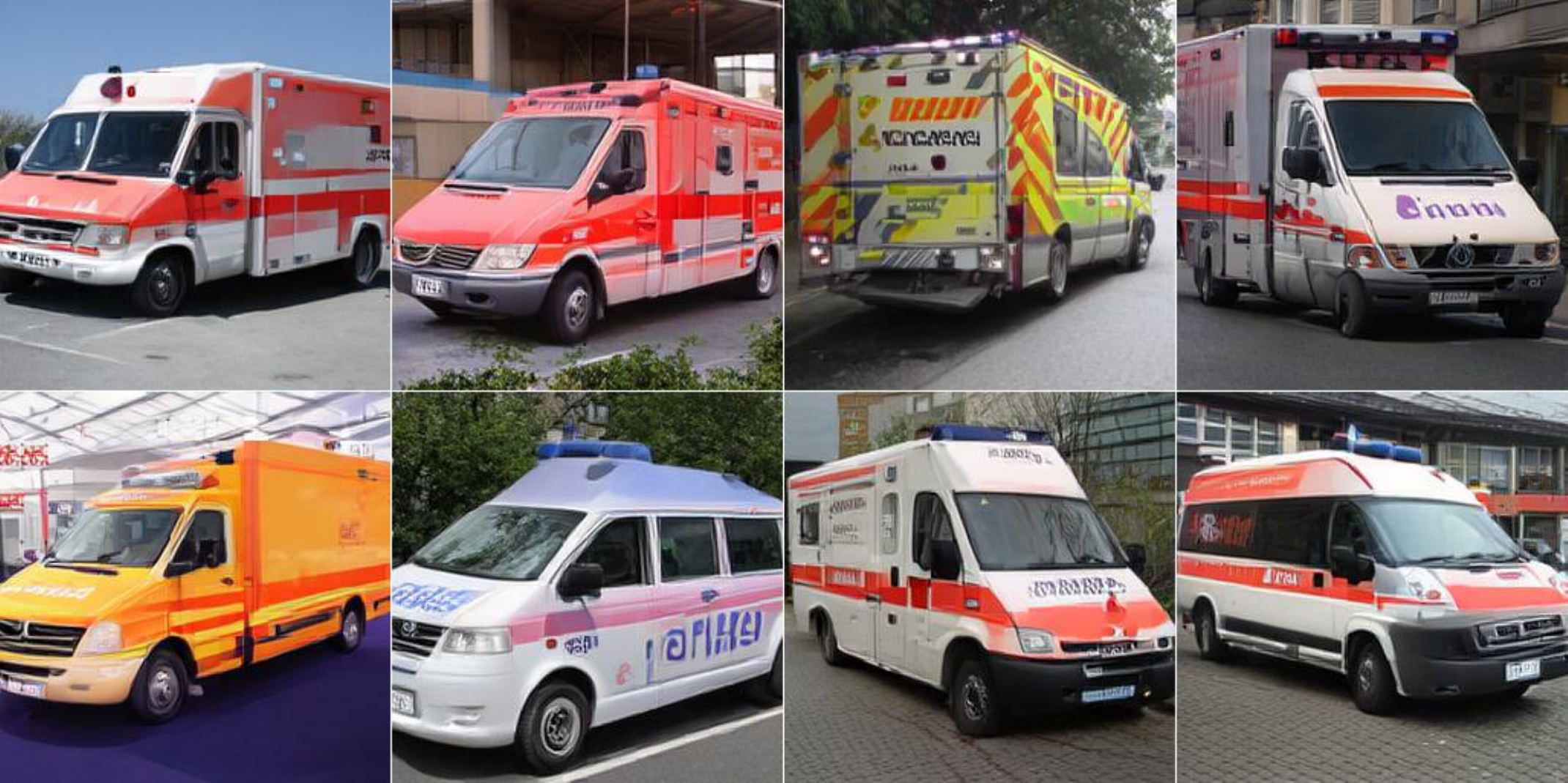}{ambulance}\\[0.35em]
\morevis{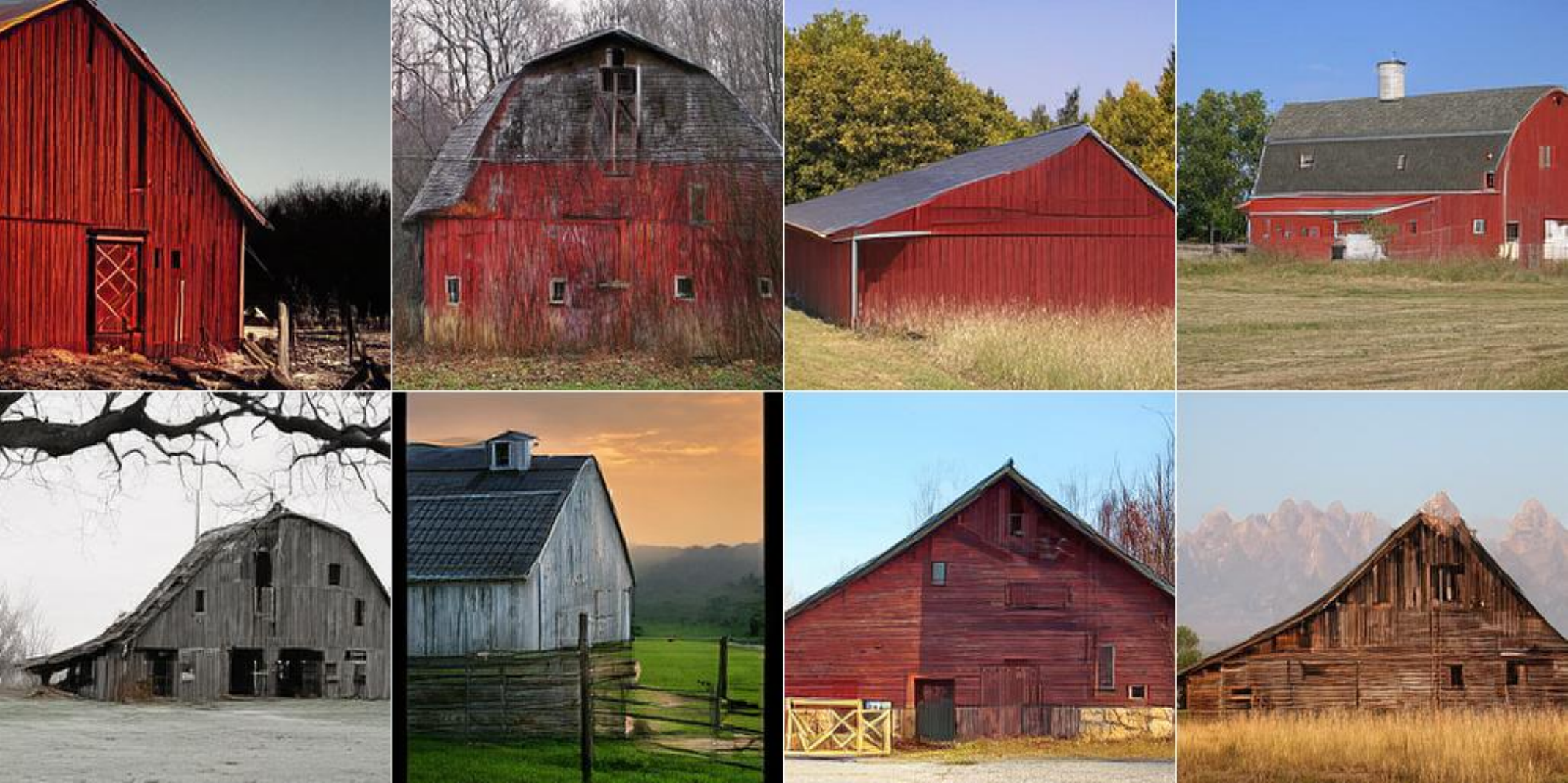}{barn}\hfill
\morevis{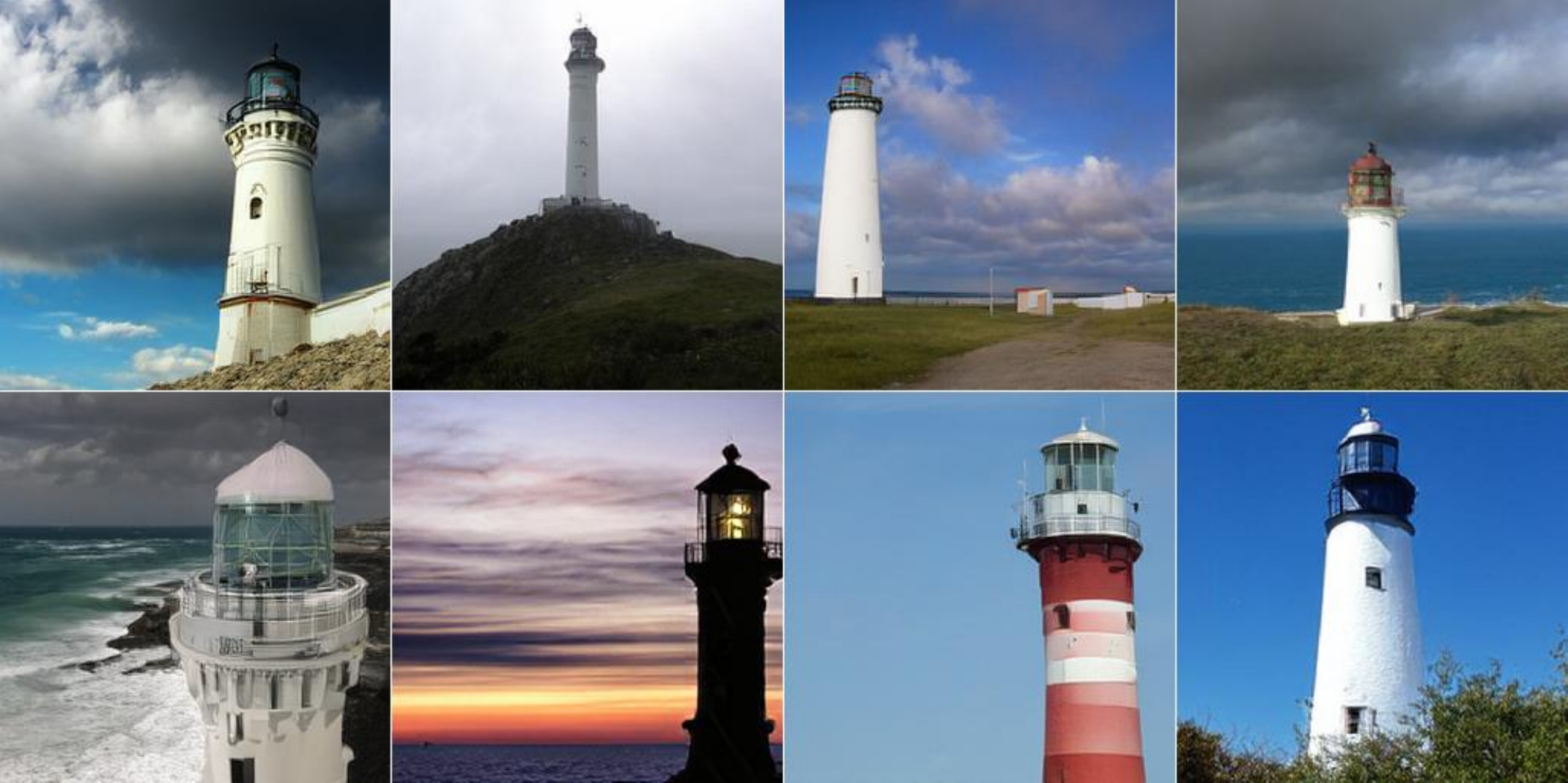}{beacon}\\[0.35em]
\morevis{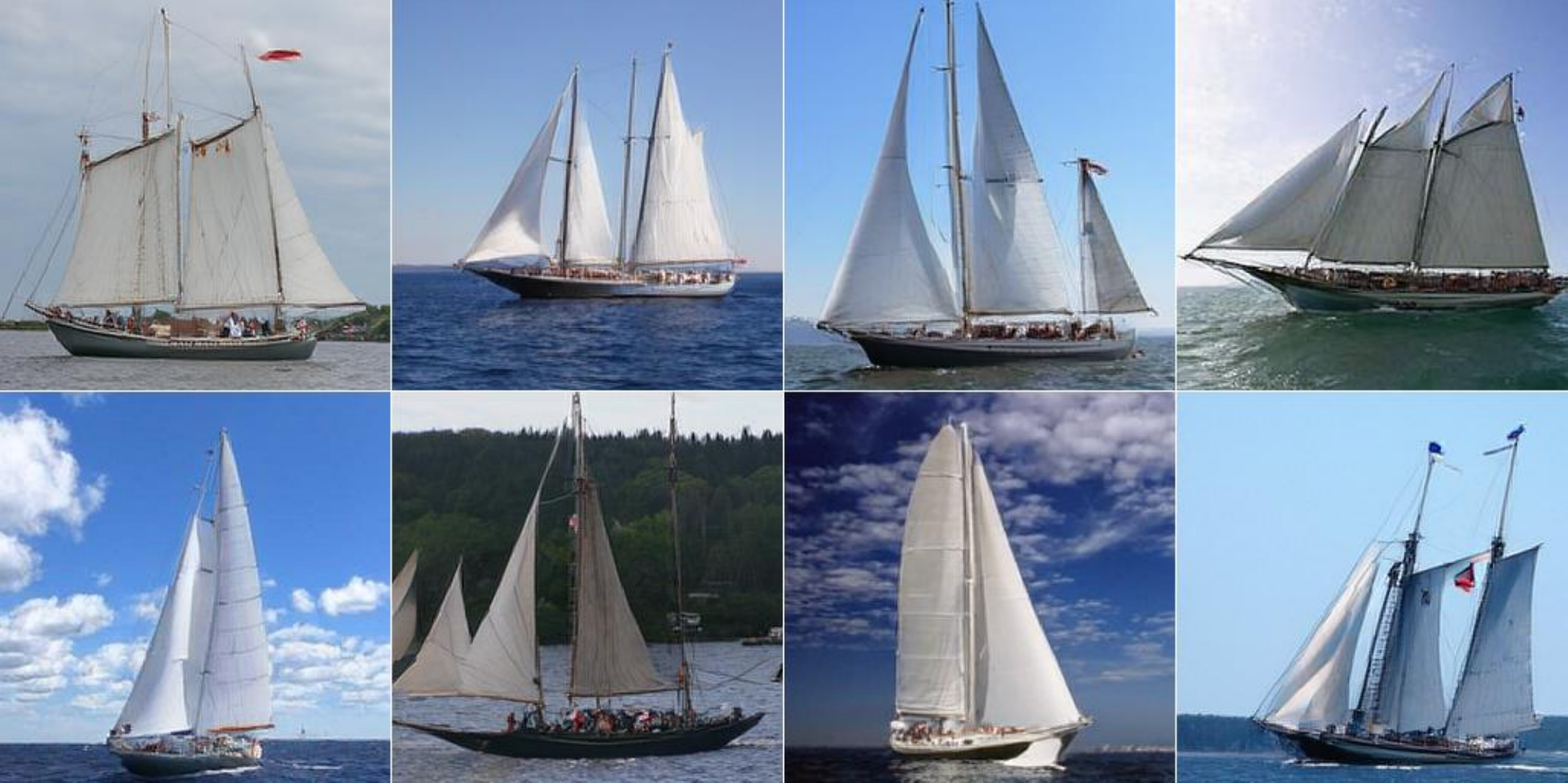}{schooner}\hfill
\morevis{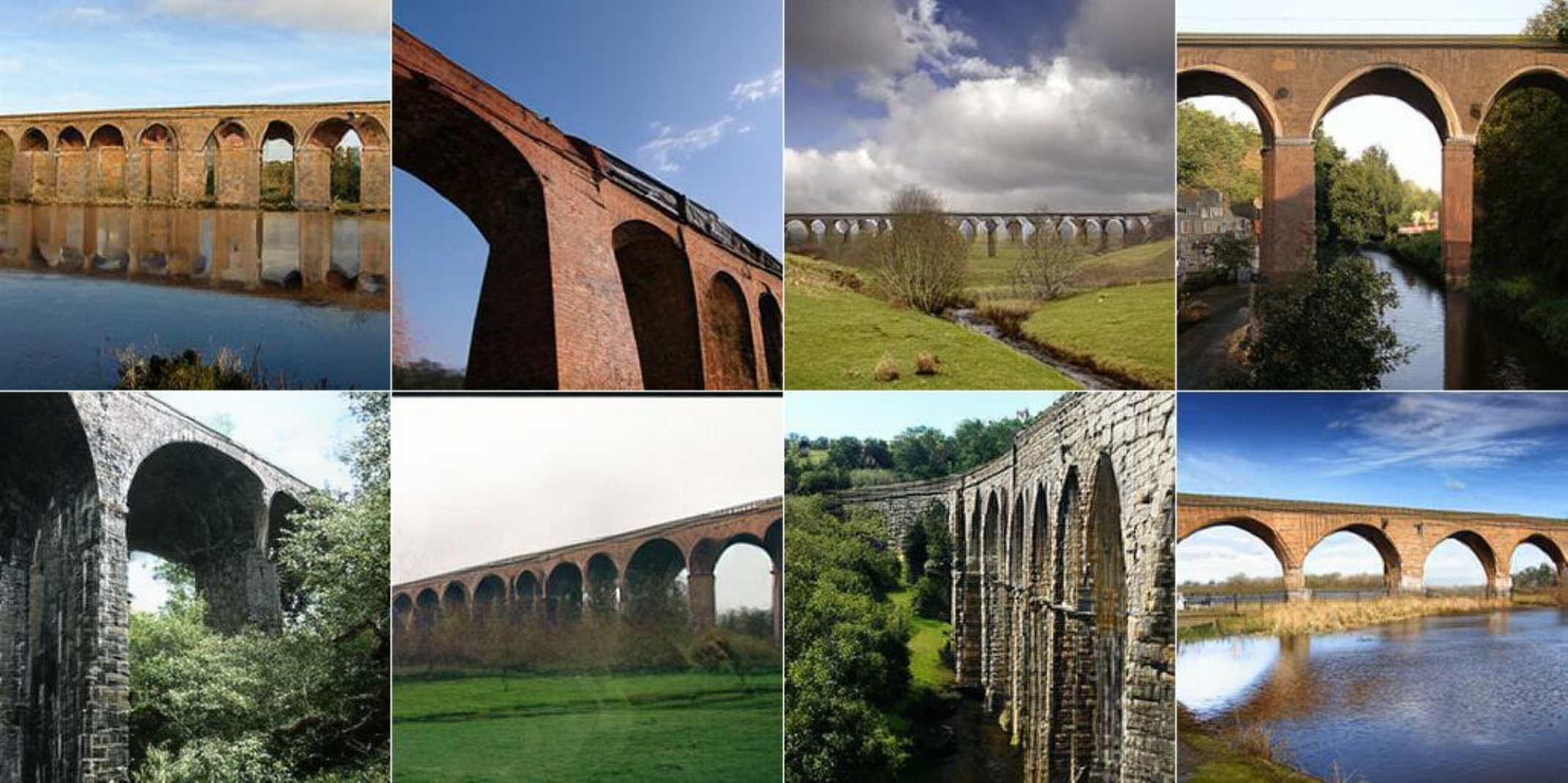}{viaduct}
\caption{Additional visualizations of generated samples from distilled $\text{DiT}^\text{DH}\text{-XL}$ (FID=1.48).}
\label{fig:appendix_more_visualizations_2}
\end{figure}
\clearpage


\begin{figure}[p]
\centering
\morevis{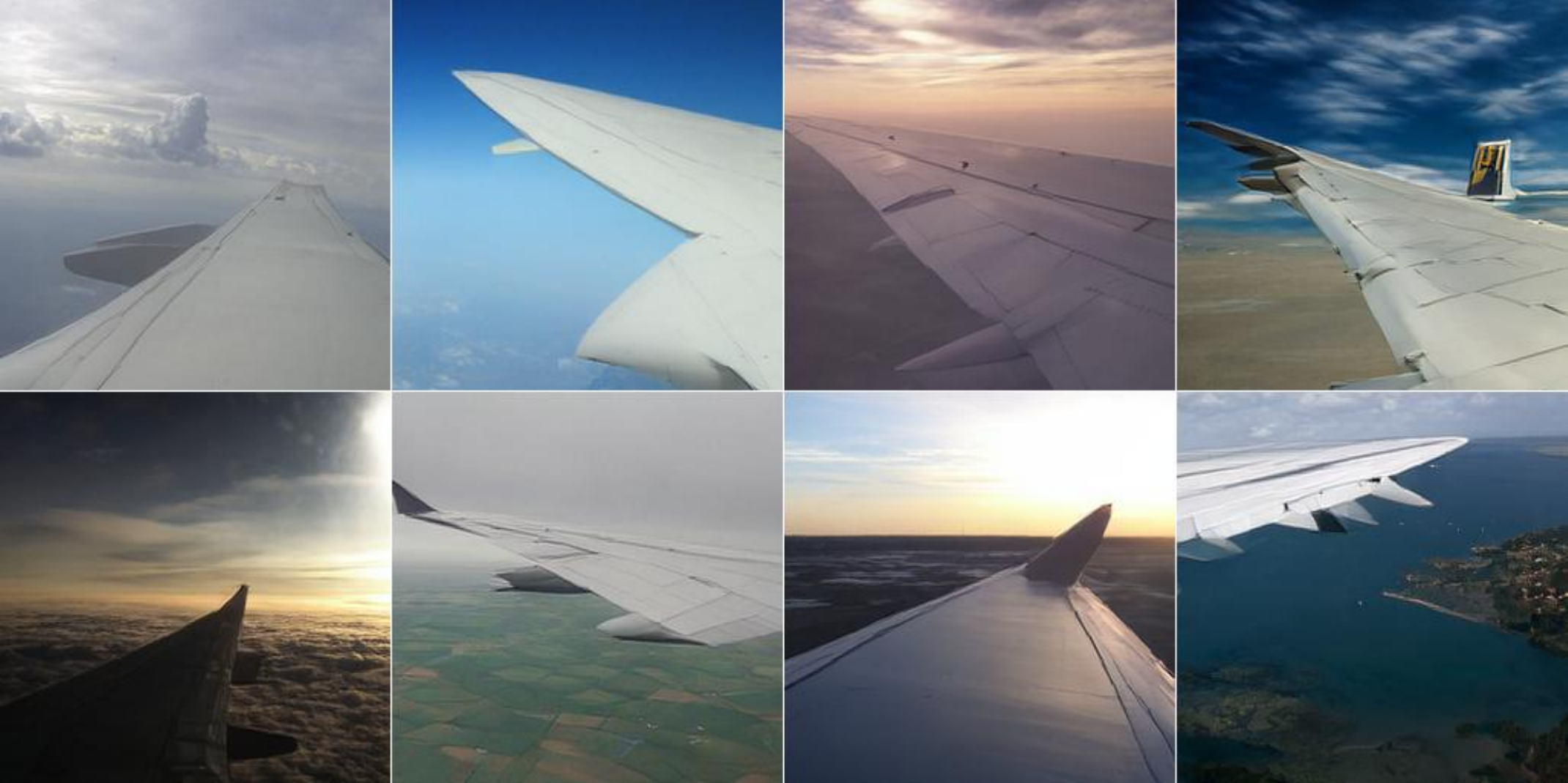}{wing}\hfill
\morevis{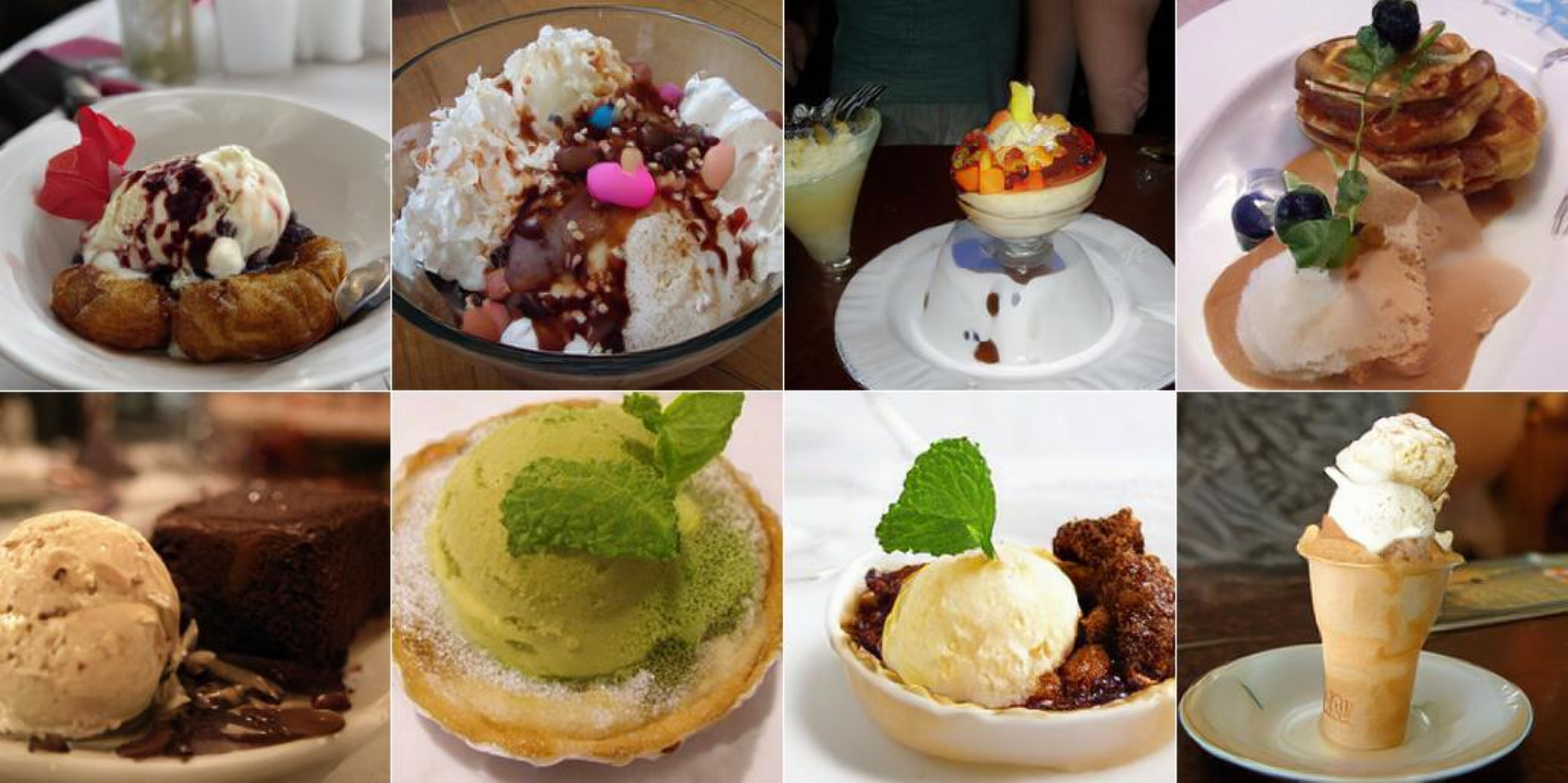}{ice cream}\\[0.35em]
\morevis{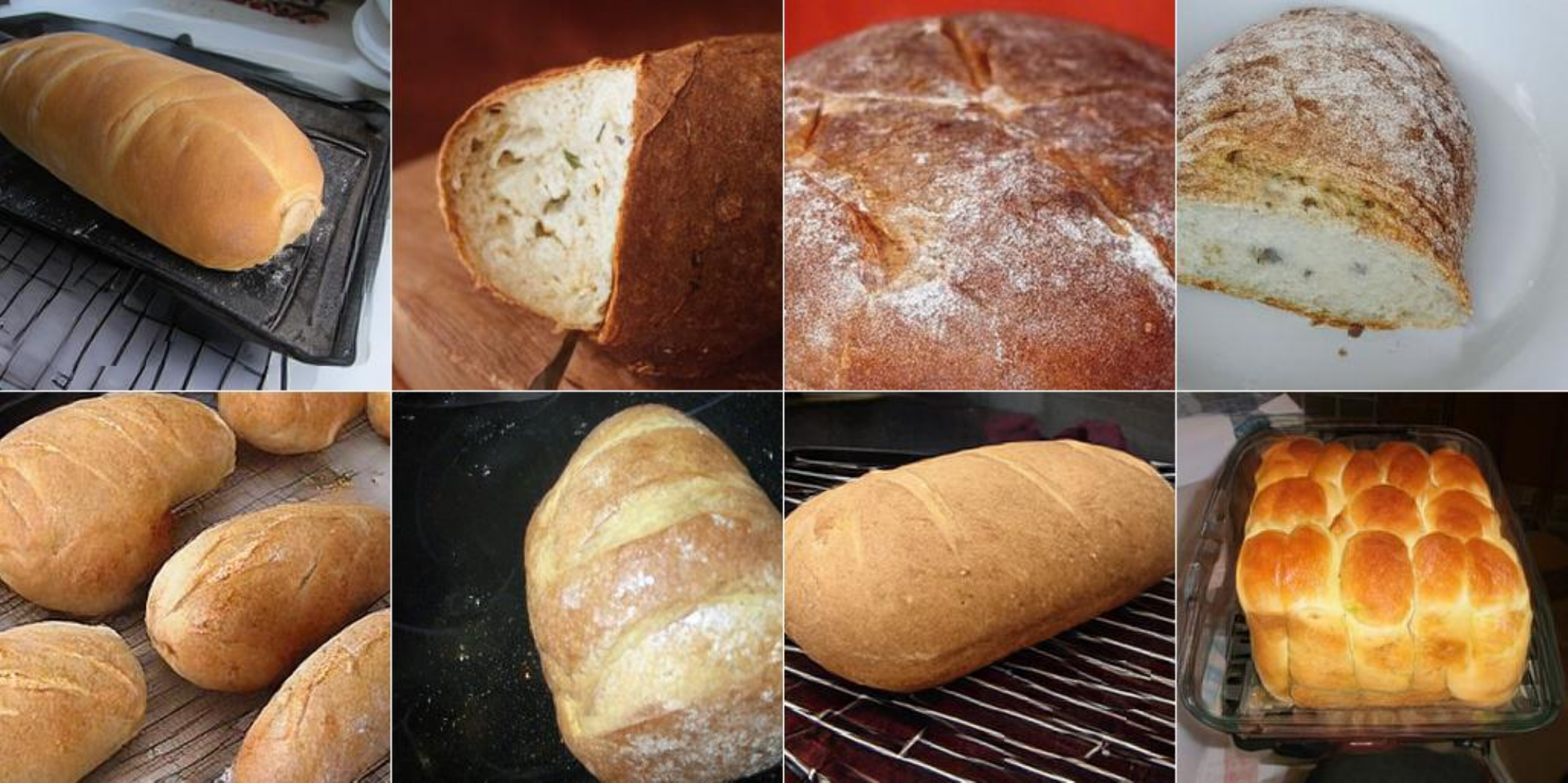}{French loaf}\hfill
\morevis{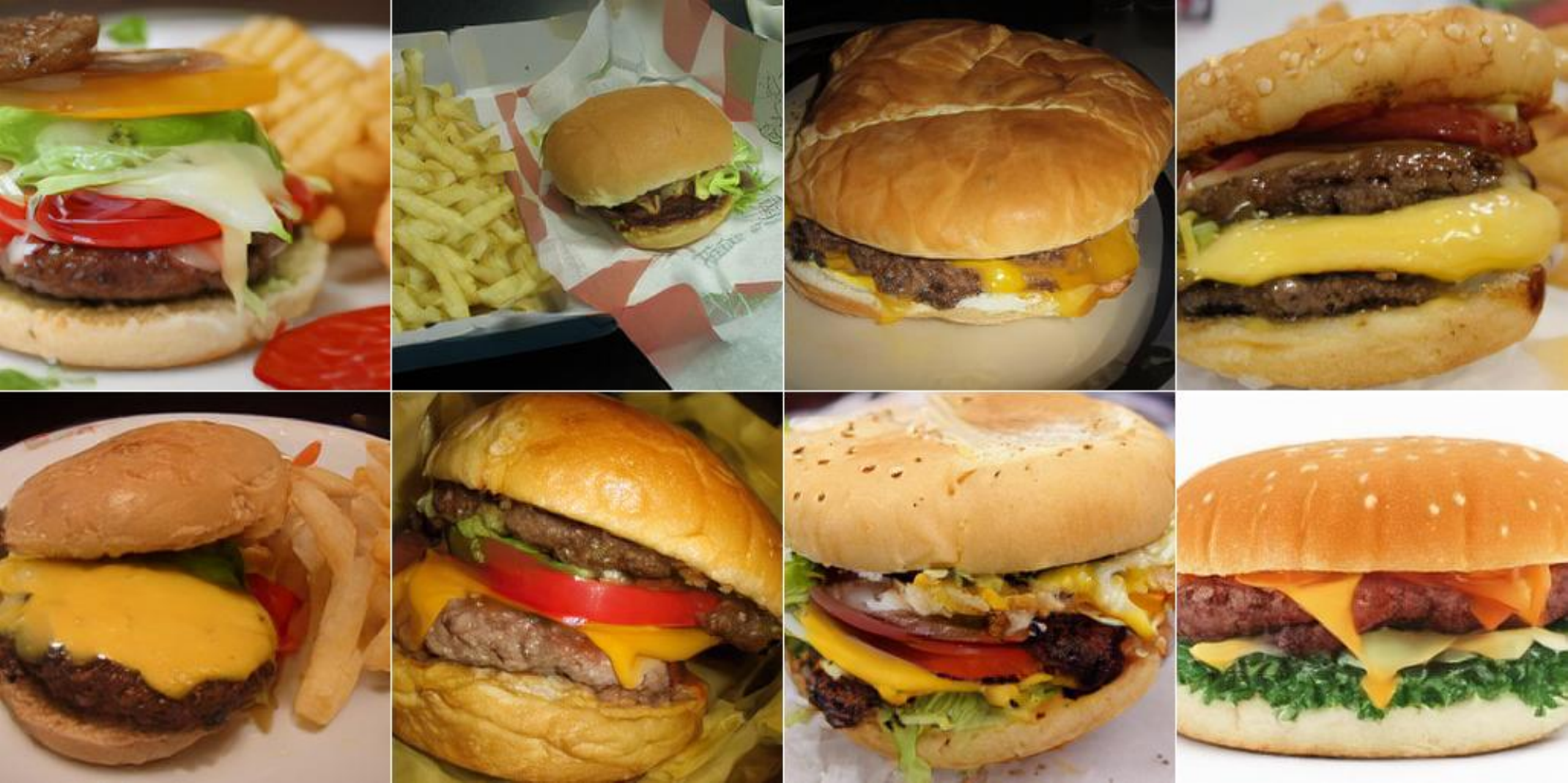}{cheeseburger}\\[0.35em]
\morevis{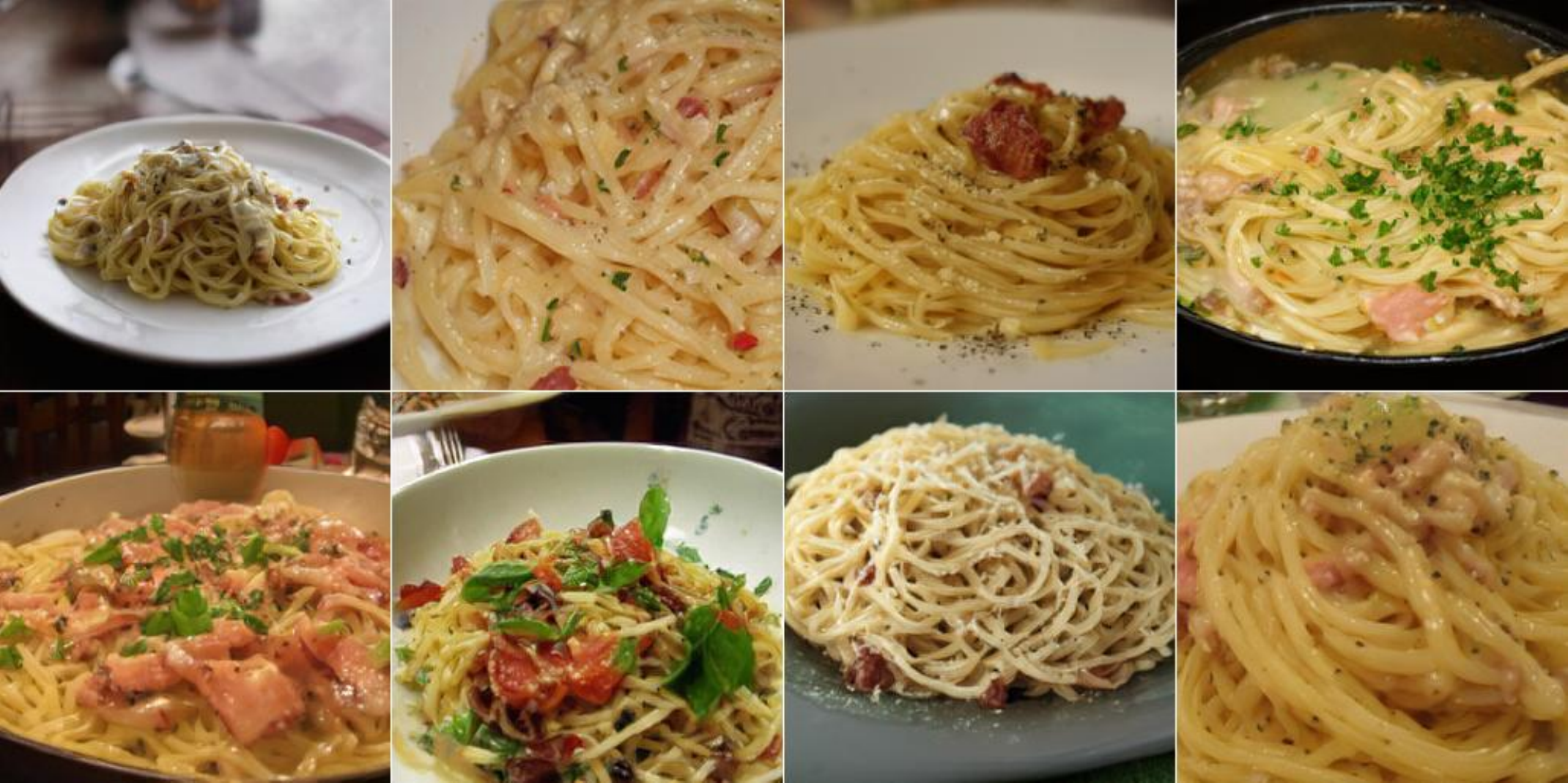}{carbonara}\hfill
\morevis{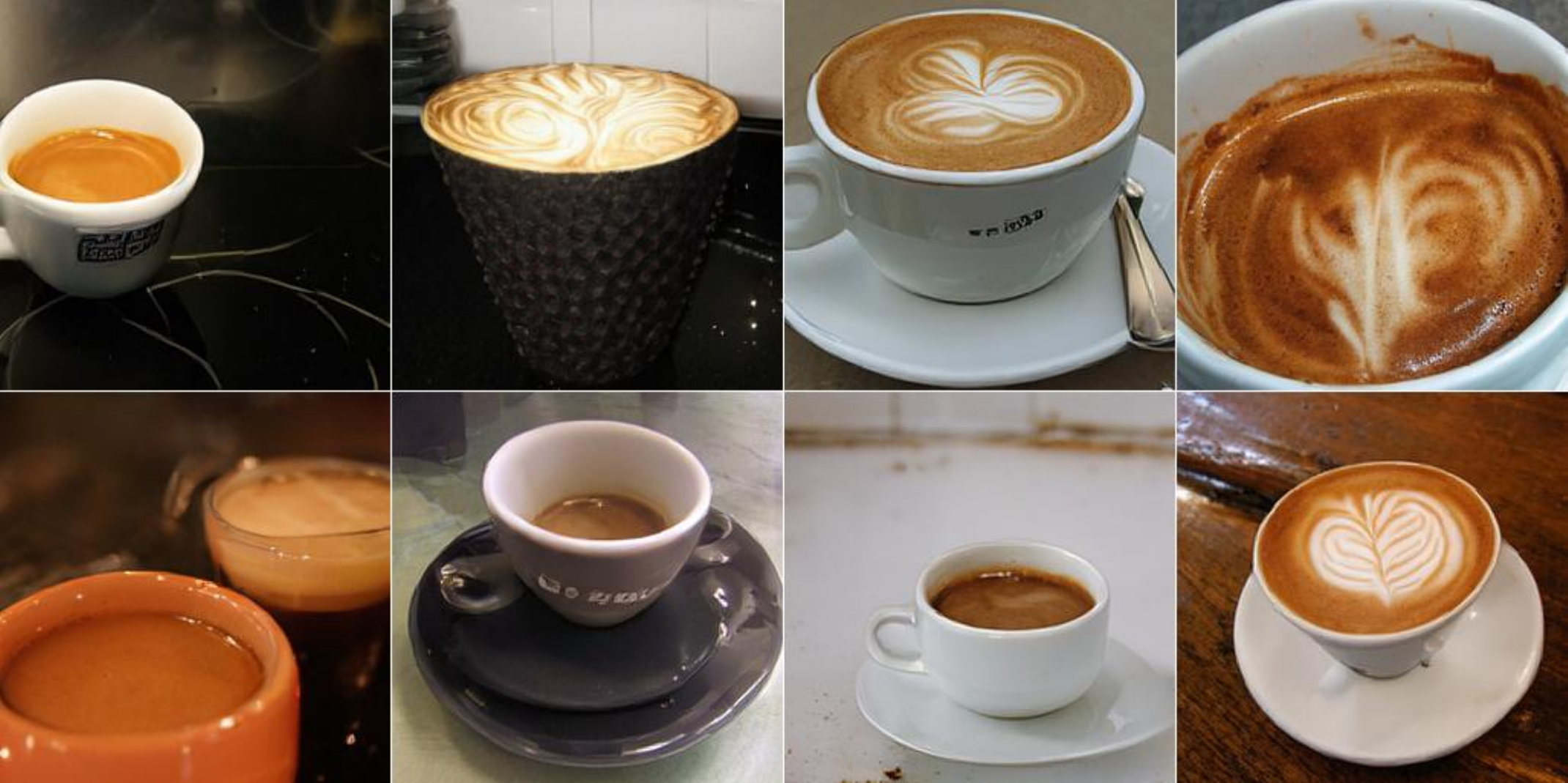}{espresso}\\[0.35em]
\morevis{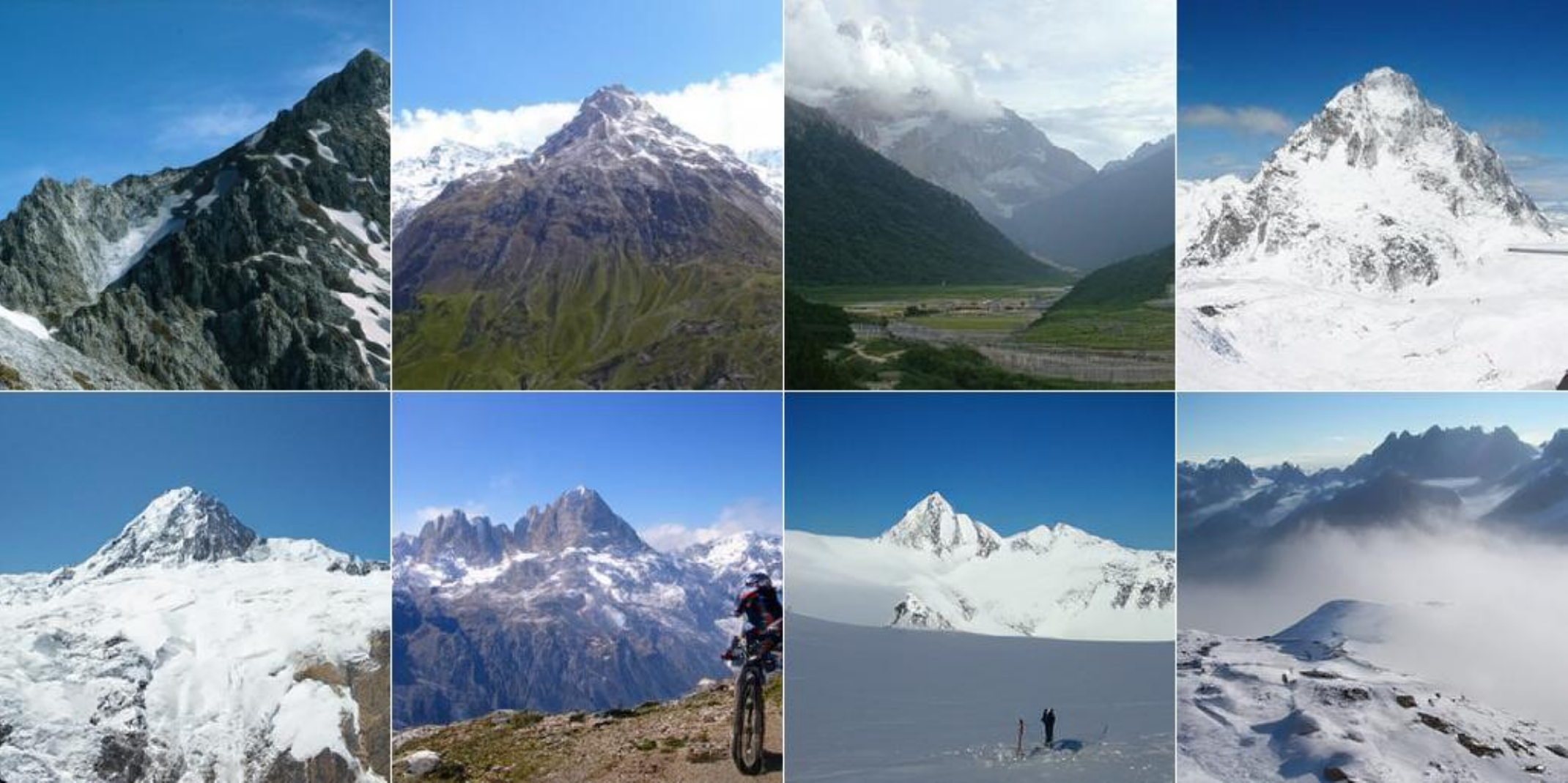}{alp}\hfill
\morevis{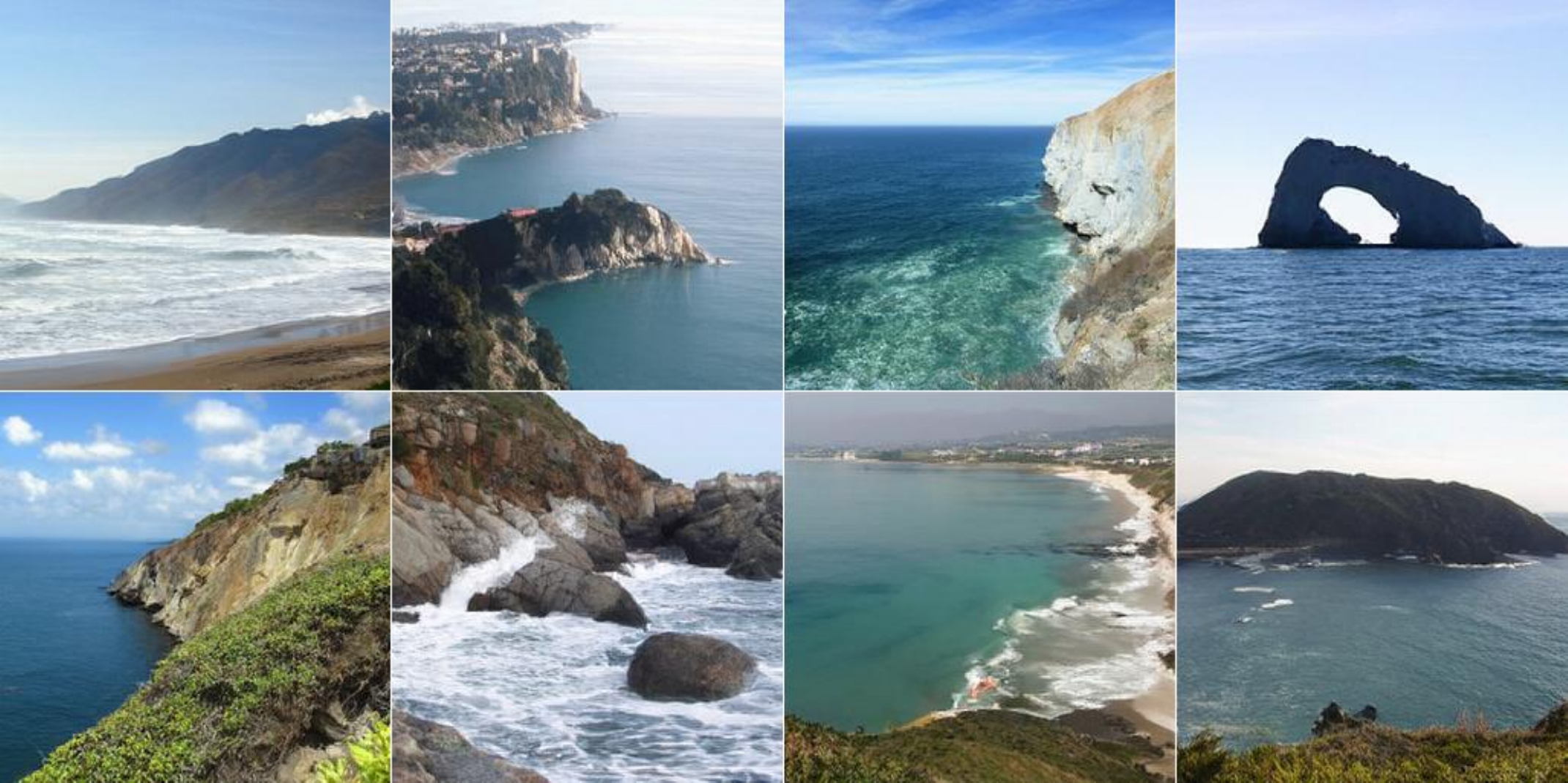}{promontory, headland}\\[0.35em]
\morevis{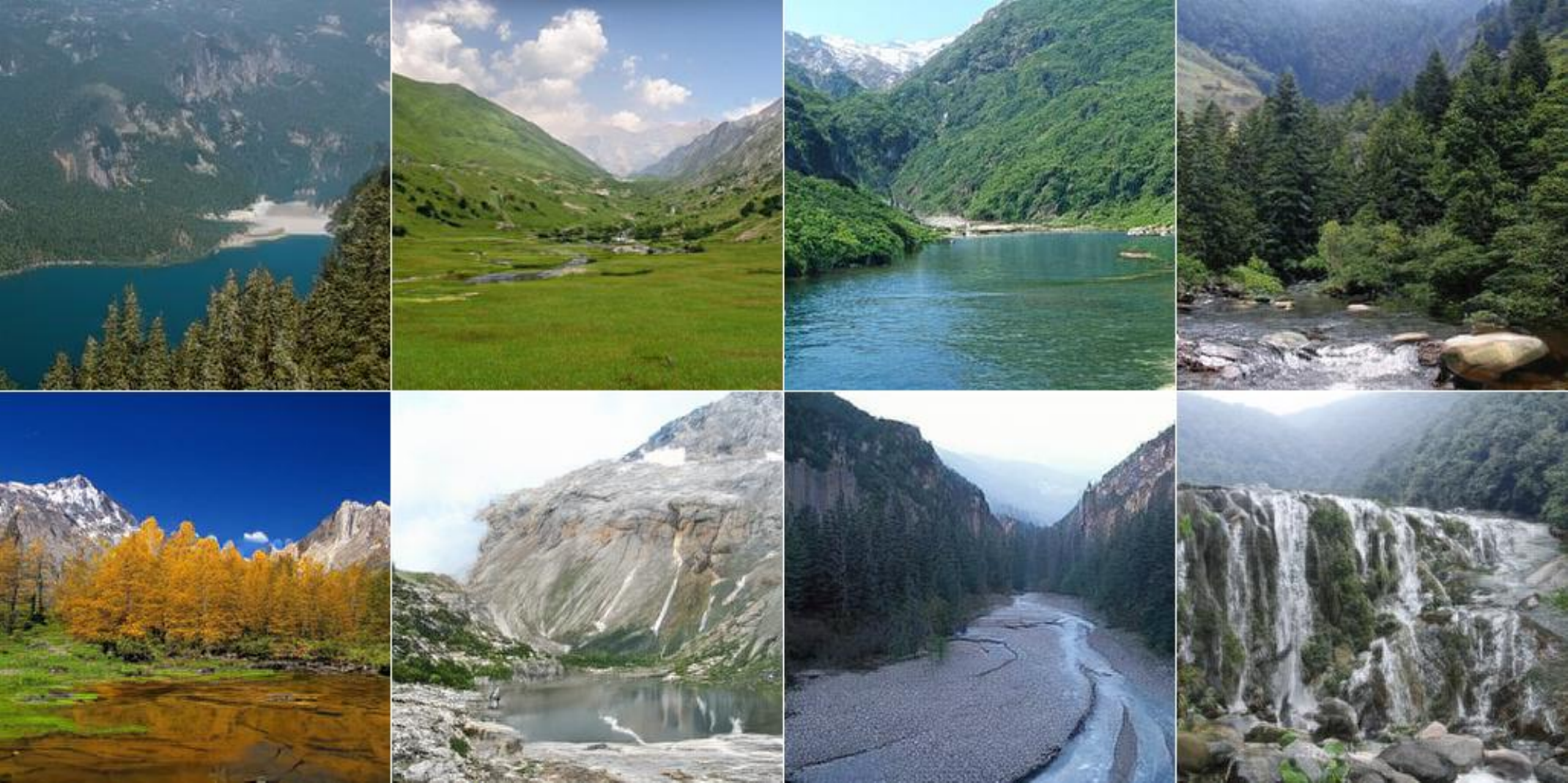}{valley}\hfill
\morevis{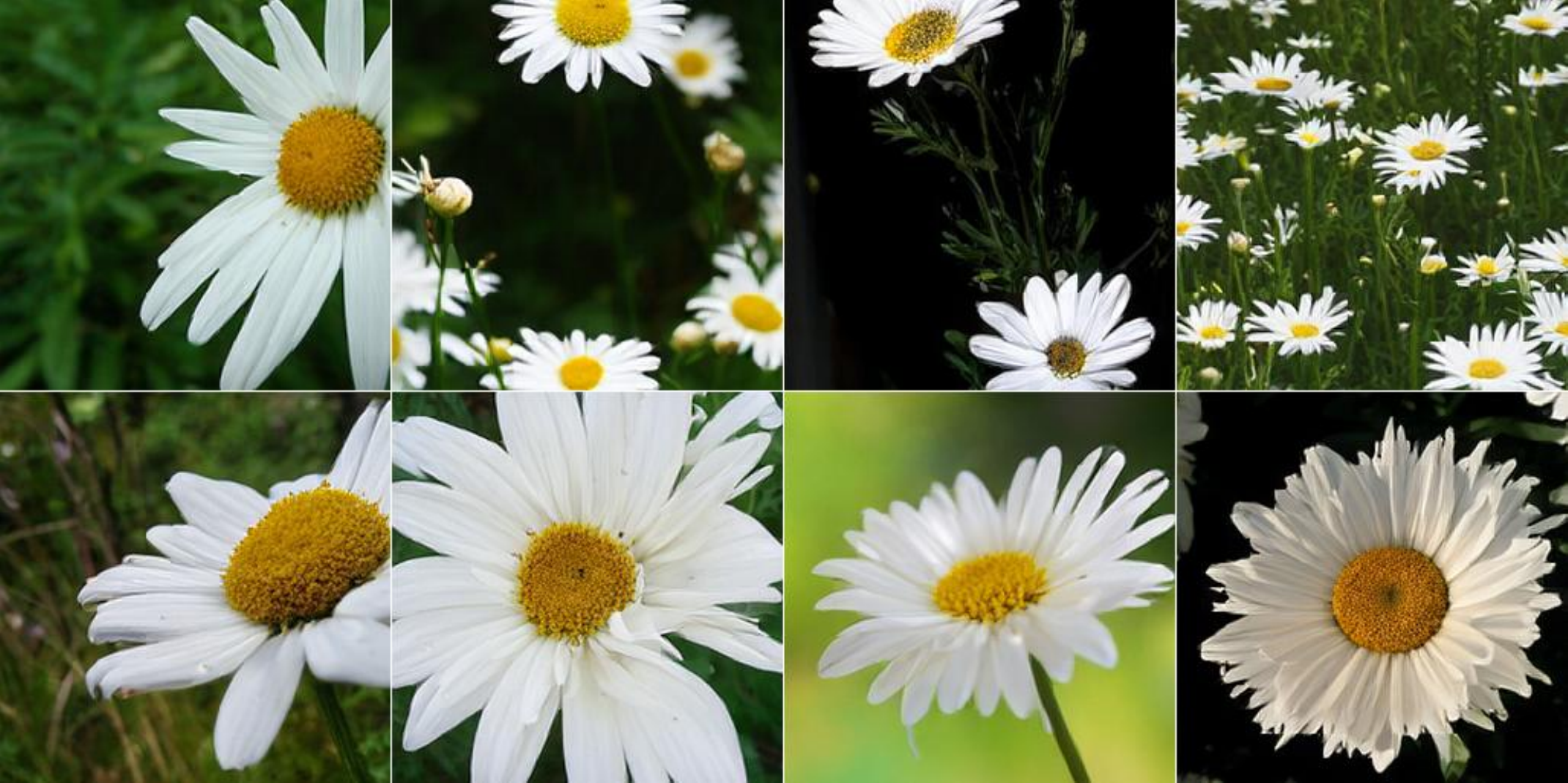}{daisy}
\caption{Additional visualizations of generated samples from distilled $\text{DiT}^\text{DH}\text{-XL}$ (FID=1.48).}
\label{fig:appendix_more_visualizations_4}
\end{figure}

%% file: tables/hyperparameters.tex
\begin{table}[t]
\centering
\caption{
Hyperparameter settings for \method.
}
\vspace{8pt}
\label{tab:hyperparameters}
\footnotesize
\SetTblrInner{rowsep=1.25pt}                     
\SetTblrInner{colsep=18.0pt}                      
\begin{tblr}{
    cell{1-11}{2}={halign=c,valign=m},         
    cell{1-11}{1}={halign=l,valign=m},           
    hline{1-2,12}={1-2}{1.0pt},      
}
Hyperparameter & Value \\
$N_c$ & $32$ \\
$N_{\mathrm{pos}}$ & $256$ \\
$N_{\mathrm{neg}}$ & $64$ \\
$N_{\mathrm{extra\_neg}}$ & $192$ \\
Temperatures $\tau$ & $\{0.02,\ 0.05,\ 0.2\}$ \\
Training steps $T$ & $10{,}000$ \\
Learning rate schedule & $\max\left\{3\times 10^{-7},\ 3\times 10^{-5}-\frac{(3\times 10^{-5}-3\times 10^{-7})}{10000}\cdot t\right\}$ \\
Gradient clipping & $1.0$ \\
Optimizer & AdamW ($\beta_1=0.9,\ \beta_2=0.95$) \\
Weight decay & $0.0$ \\
\end{tblr}
\vspace{4pt}
\end{table}

%% file: tables/idg_layers_ablation.tex
\begin{table}[t]
    \centering
    \caption{
        Ablation on the intermediate layers used by IDG at $\lambda=1.4$.\newline
    }
    \label{tab:idg_layers}
    \scriptsize
    \SetTblrInner{rowsep=2.0pt}
    \SetTblrInner{colsep=2.0pt}
    \begin{tblr}{
        width=\textwidth,
        colspec={Q[c] *{15}{X[c]}},
        cell{1-6}{1-16}={halign=c,valign=m},
        hline{1,2,6,Z}={1pt},
        vline{2}={0.5pt},
        vline{6,12,16}={0.25pt},
        cell{2-6}{8}={bg=lightgray!35},
    }
        \SetCell[c=16]{c}\textbf{Layers for Internal Drifting Guidance}
        & & & & & & & & & & & & & & & \\
        \textbf{8}
        & \checkmark &            &            &
        & \checkmark & \checkmark & \checkmark &            &            &
        & \checkmark & \checkmark & \checkmark &
        & \checkmark \\
        \textbf{10}
        &            & \checkmark &            &
        & \checkmark &            &            & \checkmark & \checkmark &
        & \checkmark & \checkmark &            &
          \checkmark & \checkmark \\
        \textbf{12}
        &            &            & \checkmark &
        &            & \checkmark &            & \checkmark &            &
          \checkmark & \checkmark &            & \checkmark &
          \checkmark & \checkmark \\
        \textbf{14}
        &            &            &            & \checkmark
        &            &            & \checkmark &            & \checkmark
        & \checkmark &            & \checkmark & \checkmark
        & \checkmark & \checkmark \\
        \textbf{FID} ($\downarrow$)
        & 1.54 & 1.56 & 1.55 & 1.52
        & 1.49 & 1.56 & \textbf{1.48} & 1.53 & 1.54
        & 1.50 & 1.60 & 1.56 & 1.55
        & 1.57 & 1.52 \\
    \end{tblr}
\end{table}

%% file: tables/idg_tokens_ablation.tex
\setlength{\columnsep}{16pt}
\begin{wraptable}{R}{0.5\textwidth}
\vspace{-1em}
\centering
\caption{
    Ablation on CLS and extra tokens.
}
\vspace{-5pt}
\label{tab:idg_tokens}
\footnotesize
\SetTblrInner{rowsep=2.0pt}
\SetTblrInner{colsep=13.0pt}
\begin{tblr}{
    cell{1-4}{1-3}={halign=c,valign=m},
    cell{1}{1}={r=2}{},
    cell{1}{2}={c=2}{},
    hline{2}={2-4}{},
    hline{1,3,Z}={1pt},
    vline{2}={1-5}{},
    vline{3}={2-5}{},
}
    \textbf{CLS + extra tokens}   & \SetCell[c=2]{c}\textbf{FID} ($\downarrow$) & \\
                                  & $\lambda=1.0$ & $\lambda=1.4$ \\
    \textcolor{red}{\ding{55}}    & 1.77          & 1.56 \\
    \textcolor{green}{\checkmark} & \textbf{1.74} & \textbf{1.48} \\
\end{tblr}
\end{wraptable}

%% file: tables/ablation_idg_schedule.tex
\setlength{\columnsep}{16pt}
\begin{wraptable}{R}{0.4\textwidth}
        \centering
        \caption{
            Ablation on the IDG training schedule with $\omega=1.4$.
        }
        \label{tab:idg_training_schedule}
        \footnotesize
        \SetTblrInner{rowsep=2.0pt}
        \SetTblrInner{colsep=15.0pt}
        \begin{tblr}{
            cell{1-3}{1-2}={halign=c,valign=m},
            hline{1,2,Z}={1pt},
            cell{3}{1-2}={bg=lightgray!35},
        }
            \textbf{Training schedule} & \textbf{FID} ($\downarrow$) \\
            Two-stage                & 1.50 \\
            One-stage                & \textbf{1.48} \\
        \end{tblr}
\end{wraptable}

%% file: tables/ablation_idg_loss.tex
\setlength{\columnsep}{16pt}
\begin{wraptable}{R}{0.4\textwidth}
        \centering
        \caption{
            Ablation on the training loss for the guidance head ($\omega=1.4$).
        }
        \label{tab:idg_training_loss}
        \footnotesize
        \SetTblrInner{rowsep=2.0pt}
        \SetTblrInner{colsep=15.0pt}
        \begin{tblr}{
            cell{1-5}{1}={halign=l,valign=m},
            cell{1-5}{2}={halign=c,valign=m},
            hline{1,2,Z}={1pt},
        }
            \textbf{Training schedule} & \textbf{FID} ($\downarrow$) \\
            Drifting                 & \textbf{1.48} \\
            MSE                      & 1.51 \\
            Cosine sim.              & 267.65 \\
            Cosine sim. + norm       & 1.49
        \end{tblr}
\end{wraptable}

%% file: tables/from_scratch.tex
\begin{table*}[t]
\centering
\caption{
Attempts to train drifting models from scratch without auxiliary MAEs.
}
\label{tab:from_scratch}
\footnotesize
\SetTblrInner{rowsep=1.25pt}                     
\SetTblrInner{colsep=22.0pt}                      
\begin{tblr}{
  colspec={l c c c},
  row{1}={font=\bfseries},
  hline{1,Z}={1pt},
  hline{2}={0.5pt},
}
Latent space & Model              & Strategy & FID ($\downarrow$) \\
SD-VAE       & DiT-XL             & --               & 220.55 \\
RAE          & DiT$^\text{DH}$-XL & --               & 100.31 \\
RAE          & DiT$^\text{DH}$-XL & + decode-encode   & 7.04 \\
\end{tblr}
\end{table*}

%% file: iclr2027_arxiv.bbl
\begin{thebibliography}{60}
\providecommand{\natexlab}[1]{#1}
\providecommand{\url}[1]{\texttt{#1}}
\expandafter\ifx\csname urlstyle\endcsname\relax
  \providecommand{\doi}[1]{doi: #1}\else
  \providecommand{\doi}{doi: \begingroup \urlstyle{rm}\Url}\fi

\bibitem[Arjovsky \& Bottou(2017)Arjovsky and Bottou]{arjovsky2017a}
Mart{\'{\i}}n Arjovsky and L{\'{e}}on Bottou.
\newblock Towards principled methods for training generative adversarial networks.
\newblock In \emph{International Conference on Learning Representations}, 2017.

\bibitem[Arjovsky et~al.(2017)Arjovsky, Chintala, and Bottou]{arjovsky2017wgan}
Mart{\'{\i}}n Arjovsky, Soumith Chintala, and L{\'{e}}on Bottou.
\newblock Wasserstein generative adversarial networks.
\newblock In \emph{International Conference on Machine Learning}, 2017.

\bibitem[Blattmann et~al.(2023)Blattmann, Dockhorn, Kulal, Mendelevitch, Kilian, Lorenz, Levi, English, Voleti, Letts, Jampani, and Rombach]{blattmann2024svd}
Andreas Blattmann, Tim Dockhorn, Sumith Kulal, Daniel Mendelevitch, Maciej Kilian, Dominik Lorenz, Yam Levi, Zion English, Vikram Voleti, Adam Letts, Varun Jampani, and Robin Rombach.
\newblock Stable video diffusion: Scaling latent video diffusion models to large datasets.
\newblock \emph{arXiv preprint arXiv:2311.15127}, 2023.

\bibitem[Brock et~al.(2019)Brock, Donahue, and Simonyan]{brock2019biggan}
Andrew Brock, Jeff Donahue, and Karen Simonyan.
\newblock Large scale {GAN} training for high fidelity natural image synthesis.
\newblock In \emph{International Conference on Learning Representations}, 2019.

\bibitem[Caron et~al.(2021)Caron, Touvron, Misra, J{\'{e}}gou, Mairal, Bojanowski, and Joulin]{caron2021dino}
Mathilde Caron, Hugo Touvron, Ishan Misra, Herv{\'{e}} J{\'{e}}gou, Julien Mairal, Piotr Bojanowski, and Armand Joulin.
\newblock Emerging properties in self-supervised vision transformers.
\newblock In \emph{International Conference on Computer Vision}, 2021.

\bibitem[Chen et~al.(2024)Chen, Zhou, Wang, Shen, and Lyu]{chen2024trajectory}
Defang Chen, Zhenyu Zhou, Can Wang, Chunhua Shen, and Siwei Lyu.
\newblock On the trajectory regularity of ode-based diffusion sampling.
\newblock In \emph{International Conference on Machine Learning}, 2024.

\bibitem[Chen et~al.(2025)Chen, Zhang, Tan, Guibas, Wetzstein, and Bi]{piflow}
Hansheng Chen, Kai Zhang, Hao Tan, Leonidas Guibas, Gordon Wetzstein, and Sai Bi.
\newblock pi-flow: Policy-based few-step generation via imitation distillation.
\newblock \emph{arXiv preprint arXiv:2510.14974}, 2025.
\newblock URL \url{https://arxiv.org/abs/2510.14974}.

\bibitem[Deng et~al.(2009)Deng, Dong, Socher, Li, Li, and Fei{-}Fei]{deng2009imagenet}
Jia Deng, Wei Dong, Richard Socher, Li{-}Jia Li, Kai Li, and Li~Fei{-}Fei.
\newblock Imagenet: {A} large-scale hierarchical image database.
\newblock In \emph{IEEE/CVF Conference on Computer Vision Pattern Recognition}, 2009.

\bibitem[Deng et~al.(2026)Deng, Li, Li, Du, and He]{deng2026generative}
Mingyang Deng, He~Li, Tianhong Li, Yilun Du, and Kaiming He.
\newblock Generative modeling via drifting.
\newblock \emph{arXiv preprint arXiv:2602.04770}, 2026.

\bibitem[Dhariwal \& Nichol(2021)Dhariwal and Nichol]{dhariwal2021adm}
Prafulla Dhariwal and Alexander~Quinn Nichol.
\newblock Diffusion models beat gans on image synthesis.
\newblock In \emph{Advances in Neural Information Processing System}, 2021.

\bibitem[Esser et~al.(2024)Esser, Kulal, Blattmann, Entezari, M{\"{u}}ller, Saini, Levi, Lorenz, Sauer, Boesel, Podell, Dockhorn, English, and Rombach]{esser2024sd3}
Patrick Esser, Sumith Kulal, Andreas Blattmann, Rahim Entezari, Jonas M{\"{u}}ller, Harry Saini, Yam Levi, Dominik Lorenz, Axel Sauer, Frederic Boesel, Dustin Podell, Tim Dockhorn, Zion English, and Robin Rombach.
\newblock Scaling rectified flow transformers for high-resolution image synthesis.
\newblock In \emph{International Conference on Machine Learning}, 2024.

\bibitem[Fan et~al.(2026)Fan, Wu, Cao, et~al.]{fan2026scot}
Xuhui Fan, Hongyu Wu, Longbing Cao, et~al.
\newblock Scot: Unifying consistency models and rectified flows via straight-consistent trajectories.
\newblock \emph{Advances in Neural Information Processing System}, 2026.

\bibitem[Fu et~al.(2026)Fu, Wang, Guo, Zhou, Nie, and Wen]{fu2026frozenpixelspacediffusionmodel}
Zixuan Fu, Chong Wang, Lanqing Guo, Kailai Zhou, Jiahao Nie, and Bihan Wen.
\newblock A frozen pixel-space diffusion model can guide itself with its own samples.
\newblock \emph{arXiv preprint arXiv:2607.29122}, 2026.

\bibitem[Geng et~al.(2026{\natexlab{a}})Geng, Deng, Bai, Kolter, and He]{geng2025meanflow}
Zhengyang Geng, Mingyang Deng, Xingjian Bai, Zico Kolter, and Kaiming He.
\newblock Mean flows for one-step generative modeling.
\newblock \emph{Advances in Neural Information Processing System}, 2026{\natexlab{a}}.

\bibitem[Geng et~al.(2026{\natexlab{b}})Geng, Lu, Wu, Shechtman, Kolter, and He]{geng2025improved}
Zhengyang Geng, Yiyang Lu, Zongze Wu, Eli Shechtman, J~Zico Kolter, and Kaiming He.
\newblock Improved mean flows: On the challenges of fastforward generative models.
\newblock \emph{IEEE/CVF Conference on Computer Vision Pattern Recognition}, 2026{\natexlab{b}}.

\bibitem[Goodfellow et~al.(2014)Goodfellow, Pouget-Abadie, Mirza, Xu, Warde-Farley, Ozair, Courville, and Bengio]{goodfellow2014gan}
Ian Goodfellow, Jean Pouget-Abadie, Mehdi Mirza, Bing Xu, David Warde-Farley, Sherjil Ozair, Aaron Courville, and Yoshua Bengio.
\newblock Generative adversarial nets.
\newblock In \emph{Advances in Neural Information Processing System}, 2014.

\bibitem[Heusel et~al.(2017)Heusel, Ramsauer, Unterthiner, Nessler, and Hochreiter]{heusel2017fid}
Martin Heusel, Hubert Ramsauer, Thomas Unterthiner, Bernhard Nessler, and Sepp Hochreiter.
\newblock Gans trained by a two time-scale update rule converge to a local nash equilibrium.
\newblock In \emph{Advances in Neural Information Processing System}, 2017.

\bibitem[Ho \& Salimans(2021)Ho and Salimans]{ho2021cfg}
Jonathan Ho and Tim Salimans.
\newblock Classifier-free diffusion guidance.
\newblock In \emph{Advances in Neural Information Processing System Workshop}, 2021.

\bibitem[Ho et~al.(2020)Ho, Jain, and Abbeel]{ho2020ddpm}
Jonathan Ho, Ajay Jain, and Pieter Abbeel.
\newblock Denoising diffusion probabilistic models.
\newblock In \emph{Advances in Neural Information Processing System}, 2020.

\bibitem[Hu et~al.(2025)Hu, Lai, Wu, Mitsufuji, and Ermon]{hu2025raemeanflow}
Zheyuan Hu, Chieh-Hsin Lai, Ge~Wu, Yuki Mitsufuji, and Stefano Ermon.
\newblock Meanflow transformers with representation autoencoders.
\newblock \emph{arXiv preprint arXiv:2511.13019}, 2025.

\bibitem[Karras et~al.(2024)Karras, Aittala, Kynk{\"{a}}{\"{a}}nniemi, Lehtinen, Aila, and Laine]{karras2024autoguidance}
Tero Karras, Miika Aittala, Tuomas Kynk{\"{a}}{\"{a}}nniemi, Jaakko Lehtinen, Timo Aila, and Samuli Laine.
\newblock Guiding a diffusion model with a bad version of itself.
\newblock In \emph{Advances in Neural Information Processing System}, 2024.

\bibitem[Kim et~al.(2024)Kim, Lai, Liao, Takida, Murata, Uesaka, Mitsufuji, and Ermon]{Kim2024pagoda}
Dongjun Kim, Chieh{-}Hsin Lai, Wei{-}Hsiang Liao, Yuhta Takida, Naoki Murata, Toshimitsu Uesaka, Yuki Mitsufuji, and Stefano Ermon.
\newblock Pagoda: Progressive growing of a one-step generator from a low-resolution diffusion teacher.
\newblock In \emph{Advances in Neural Information Processing System}, 2024.

\bibitem[Kynk{\"{a}}{\"{a}}nniemi et~al.(2019)Kynk{\"{a}}{\"{a}}nniemi, Karras, Laine, Lehtinen, and Aila]{kynkaanniemi2019pr}
Tuomas Kynk{\"{a}}{\"{a}}nniemi, Tero Karras, Samuli Laine, Jaakko Lehtinen, and Timo Aila.
\newblock Improved precision and recall metric for assessing generative models.
\newblock In \emph{Advances in Neural Information Processing System}, 2019.

\bibitem[Labs et~al.(2025)Labs, Batifol, Blattmann, Boesel, Consul, Diagne, Dockhorn, English, English, Esser, Kulal, Lacey, Levi, Li, Lorenz, Müller, Podell, Rombach, Saini, Sauer, and Smith]{labs2025flux1kontextflowmatching}
Black~Forest Labs, Stephen Batifol, Andreas Blattmann, Frederic Boesel, Saksham Consul, Cyril Diagne, Tim Dockhorn, Jack English, Zion English, Patrick Esser, Sumith Kulal, Kyle Lacey, Yam Levi, Cheng Li, Dominik Lorenz, Jonas Müller, Dustin Podell, Robin Rombach, Harry Saini, Axel Sauer, and Luke Smith.
\newblock Flux.1 kontext: Flow matching for in-context image generation and editing in latent space.
\newblock \emph{arXiv preprint arXiv:2506.15742}, 2025.

\bibitem[Lai et~al.(2026)Lai, Nguyen, Murata, Takida, Uesaka, Mitsufuji, Ermon, and Tao]{lai2026unifiedviewscorebaseddrifting}
Chieh-Hsin Lai, Bac Nguyen, Naoki Murata, Yuhta Takida, Toshimitsu Uesaka, Yuki Mitsufuji, Stefano Ermon, and Molei Tao.
\newblock A unified view of drifting and score-based models.
\newblock \emph{arXiv preprint arXiv:2603.07514}, 2026.

\bibitem[Lin et~al.(2024)Lin, Wang, and Yang]{lin2024sdxllightning}
Shanchuan Lin, Anran Wang, and Xiao Yang.
\newblock Sdxl-lightning: Progressive adversarial diffusion distillation.
\newblock \emph{arXiv preprint arXiv:2402.13929}, 2024.

\bibitem[Lipman et~al.(2023)Lipman, Chen, Ben{-}Hamu, Nickel, and Le]{lipman2023}
Yaron Lipman, Ricky T.~Q. Chen, Heli Ben{-}Hamu, Maximilian Nickel, and Matthew Le.
\newblock Flow matching for generative modeling.
\newblock In \emph{International Conference on Learning Representations}, 2023.

\bibitem[Liu \& Yue(2026)Liu and Yue]{liu2026learning}
Wenze Liu and Xiangyu Yue.
\newblock Learning to integrate diffusion {ODE}s by averaging the derivatives.
\newblock In \emph{Advances in Neural Information Processing System}, 2026.

\bibitem[Liu et~al.(2023)Liu, Gong, and Liu]{liu2023rectifiedflow}
Xingchao Liu, Chengyue Gong, and Qiang Liu.
\newblock Flow straight and fast: Learning to generate and transfer data with rectified flow.
\newblock In \emph{International Conference on Learning Representations}, 2023.

\bibitem[Liu et~al.(2024)Liu, Zhang, Li, Yan, Gao, Chen, Yuan, Huang, Sun, Gao, He, and Sun]{liu2024sora}
Yixin Liu, Kai Zhang, Yuan Li, Zhiling Yan, Chujie Gao, Ruoxi Chen, Zhengqing Yuan, Yue Huang, Hanchi Sun, Jianfeng Gao, Lifang He, and Lichao Sun.
\newblock Sora: {A} review on background, technology, limitations, and opportunities of large vision models.
\newblock \emph{arXiv preprint arXiv:2402.17177}, 2024.

\bibitem[Lu et~al.(2026)Lu, Lu, Sun, Zhao, Jiang, Wang, Li, Geng, and He]{lu2026one}
Yiyang Lu, Susie Lu, Qiao Sun, Hanhong Zhao, Zhicheng Jiang, Xianbang Wang, Tianhong Li, Zhengyang Geng, and Kaiming He.
\newblock One-step latent-free image generation with pixel mean flows.
\newblock \emph{arXiv preprint arXiv:2601.22158}, 2026.

\bibitem[Luo et~al.(2023)Luo, Tan, Huang, Li, and Zhao]{luo2023a}
Simian Luo, Yiqin Tan, Longbo Huang, Jian Li, and Hang Zhao.
\newblock Latent consistency models: Synthesizing high-resolution images with few-step inference.
\newblock \emph{arXiv preprint arXiv:2310.04378}, 2023.

\bibitem[Ma et~al.(2024)Ma, Goldstein, Albergo, Boffi, Vanden-Eijnden, and Xie]{ma2024sit}
Nanye Ma, Mark Goldstein, Michael~S Albergo, Nicholas~M Boffi, Eric Vanden-Eijnden, and Saining Xie.
\newblock Sit: Exploring flow and diffusion-based generative models with scalable interpolant transformers.
\newblock In \emph{European Conference on Computer Vision}, 2024.

\bibitem[Peebles \& Xie(2023)Peebles and Xie]{peebles2023dit}
William Peebles and Saining Xie.
\newblock Scalable diffusion models with transformers.
\newblock In \emph{International Conference on Computer Vision}, 2023.

\bibitem[Peng et~al.(2026)Peng, Zhu, Liu, Wu, Li, Sun, and Wu]{peng2026facm}
Yansong Peng, Kai Zhu, Yu~Liu, Pingyu Wu, Hebei Li, Xiaoyan Sun, and Feng Wu.
\newblock {FACM}: Flow-anchored consistency models.
\newblock In \emph{International Conference on Learning Representations}, 2026.

\bibitem[Rombach et~al.(2022)Rombach, Blattmann, Lorenz, Esser, and Ommer]{rombach2O22ldm}
Robin Rombach, Andreas Blattmann, Dominik Lorenz, Patrick Esser, and Bj{\"{o}}rn Ommer.
\newblock High-resolution image synthesis with latent diffusion models.
\newblock In \emph{IEEE/CVF Conference on Computer Vision Pattern Recognition}, 2022.

\bibitem[Salimans \& Ho(2022)Salimans and Ho]{salimans2022pd}
Tim Salimans and Jonathan Ho.
\newblock Progressive distillation for fast sampling of diffusion models.
\newblock In \emph{International Conference on Learning Representations}, 2022.

\bibitem[Sauer et~al.(2022)Sauer, Schwarz, and Geiger]{sauer2022styleganxl}
Axel Sauer, Katja Schwarz, and Andreas Geiger.
\newblock Stylegan-xl: Scaling stylegan to large diverse datasets.
\newblock In \emph{SIGGRAPH}, 2022.

\bibitem[Sauer et~al.(2024{\natexlab{a}})Sauer, Boesel, Dockhorn, Blattmann, Esser, and Rombach]{sauer2024ladd}
Axel Sauer, Frederic Boesel, Tim Dockhorn, Andreas Blattmann, Patrick Esser, and Robin Rombach.
\newblock Fast high-resolution image synthesis with latent adversarial diffusion distillation.
\newblock In \emph{SIGGRAPH Asia}, 2024{\natexlab{a}}.

\bibitem[Sauer et~al.(2024{\natexlab{b}})Sauer, Lorenz, Blattmann, and Rombach]{sauer2024add}
Axel Sauer, Dominik Lorenz, Andreas Blattmann, and Robin Rombach.
\newblock Adversarial diffusion distillation.
\newblock In \emph{European Conference on Computer Vision}, 2024{\natexlab{b}}.

\bibitem[Singh et~al.(2026)Singh, Zheng, Wu, Zhang, Shechtman, and Xie]{singh2026raev2}
Jaskirat Singh, Boyang Zheng, Zongze Wu, Richard Zhang, Eli Shechtman, and Saining Xie.
\newblock Improved baselines with representation autoencoders.
\newblock \emph{arXiv preprint arXiv:2605.18324}, 2026.

\bibitem[Sohl{-}Dickstein et~al.(2015)Sohl{-}Dickstein, Weiss, Maheswaranathan, and Ganguli]{dickstein2015dm}
Jascha Sohl{-}Dickstein, Eric~A. Weiss, Niru Maheswaranathan, and Surya Ganguli.
\newblock Deep unsupervised learning using nonequilibrium thermodynamics.
\newblock In \emph{International Conference on Machine Learning}, 2015.

\bibitem[Song et~al.(2021{\natexlab{a}})Song, Meng, and Ermon]{song2021ddim}
Jiaming Song, Chenlin Meng, and Stefano Ermon.
\newblock Denoising diffusion implicit models.
\newblock In \emph{International Conference on Learning Representations}, 2021{\natexlab{a}}.

\bibitem[Song \& Dhariwal(2024)Song and Dhariwal]{song2024ict}
Yang Song and Prafulla Dhariwal.
\newblock Improved techniques for training consistency models.
\newblock In \emph{International Conference on Learning Representations}, 2024.

\bibitem[Song et~al.(2021{\natexlab{b}})Song, Sohl{-}Dickstein, Kingma, Kumar, Ermon, and Poole]{song2021scoresde}
Yang Song, Jascha Sohl{-}Dickstein, Diederik~P. Kingma, Abhishek Kumar, Stefano Ermon, and Ben Poole.
\newblock Score-based generative modeling through stochastic differential equations.
\newblock In \emph{International Conference on Learning Representations}, 2021{\natexlab{b}}.

\bibitem[Song et~al.(2023)Song, Dhariwal, Chen, and Sutskever]{song2023cm}
Yang Song, Prafulla Dhariwal, Mark Chen, and Ilya Sutskever.
\newblock Consistency models.
\newblock In \emph{International Conference on Machine Learning}, 2023.

\bibitem[Tong et~al.(2025)Tong, Ma, Xie, and Jaakkola]{tong2025freeflow}
Shangyuan Tong, Nanye Ma, Saining Xie, and Tommi Jaakkola.
\newblock Flow map distillation without data.
\newblock \emph{arXiv preprint arXiv:2511.19428}, 2025.

\bibitem[Tong et~al.(2026)Tong, Zheng, Wang, Tang, Ma, Brown, Yang, Fergus, LeCun, and Xie]{scale-rae-2026}
Shengbang Tong, Boyang Zheng, Ziteng Wang, Bingda Tang, Nanye Ma, Ellis Brown, Jihan Yang, Rob Fergus, Yann LeCun, and Saining Xie.
\newblock Scaling text-to-image diffusion transformers with representation autoencoders.
\newblock \emph{arXiv preprint arXiv:2601.16208}, 2026.

\bibitem[Wang et~al.(2024)Wang, Huang, Bergman, Shen, Gao, Lingelbach, Sun, Bian, Song, Liu, et~al.]{wang2024phased}
Fu-Yun Wang, Zhaoyang Huang, Alexander~W Bergman, Dazhong Shen, Peng Gao, Michael Lingelbach, Keqiang Sun, Weikang Bian, Guanglu Song, Yu~Liu, et~al.
\newblock Phased consistency models.
\newblock \emph{Advances in Neural Information Processing System}, 2024.

\bibitem[Wang et~al.(2023)Wang, Zheng, He, Chen, and Zhou]{wang2023diffusiongan}
Zhendong Wang, Huangjie Zheng, Pengcheng He, Weizhu Chen, and Mingyuan Zhou.
\newblock Diffusion-gan: Training gans with diffusion.
\newblock In \emph{International Conference on Learning Representations}, 2023.

\bibitem[Yao \& Wang(2025)Yao and Wang]{Yao2025ReconstructionVG}
Jingfeng Yao and Xinggang Wang.
\newblock Reconstruction vs. generation: Taming optimization dilemma in latent diffusion models.
\newblock \emph{IEEE/CVF Conference on Computer Vision Pattern Recognition}, 2025.

\bibitem[Yin et~al.(2024{\natexlab{a}})Yin, Gharbi, Park, Zhang, Shechtman, Durand, and Freeman]{yin2024dmd2}
Tianwei Yin, Micha{\"e}l Gharbi, Taesung Park, Richard Zhang, Eli Shechtman, Fredo Durand, and William~T Freeman.
\newblock Improved distribution matching distillation for fast image synthesis.
\newblock In \emph{Advances in Neural Information Processing System}, 2024{\natexlab{a}}.

\bibitem[Yin et~al.(2024{\natexlab{b}})Yin, Gharbi, Park, Zhang, Shechtman, Durand, and Freeman]{yin2024improved}
Tianwei Yin, Micha{\"e}l Gharbi, Taesung Park, Richard Zhang, Eli Shechtman, Fredo Durand, and William~T Freeman.
\newblock Improved distribution matching distillation for fast image synthesis.
\newblock In \emph{Advances in Neural Information Processing System}, 2024{\natexlab{b}}.

\bibitem[Yin et~al.(2024{\natexlab{c}})Yin, Gharbi, Zhang, Shechtman, Durand, Freeman, and Park]{yin2024dmd}
Tianwei Yin, Micha{\"{e}}l Gharbi, Richard Zhang, Eli Shechtman, Fr{\'{e}}do Durand, William~T. Freeman, and Taesung Park.
\newblock One-step diffusion with distribution matching distillation.
\newblock In \emph{IEEE/CVF Conference on Computer Vision Pattern Recognition}, 2024{\natexlab{c}}.

\bibitem[Yu et~al.(2025)Yu, Kwak, Jang, Jeong, Huang, Shin, and Xie]{yu2025repa}
Sihyun Yu, Sangkyung Kwak, Huiwon Jang, Jongheon Jeong, Jonathan Huang, Jinwoo Shin, and Saining Xie.
\newblock Representation alignment for generation: Training diffusion transformers is easier than you think.
\newblock In \emph{International Conference on Learning Representations}, 2025.

\bibitem[Zhang et~al.(2026)Zhang, Mang, and Agrawal]{zhang2026rit}
Le~Zhang, Ning Mang, and Aishwarya Agrawal.
\newblock Rit: Vanilla diffusion transformers suffice in representation space.
\newblock \emph{arXiv preprint arXiv:2605.21981}, 2026.

\bibitem[Zheng et~al.(2025)Zheng, Ma, Tong, and Xie]{zheng2025diffusion}
Boyang Zheng, Nanye Ma, Shengbang Tong, and Saining Xie.
\newblock Diffusion transformers with representation autoencoders.
\newblock \emph{arXiv preprint arXiv:2510.11690}, 2025.

\bibitem[Zhou et~al.(2025)Zhou, Ermon, and Song]{zhou2025inductive}
Linqi Zhou, Stefano Ermon, and Jiaming Song.
\newblock Inductive moment matching.
\newblock In \emph{International Conference on Machine Learning}, 2025.

\bibitem[Zhou et~al.(2024)Zhou, Zheng, Wang, Yin, and Huang]{zhou2024sid}
Mingyuan Zhou, Huangjie Zheng, Zhendong Wang, Mingzhang Yin, and Hai Huang.
\newblock Score identity distillation: Exponentially fast distillation of pretrained diffusion models for one-step generation.
\newblock In \emph{International Conference on Machine Learning}, 2024.

\bibitem[Zhou et~al.(2026)Zhou, Li, Hu, Chen, and Gu]{zhou2025IG}
Xingyu Zhou, Qifan Li, Xiaobin Hu, Hai Chen, and Shuhang Gu.
\newblock Guiding a diffusion transformer with the internal dynamics of itself.
\newblock \emph{IEEE/CVF Conference on Computer Vision Pattern Recognition}, 2026.

\end{thebibliography}
